\documentclass[letterpaper]{article} 
\usepackage{microtype}
\usepackage{subcaption }
\usepackage{caption}
\usepackage{lmodern}
\usepackage{xspace}
\usepackage{booktabs} 

\usepackage{fullpage}
\usepackage{wrapfig}
\usepackage{amsmath,amssymb,bm,epsfig,epsf,color,mathbbol,fmtcount,semtrans,multirow,comment,boldline,pgfplots}
\usepackage{tcolorbox}
\tcbuselibrary{skins}
\usepackage{tikz}
\usetikzlibrary{pgfplots.groupplots}
\usepackage[utf8]{inputenc} 
\usepackage[T1]{fontenc}    
\usepackage{booktabs}       
\usepackage{amsfonts}       
\usepackage{nicefrac}       
\usepackage{microtype}      
\usepackage{mathtools}
\definecolor{darkred}{RGB}{150,0,0}
\definecolor{darkgreen}{RGB}{0,150,0}
\definecolor{darkblue}{RGB}{0,0,200}

\newtheorem{theorem}{Theorem}

\newtheorem{assumption}{Assumption}

\newtheorem{lemma}{Lemma}
\newtheorem{corollary}{Corollary}

\newtheorem{definition}{Definition}

\newtheorem{remark}{Remark}

\newcommand{\cln}[1]{\textcolor{red}{}}

\newcommand{\ylm}[1]{\marginpar{\color{orange}\tiny\ttfamily YL: #1}}

\def \endprf{\hfill {\vrule height6pt width6pt depth0pt}\medskip}

\newenvironment{proof}{\noindent {\bf Proof} }{\endprf\par}

\newcommand{\tsn}[1]{{\left\vert\kern-0.25ex\left\vert\kern-0.25ex\left\vert #1 
		\right\vert\kern-0.25ex\right\vert\kern-0.25ex\right\vert}}

\newcommand{\noi}{\noindent}

\newcommand{\eps}{\varepsilon}
\newcommand{\beps}{\boldsymbol{\eps}}
\newcommand{\Gh}{\widehat{\Gc}}
\newcommand{\Gt}{\widetilde{\Gc}}

\newcommand{\NT}{N_{\text{tot}}}

\newcommand{\st}{\star}

\newcommand{\h}{\vct{h}}
\newcommand{\f}{\vct{f}}

\newcommand{\distas}{\overset{\text{i.i.d.}}{\sim}}

\newcommand{\MP}{\text{Multipath}\xspace}
\newcommand{\Mp}{\text{multipath}\xspace}

\newcommand{\beq}{\begin{equation}}
	\newcommand{\ba}{\begin{align}}
		\newcommand{\ea}{\end{align}}

	\newcommand{\eeq}{\end{equation}}

\newcommand{\nn}{\nonumber}
\newcommand{\la}{\lambda}

\newcommand{\K}{{K}}
\newcommand{\A}{{\mtx{A}}}

\newcommand{\B}{{{\mtx{B}}}}
\newcommand{\hB}{\hat{\B}}

\newcommand{\Gb}{{\mtx{G}}}

\newcommand{\Lc}{{\cal{L}}}

\newcommand{\Lch}{{\widehat{\cal{L}}}}

\newcommand{\Nc}{{\cal{N}}}

\newcommand{\Dc}{{\cal{D}}}

\newcommand{\Tbar}{{\bar{T}}}

\newcommand{\Hb}{{\mtx{H}}}

\newcommand{\Gc}{{\cal{G}}}
\newcommand{\Qc}{{\cal{Q}}}

\newcommand{\bSi}{{\boldsymbol{{\Sigma}}}}

\newcommand{\onebb}{{\mathbf{1}}}
\newcommand{\Iden}{{\mtx{I}}}
\newcommand{\M}{{\mtx{M}}}

\newcommand{\order}[1]{{\cal{O}}\left(#1\right)}
\newcommand{\ordet}[1]{{\widetilde{\cal{O}}}\left(#1\right)}

\newcommand{\z}{{\vct{z}}}

\newcommand{\tn}[1]{\|{#1}\|}

\newcommand{\tf}[1]{\|{#1}\|_{F}}
\newcommand{\te}[1]{\|{#1}\|_{\psi_1}}

\newcommand{\Cc}{\mathcal{C}}

\newcommand{\Ac}{\mathcal{A}}

\newcommand{\Rc}{\mathcal{R}}

\newcommand{\bt}{{\boldsymbol{\theta}}}
\newcommand{\bT}{{\boldsymbol{\Theta}}}
\newcommand{\bTh}{{\boldsymbol{\hat\Theta}}}
\newcommand{\bts}{{\boldsymbol{\theta}_\st}}

\newcommand{\bal}{{\boldsymbol{\alpha}}}

\newcommand{\bth}{{\boldsymbol{\hat{\theta}}}}
\newcommand{\bPhi}{{\boldsymbol{\Phi}}}
\newcommand{\Phic}{\Phi_{\text{used}}}
\newcommand{\bphi}{{\boldsymbol{\phi}}}
\newcommand{\genbound}{\Gt_N(\Hc)+\sum_{\ell=1}^L\sqrt{\hat{\K}_\ell}\Gt_{NT}(\Psi_\ell)+\sqrt{\frac{\log|\Ac|}{N}+\frac{\log(2/\delta)}{NT}}}
\newcommand{\genb}[1]{\Cc_N(#1)}

\newcommand{\bpsi}{{\boldsymbol{\psi}}}

\newcommand{\psit}{{{\tilde{\psi}}}}

\newcommand{\bho}{{\boldsymbol{\omega}}}
\newcommand{\bsgm}{{\boldsymbol{\sigma}}}

\newcommand{\Bc}{\mathcal{B}}
\newcommand{\Sc}{\mathcal{S}}
\newcommand{\Scb}{{{\mathcal{S}}_\all}}
\newcommand{\Dcb}{{\bar{\mathcal{D}}}}

\newcommand{\Nn}{\mathcal{N}}

\newcommand{\vb}{\vct{v}}

\newcommand{\fb}{\vct{f}}

\newcommand{\Ic}{{\mathcal{I}}}
\newcommand{\all}{{\text{all}}}
\newcommand{\used}{{\text{used}}}

\newcommand{\hhb}{{\hat{\h}}}

\newcommand{\g}{{\vct{g}}}

\newcommand{\Tc}{\mathcal{T}}

\newcommand{\Fc}{\mathcal{F}}

\newcommand{\Xc}{\mathcal{X}}
\newcommand{\Zc}{\mathcal{Z}}

\newcommand{\Hc}{\mathcal{H}}

\newcommand{\x}{\vct{x}}

\newcommand{\y}{\vct{y}}
\newcommand{\prb}{1-e^{-cn}-2e^{-\sqrt{tn}\wedge t}}
\newcommand{\prbb}{1-2e^{-cn}-4e^{-\sqrt{tn}\wedge t}}
\newcommand{\prbM}{1-e^{-cM}-2e^{-\sqrt{tM}\wedge t}}

\newcommand{\bgl}{{~\big |~}}

\definecolor{emmanuel}{RGB}{255,127,0}

\newcommand{\R}{\mathbb{R}}
\newcommand{\Pro}{\mathbb{P}}

\renewcommand{\P}{\operatorname{\mathbb{P}}}
\newcommand{\E}{\operatorname{\mathbb{E}}}
\newcommand{\Ec}{{\cal{E}}}

\newcommand{\vct}[1]{\bm{#1}}
\newcommand{\mtx}[1]{\bm{#1}}

\newcommand{\Id}{\text{\em I}}

\newcommand{\X}{{\mtx{X}}}
\newcommand{\Y}{{\mtx{Y}}}
\newcommand{\Z}{{\mtx{Z}}}

\newcommand{\red}[1]{\textcolor{black}{#1}}

\newcommand{\comp}{\text{comp}}

\newcommand{\tgt}{\Tc}
\newcommand{\MTL}{\text{M$^2$TL}}
\newcommand{\TFL}{\text{TLOP}}

\newcommand{\RMTL}{{\cal{R}}_{\MTL}}
\newcommand{\RTFL}{{\cal{R}}_{\TFL}}
\newcommand{\Rt}{{\cal{R}}_{t}}
\newcommand{\Bias}{\text{Bias}}
\newcommand{\Dist}{\text{Dist}}
\newcommand{\tth}{\text{th}}
\newcommand{\Psil}[1]{\Psi^{(#1)}}
\newcommand{\psil}[1]{\psi^{(#1)}}
\newcommand{\GAMMA}{\Gamma^{\dagger}}
\newcommand{\MMTL}{{Multipath MTL }}
\newcommand{\OTFL}{{Transfer Learning with Optimal Pathway }}
\newcommand{\bias}{{supernet bias }}

\newcommand\scalemath[2]{\scalebox{#1}{\mbox{\ensuremath{\displaystyle #2}}}}
\newcommand*{\QE}{\hfill\ensuremath{\square}}%

\newcommand{\DoF}{\text{DoF}}

\usepackage{aaai23}  
\usepackage{times}  
\usepackage{helvet}  
\usepackage{courier}  
\usepackage[hyphens]{url}  
\usepackage{graphicx} 
\usepackage{natbib}  
\usepackage{caption} 
\usepackage{algorithm}
\usepackage{algorithmic}

\usepackage{newfloat}
\usepackage{listings}
\DeclareCaptionStyle{ruled}{labelfont=normalfont,labelsep=colon,strut=off} 
\floatstyle{ruled}
\newfloat{listing}{tb}{lst}{}
\floatname{listing}{Listing}
\title{Provable Pathways: Learning Multiple Tasks over Multiple Paths}
\author{
    Yingcong Li\thanks{Emails: \{yli692@,oymak@ece.\}ucr.edu}\quad\quad\quad\quad
    Samet Oymak$^{\ast\dagger}$
}
\affiliations{
    $^\ast$ University of California, Riverside\\
    ~$^\dagger$ University of Michigan, Ann Arbor
}

\begin{document}
\maketitle

\begin{abstract}
Constructing useful representations across a large number of tasks is a key requirement for sample-efficient intelligent systems. A traditional idea in multitask learning (MTL) is building a shared representation across tasks which can then be adapted to new tasks by tuning last layers. A desirable refinement of using a shared one-fits-all representation is to construct task-specific representations. To this end, recent PathNet/muNet architectures represent individual tasks as pathways within a larger supernet. The subnetworks induced by pathways can be viewed as task-specific representations that are composition of modules within supernet's computation graph. This work explores the pathways proposal from the lens of statistical learning: We first develop novel generalization bounds for empirical risk minimization problems learning multiple tasks over multiple paths (Multipath MTL). In conjunction, we formalize the benefits of resulting multipath representation when adapting to new downstream tasks. Our bounds are expressed in terms of Gaussian complexity, lead to tangible guarantees for the class of linear representations, and provide novel insights into the quality and benefits of a multipath representation. When computation graph is a tree, Multipath MTL hierarchically clusters the tasks and builds cluster-specific representations. We provide further discussion and experiments for hierarchical MTL and rigorously identify the conditions under which Multipath MTL is provably superior to traditional MTL approaches with shallow supernets.


\end{abstract}

\section{Introduction}

Multitask learning (MTL) promises to deliver significant accuracy improvements by leveraging similarities across many tasks through shared representations. The potential of MTL has been recognized since 1990s \cite{caruana1997multitask} however its impact has grown over time thanks to more recent machine learning applications arising in computer vision and NLP that involve large datasets with thousands of classes/tasks. Representation learning techniques (e.g.~MTL and self-supervision) are also central to the success of deep learning as large pretrained models enable data-efficient learning for downstream transfer learning tasks \cite{deng2009imagenet,brown2020language}.

As we move from tens of tasks trained with small models to thousands of tasks trained with large models, new statistical and computational challenges arise: First, not all tasks will be closely related to each other, for instance, tasks might admit a natural clustering into groups. This is also connected to heterogeneity challenge in federated learning where clients have distinct distributions and benefit from personalization. To address this challenge, rather than a single task-agnostic representation, it might be preferable to use a task-specific representation. Secondly, pretrained language and vision models achieve better accuracy with larger sizes which creates computational challenges as they push towards trillion parameters. This motivated new architectural proposals such as Pathways/PathNet \cite{fernando2017pathnet,pathways,gesmundo2022munet} where tasks can be computed over compute-efficient subnetworks. At a high-level, each subnetwork is created by a composition of modules within a larger supernet which induces a pathway as depicted in Figure \ref{fig:intro}. Inspired from these challenges, we ask
\begin{figure}[t]
\centering
\begin{subfigure}[t]{.245\textwidth}
  \begin{tikzpicture}
  \centering
  \node at (0,0) {\includegraphics[width=\linewidth]{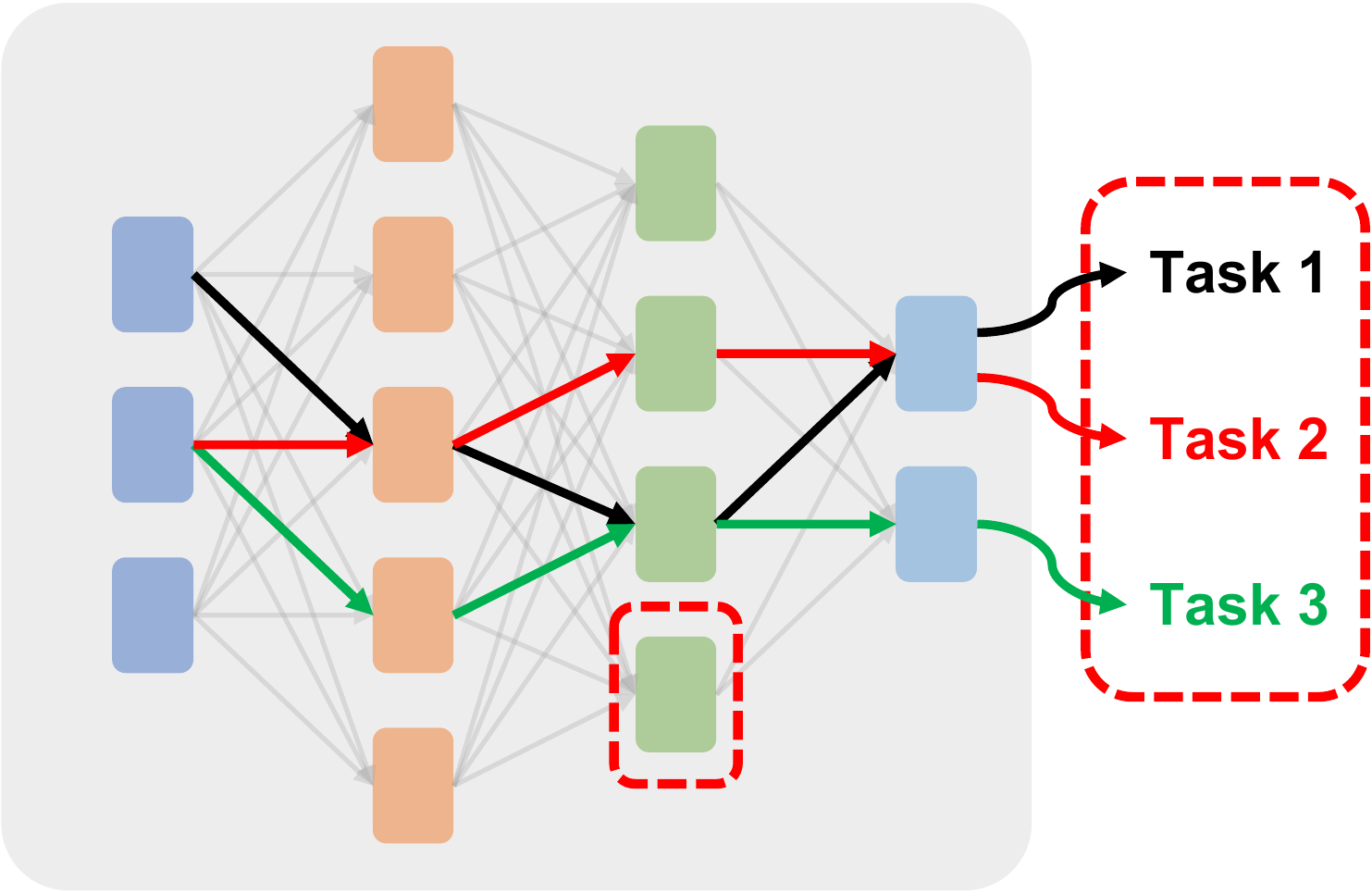}};
  \node at (-0.5, 1.55) [scale=0.7] {Supernet ($\bphi$)};
  \node at (1.72, 1.0) [scale=0.7] {Heads ($h_t$)};
  \node at (-0, -1.25) [scale=0.7] {Module ($\psi_\ell^k$)};
  \end{tikzpicture}
  \caption{General computation graph}\label{fig:intro_grid}
\end{subfigure}
\begin{subfigure}[t]{.22\textwidth}
  \begin{tikzpicture}
  \centering    
  \node at (0,0) {\includegraphics[width=\linewidth]{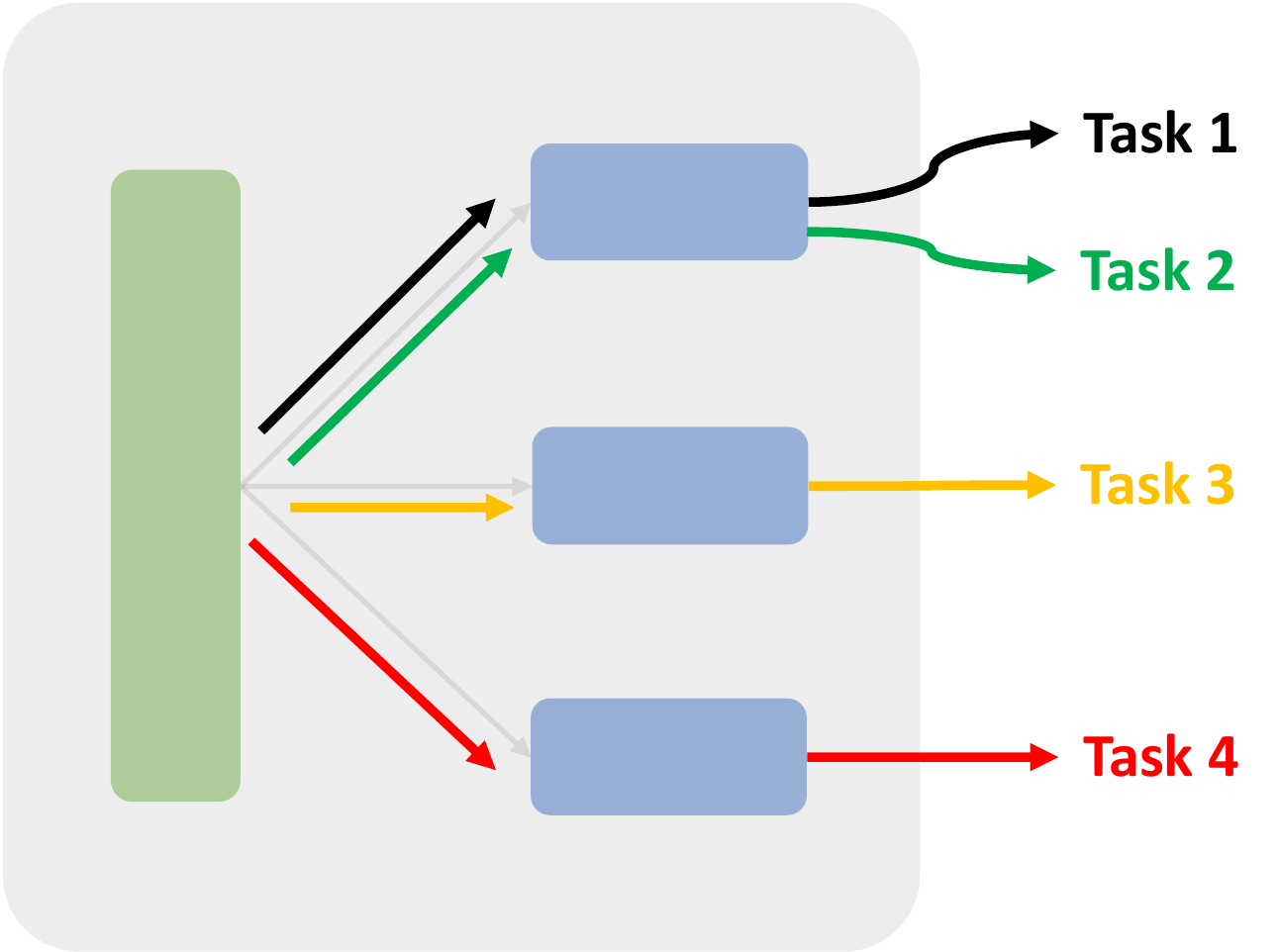}};
  \node at (-1.4, 0.0) [scale=0.7] {$\psi_1$};
  \node at (0.1, 0.85) [scale=0.7] {$\psi_2^1$};
  \node at (0.1, 0.0) [scale=0.7] {$\psi_2^2$};
  \node at (0.1, -0.85) [scale=0.7] {$\psi_2^3$};
  \end{tikzpicture}
  \caption{Hierarchical MTL}\label{fig:intro_tree}
\end{subfigure}%
\vspace{-7pt}
\caption{In \MP MTL, each task selects a pathway within a supernet graph. The composition of the modules along the pathway forms the task-specific representation. Fig.~\ref{fig:intro_grid} depicts a general supernet graph (highlighted in gray block), and the pathways for different tasks are shown in colored arrows. Fig.~\ref{fig:intro_tree} is a special instance where related tasks are hierarchically clustered: For instance, Tasks 1 and 2 are assinged the same representation $\psi_2^1\circ\psi_1$.}\label{fig:intro}
\vspace{-10pt}
\end{figure}
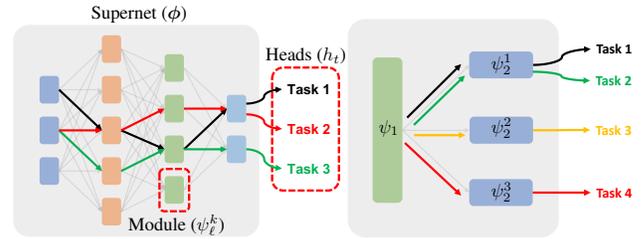

\begin{quote}
\textbf{Q:} What are the statistical benefits of learning task-specific representations along supernet pathways?
\end{quote}
Our primary contribution is formalizing the \MP MTL problem depicted in Figure \ref{fig:intro} and developing associated statistical learning guarantees that shed light on its benefits. Our formulation captures important aspects of the problem including learning compositional MTL representations, multilayer nature of supernet, assigning optimal pathways to individual tasks, and transferring learned representations to novel downstream tasks. Our specific contributions are as follows.

\noindent$\bullet$ Suppose we have $N$ samples per task and $T$ tasks in total. Denote the hypothesis sets for \Mp representation by $\Phi$, task specific heads by $\Hc$ and potential pathway choices by $\Ac$. 
{Our main result bounds the task-averaged risk of MTL as
\begin{align}
\sqrt{\frac{\DoF(\Phic)}{NT}}+\sqrt{\frac{\DoF(\Hc)+\DoF(\Ac)}{N}}.\label{basic bound}
\end{align}
Here, $\DoF(\cdot)$ returns the \emph{degrees of freedom} of a hypothesis set (i.e. number of parameters). More generally, Theorem~\ref{thm:main} states our guarantees in terms of Gaussian complexity.} $\Phic\subseteq\Phi$ is the supernet spanned by the pathways of the empirical solution and $1/NT$ dependence implies that cost of representation learning is shared across tasks. {We also show a \emph{no-harm} result (Lemma \ref{no harm}): If the supernet is sufficiently expressive to achieve zero empirical risk, then, the excess risk of individual tasks will not be harmed by the other tasks.} {Theorem \ref{thm:tfl} develops guarantees for transferring the resulting MTL representation to a new task in terms of representation bias of the empirical MTL supernet.}

\noindent$\bullet$ When the supernet has a single module, the problem boils down to (vanilla) MTL with single shared representation and our bounds recover the results by \cite{maurer2016benefit,tripuraneni2021provable}. When the supernet graph is hierarchical (as in Figure \ref{fig:intro_tree}), our bounds provide insights for the benefits of clustering tasks into similar groups and superiority of multilayer \MP MTL over using single-layer shallow supernets (Section \ref{hierarchy}). 

\noindent$\bullet$ We develop stronger results for linear representations over a supernet and obtain novel MTL and transfer learning bounds (Sec.~\ref{sec linear} and Theorem \ref{e2e thm}). These are accomplished by developing new task-diversity criteria to account for the task-specific (thus heterogeneous) nature of \Mp representations. Numerical experiments support our theory and verify the benefits of \Mp representations. Finally, we also highlight multiple future directions. 
\section{Setup and Problem Formulations}\label{sec:setup}
\noindent\textbf{Notation.} Let $\|\cdot\|$ denote the $\ell_2$-norm of a vector and operator norm of a matrix. $|\cdot|$ denotes the absolute value for scalars and cardinality for discrete sets. We use $[K]$ to denote the set $\{1,2,\dots,K\}$ and $\lesssim,\gtrsim$ for inequalities that hold up to constant/logarithmic factors. $\Qc^K$ denotes $K$-times Cartesian product of a set $\Qc$ with itself. $\circ$ denotes functional composition, i.e.,~$f\circ g(x)=f(g(x))$.


\noindent\textbf{Setup.} Suppose we have $T$ tasks each following data distribution $\{\Dc_t\}_{t=1}^T$. During MTL phase, we are given $T$ training datasets $\{\Sc_t\}_{t=1}^T$ each drawn i.i.d.~from its corresponding distribution $\Dc_t$. Let $\Sc_t=\{(\x_{ti},y_{ti})\}_{i=1}^{N}$, where $(\x_{ti},y_{ti})\in(\Xc,\R)$ is an input-label pair and $\Xc$ is the input space, and $|\Sc_t|=N$ is the number of samples per task. We assume the same $N$ for all tasks for cleaner exposition. Define the union of the datasets by $\Scb=\bigcup_{t=1}^T\Sc_t$ (with $|\Scb|=NT$), and the set of distributions by $\Dcb=\{\Dc_t\}_{t=1}^T$.

Following the setting of related works \cite{tripuraneni2021provable}, we will consider two problems: \textbf{(1) MTL problem} will use these $T$ datasets to learn a supernet and establish guarantees for representation learning. \textbf{(2) Transfer learning problem} will use the resulting representation for a downstream task in a sample efficient fashion.




%
\smallskip
\noindent \textbf{Problem (1): \MP Multitask Learning (\MTL). } We consider a supernet with $L$ layers where layer $\ell$ has $K_\ell$ modules for $\ell\in[L]$. 
As depicted in Figure~\ref{fig:intro}, each task will compose a task-specific representation by choosing one module from each layer. We refer to each sequence of $L$ modules as a \emph{pathway}.
Let $\Ac=[K_1]\times \dots\times [K_L]$ be the set of all pathway choices obeying $|\Ac|=\prod_{\ell=1}^L K_\ell$. Let $\alpha_t\in\Ac$ denote the pathway associated with task $t\in [T]$ where $\alpha_{t}[\ell]\in[K_\ell]$ denotes the selected module index from layer $\ell$. We remark that results can be extended to more general pathway sets as discussed in Section \ref{sec:main mtl}.




As depicted in Figure~\ref{fig:intro}, let $\Psi_\ell$ be the hypothesis set of modules in $\ell_\tth$ layer and $\psi_{\ell}^k\in\Psi_\ell$ denote the $k_\tth$ module function in the $\ell_\tth$ layer, referred to as ($\ell,k$)'th module. 
Let $h_t\in\Hc$ be the prediction head of task $t$ where all tasks use the same hypothesis set $\Hc$ for prediction. Let us denote the combined hypothesis 
\begin{align*}
&\h=[h_1,\dots,h_T]\in\Hc^T,\\
&\bal=[\alpha_1,\dots,\alpha_T]\in\Ac^T,\\
&\bpsi_\ell=[\psi_{\ell}^1,\dots,\psi_{\ell}^{K_\ell}]\in\Psi_\ell^{K_\ell},~\forall \ell\in[L],\\
&\bphi:=[\bpsi_1,\dots,\bpsi_L]\in\Phi
\end{align*}
where $\Phi=\Psi_1^{K_1}\times \dots\times\Psi_L^{K_L}$ is the supernet hypothesis class containing all modules/layers. Given a supernet $\bphi\in\Phi$ and pathway $\alpha$, $\bphi_{\alpha}=\psi_{L}^{\alpha}\circ\dots\circ\psi_1^{\alpha}$ denotes the representation induced by pathway $\alpha$ where we use the convention $\psi_{\ell}^{\alpha}:=\psi_{\ell}^{\alpha[\ell]}$. Hence, $\bphi_{\alpha_t}$ is the representation of task $t$. We would like to solve for supernet weights $\bphi$, pathways $\bal$, and heads $\h$. Thus, given a loss function $\ell(\hat y,y)$, \MP MTL (\MTL) solves the following empirical risk minimization problem over $\Scb$ to optimize the combined hypothesis $\f=(\h,\bal,\bphi)$: 
\begin{align}
   \hat\f=\underset{\f\in\Fc}{\arg\min}~&\widehat\Lc_{\Scb}(\f):=\frac{1}{T}\sum_{t=1}^T\widehat\Lc_t(h_t\circ\bphi_{\alpha_t})\label{formula:multipath}\tag{\MTL}\\
   \text{where}~~~&\widehat\Lc_t(f)=\frac{1}{N}\sum_{i=1}^N\ell(f(\x_{ti}),y_{ti})\nn\\
   &\Fc:=\Hc^T\times\Ac^T\times\Phi.\nn
\end{align}
Here $\widehat\Lc_t$ and $\widehat\Lc_{\Scb}$ are task-conditional and task-averaged empirical risks. We are primarily interested in controlling the task-averaged test risk $\Lc_{\Dcb}(\f)=\E[\widehat\Lc_{\Scb}(\f)]$. Let $\Lc_{\Dcb}^\star:=\min_{\f\in\Fc}\Lc_{\Dcb}(\f)$, then the \emph{excess MTL risk} is defined as
\begin{align}
    \RMTL(\hat\f)=\Lc_{\Dcb}(\hat\f)-\Lc_{\Dcb}^\star.\label{risk mtl}
\end{align}




\smallskip
\noindent \textbf{Problem (2): Transfer Learning with Optimal Pathway (\TFL). } 
Suppose we have a novel target task with i.i.d.~training dataset $\Sc_\tgt=\{(\x_i,y_i)\}_{i=1}^M$ with $M$ samples drawn from distribution $\Dc_\tgt$. 
Given a pretrained supernet $\bphi$ (e.g.,~following \eqref{formula:multipath}), we can search for a pathway $\alpha$ so that $\bphi_{\alpha}$ becomes a suitable representation for $\Dc_\tgt$. 
Thus, for this new task, we only need to optimize the path $\alpha\in \Ac$ and the prediction head $h\in\Hc_\tgt$ while reusing weights of $\bphi$. This leads to the following problem:
\begin{align}
    \hat f_{\bphi}=\underset{h\in\Hc_\tgt,\alpha\in\Ac}{\arg\min}&\widehat\Lc_{\tgt}(f)~~~\text{where}~~~f=h\circ\bphi_{\alpha}\label{formula:transfer}\tag{\TFL}\\
    \text{and}~~~&\widehat\Lc_{\tgt}(f)=\frac{1}{M}\sum_{i=1}^M\ell(f(\x_i),y_i).\nn
\end{align}
Here, $\hat f_{\bphi}$ reflects the fact that solution depends on the suitability of pretrained supernet $\bphi$.
 Let $f^\star_\bphi$ be a population minima of (\ref{formula:transfer}) given supernet $\bphi$ (as $M\rightarrow\infty$) and define the population risk $\Lc_{\tgt}(f)=\E[\widehat\Lc_{{\tgt}}(f)]$. 
 {
 \eqref{formula:transfer} will be evaluated against the hindsight knowledge of optimal supernet for target: Define the optimal target risk $\Lc_\tgt^\st:=\min_{h\in\Hc_\tgt,\bphi\in\Phi}\Lc_\tgt(h\circ \bphi_{\alpha})$ which optimizes $h,\bphi$ for the target task along the fixed pathway $\alpha=[1,\dots,1]$. Here we can fix $\alpha$ since all pathways result in the same search space.}
We define the \emph{excess transfer learning risk} to be
\begin{align}
\RTFL&(\hat f_{\bphi})=\Lc_{\tgt}(\hat f_{\bphi})-\Lc_{\tgt}^\st\label{TFL risk}\\
    &=\underset{\text{variance}}{\underbrace{\Lc_{\tgt}(\hat f_{\bphi})-\Lc_{\tgt}(f^\star_{\bphi})}}+\underset{\text{\bias}}{\underbrace{\Lc_{\tgt}( f^\star_{\bphi})-\Lc_{\tgt}^\st}}.\nn
\end{align}
 The final line decomposes the overall risk into a \emph{variance} term and \emph{\bias}. The former arises from the fact that we solve the problem with finite training samples. This term will vanish as $M\rightarrow\infty$. The latter term quantifies the bias induced by the fact that \eqref{formula:transfer} uses the representation $\bphi$ {rather than the optimal representation}. Finally, while supernet $\bphi$ in \eqref{formula:transfer} is arbitrary, for end-to-end guarantees we will set it to the solution $\hat\bphi$ of \eqref{formula:multipath}. In this scenario, we will refer to $\{\Dc_t\}_{t=1}^T$ as source tasks. 

\section{Main Results}\label{sec:main}

We are ready to present our results that establish generalization guarantees for multitask and transfer learning problems over supernet pathways. Our results will be stated in terms of Gaussian complexity which is introduced below.


\begin{definition}
[Gaussian Complexity]\label{def:worst-case} 
Let $\Qc$ be a set of hypotheses that map $\Zc$ to $\R^r$. Let $(\g_{i})_{i=1}^n$ ($\g_i\in\R^r$) be $n$ independent vectors each distributed as $\Nn(\boldsymbol{0},\Iden_r)$ and let $\Z=(\z_i)_{i=1}^n\in\Zc^n$ be a dataset of input features. Then, the empirical Gaussian complexity is defined as
\[
\Gh_\Z(\Qc)=\E_{\g_{i}}\left[\sup_{q\in\Qc}\frac{1}{n}\sum_{i=1}^n\g_{i}^\top q(\z_i)\right].
\]
The worst-case Gaussian complexity is obtained by considering the supremum over $\Z\in\Zc^n$ as follows
\begin{align}
    \Gt^\Zc_n(\Qc)=\sup_{\Z\in\Zc^n}[\Gh_\Z(\Qc)].\nonumber
\end{align}
\end{definition}

For cleaner notation, we drop the superscript $\Zc$ from the worst-case Gaussian complexity (using $\Gt_n(\Qc)$) as its input space will be clear from context. When $\Z=(\z_i)_{i=1}^n$ are drawn i.i.d.~from $\Dc$, the (usual) Gaussian complexity is defined by $\Gc_n(\Qc)=\E_{\Z\sim\Dc^n}[\Gh_\Z(\Qc)]$. Note that, we always have $\Gc_n(\Qc)\leq\Gt_n(\Qc)$ assuming $\Dc$ is supported on $\Zc$. In our setting, keeping track of distributions along exponentially many pathways proves challenging, and we opt to use $\Gt_n(\Qc)$ which leads to clean upper bounds. The supplementary material also derives tighter but more convoluted bounds in terms of empirical complexity. Finally, it is well-known that Gaussian/Rademacher complexities scale as $\sqrt{{\comp(\Qc)}/{n}}$ where $\comp(\Qc)$ is a set complexity such as VC-dimension, which links to our informal statement \eqref{basic bound}.


We will first present our generalization bounds for the \MMTL problem using empirical process theory arguments. Our bounds will lead to meaningful guarantees for specific MTL settings, including vanilla MTL where all tasks share a single representation, as well as hierarchical MTL depicted in Fig.~\ref{fig:intro_tree}. We will next derive transfer learning guarantees in terms of supernet bias, which quantifies the performance difference of a supernet from its optimum for a target. 
To state our results, we introduce two standard assumptions. 
%


    
\begin{assumption}\label{assum:lip1}
    Elements of hypothesis sets $\Hc$ and $(\Psi_{\ell})_{\ell=1}^L$ are $\Gamma$-Lipschitz functions with respect to Euclidean norm.
\end{assumption}
\begin{assumption}\label{assum:lip2}
    Loss function {$\ell(\cdot,y):\R\times\R\rightarrow[0,1]$} is $\Gamma$-Lipschitz with respect to Euclidean norm. 
\end{assumption}

\subsection{Results for Multipath Multitask Learning}\label{sec:main mtl}
This section presents our task-averaged generalization bound for \MP MTL problem. Recall that $\hat\f=(\hat\h,\hat\bal,\hat\bphi)$ is the outcome of the ERM problem \eqref{formula:multipath}. Observe that, if we were solving the problem with only one task, the generalization bound would depend on only one module per layer rather than the overall size of the supernet. This is because each task gets to select a single module through their pathway. In light of this, we can quantify the utilization of supernet layers as follows: Let $\hat{K}_\ell$ be the number of modules utilized by the empirical solution $\hat\f$. Formally, $\hat{K}_\ell=|\{\hat\alpha_t[\ell]~~\text{for}~~ t\in [T]\}|$. The following theorem provides our guarantee in terms of Gaussian complexities of individual modules.



\begin{theorem}\label{thm:main}
    Suppose Assumptions 1\&2 hold. Let $\hat\f$ be the empirical solution of \eqref{formula:multipath}. Then, with probability at least $1-\delta$, the excess test risk in \eqref{risk mtl} obeys $\RMTL(\hat\f)$
    \begin{align*}
        \lesssim\Gt_N(\Hc)+\sum_{\ell=1}^L\sqrt{\hat{\K}_\ell}\Gt_{NT}(\Psi_\ell)+\sqrt{\frac{\log|\Ac|}{N}+\frac{\log(2/\delta)}{NT}}.
    \end{align*}
Here, the input spaces for $\Hc$ and $\Psi_{\ell}$ are $\Xc_\Hc=\Psi_L\circ\dots\Psi_1\circ\Xc$, $\Xc_{\Psi_\ell}=\Psi_{\ell-1}\circ\dots\Psi_1\circ\Xc$ for $\ell>1$, and $\Xc_{\Psi_1}=\Xc$. 
\end{theorem}

In Theorem \ref{thm:main}, $\sqrt{\frac{\log|\Ac|}{N}}$ quantifies the cost of learning the pathway and $\Gt_N(\Hc)$ quantifies the cost of learning the prediction head for each task $t\in [T]$. $\log|\Ac|$ dependence is standard for the discrete search space $|\Ac|$. The $\Gt_{NT}(\Psi_\ell)$ terms are more interesting and reflect the benefits of MTL. The reason is that, these modules are essentially learned with $NT$ samples rather than $N$ samples, thus cost of representation learning is shared across tasks. The $\sqrt{\hat{\K}_\ell}$ multiplier highlights the fact that, we only need to worry about the used modules rather than all possible $K_\ell$ modules we could have used. In essence, $\sum_{\ell=1}^L\sqrt{\hat{\K}_\ell}\Gt_{NT}(\Psi_\ell)$ summarizes the Gaussian complexity of $\Gt(\Phi_{\text{used}})$ where $\Phi_{\text{used}}$ is the subnetwork of the supernet utilized by the ERM solution $\hat\f$. By definition $\Gt(\Phi_{\text{used}})\leq \Gt(\Phi)$. With all these in mind, Theorem \ref{thm:main} formalizes our earlier statement \eqref{basic bound}.

A key challenge we address in Theorem \ref{thm:main} is decomposing the complexity of the combined hypothesis class $\Fc$ in \eqref{formula:multipath} into its building blocks $\Ac,\Hc,(\Psi_\ell)_{\ell=1}^L$. This is accomplished by developing Gaussian complexity chain rules inspired from the influential work of \cite{tripuraneni2020theory,maurer2016chain}. While this work focuses on two layer composition (prediction heads composed with a shared representation), we develop bounds to control arbitrarily long compositions of hypotheses. Accomplishing this in our \Mp setting presents additional technical challenges because each task gets to choose a unique pathway. Thus, tasks don't have to contribute to the learning process of each module unlike the vanilla MTL with shared representation. Consequently, ERM solution is highly heterogeneous and some modules and tasks will be learned better than the others. Worst-case Gaussian complexity plays an important role to establish clean upper bounds in the face of this heterogeneity. {In fact, in supplementary material, we provide tighter bounds in terms of empirical Gaussian complexity $\Gh$, however, they necessitate more convoluted definitions that involve the number of tasks that choose a particular module.}

Finally, we note that our bound has a natural interpretation for parametric classes whose $\log(\eps\text{-covering number})$ (i.e.~metric entropy) grows with {degrees of freedom as $\DoF\cdot\log(1/\eps)$.} Then, Theorem \ref{thm:main} implies a risk bound proportional to $\sqrt{\frac{T\cdot (\DoF(\Hc)+\log |\Ac|)+\sum_{\ell=1}^L\hat{\K}_\ell\cdot\DoF(\Psi_\ell)}{NT}}$. For a neural net implementation, this means small risk as soon as total sample size $NT$ exceeds total number of weights.

We have a few more remarks in place, discussed below.


\smallskip 
\noindent $\bullet$ \textbf{Dependencies.} In Theorem \ref{thm:main}, $\lesssim$ suppresses dependencies on $\log(NT)$ and $\Gamma^L$. The latter term arises from the exponentially growing Lipschitz constant as we compose more/deeper modules, however, it can be treated as a constant for fixed depth $L$. We note that such exponential depth dependence is frequent in existing generalization guarantees in deep learning literature \cite{golowich2018size,bartlett2017spectrally,neyshabur2018towards,neyshabur2017exploring}.
{In supplementary material, we prove that the exponential dependence can be replaced with a much stronger bound of $\sqrt{L}$ by assuming parameterized hypothesis classes.}

\smallskip 
\noindent$\bullet$ \textbf{Implications for Vanilla MTL.} Observe that Vanilla MTL with single shared representation corresponds to the setting $L=1$ and $K_1=1$. Also supernet is simply $\Phi=\Psi_1$ and $\log |\Ac|=0$. Applying Theorem \ref{thm:main} to this setting with $T$ tasks each with $N$ samples, we obtain an excess risk upper bound of $\ordet{\Gt_{NT}(\Phi)+\Gt_N(\Hc)}$, where representation $\Phi$ is trained with $NT$ samples with input space $\Xc$, and task-specific heads $h_t\in\Hc$ are trained with $N$ samples with input space $\Phi\circ\Xc$. This bound recovers earlier guarantees by \cite{maurer2016benefit,tripuraneni2020theory}.


\smallskip 
\noindent$\bullet$ \textbf{Unselected modules do not hurt performance.} A useful feature of our bound is its dependence on $\Phi_{\text{used}}$ (spanned by empirical pathways) rather than full hypothesis class $\Phi$. This feature arises from a uniform concentration argument where we uniformly control the excess MTL risk over all potential $\Phi_{\text{used}}$ choices. This uniform control ensures $\Gt_{NT}(\Phi_{\text{used}})$ cost for the actual solution $\hat{\f}$ and it only comes at the cost of an additional $\sqrt{\frac{\log|\Ac|}{N}}$ term which is free (up to constant)!



\smallskip 
{\noindent$\bullet$ \textbf{Continuous pathways.}} This work focuses on relatively simple pathways where tasks choose one module from each layer. The results can be extended to other choices of pathway sets $\Ac$. First, note that, as long as $\Ac$ is a discrete set, we will naturally end up with the excess risk dependence of $\sqrt{\frac{\log|\Ac|}{N}}$. However, one can also consider continuous $\alpha$, for instance, due to relaxation of the discrete set with a simplex constraint. Such approaches are common in differentiable architecture search methods \cite{liu2018darts}. In this case, each entry $\alpha[\ell]$ can be treated as a $K_\ell$ dimensional vector that chooses a continuous superposition of $\ell$'th layer modules. {Thus, the overall $\alpha\in\Ac$ parameter would have $\comp(\Ac)=\sum_{\ell=1}^L K_\ell$ resulting in an excess risk term of $\sqrt{{\sum_{\ell=1}^L K_\ell}/{N}}$.} Note that, these are high-level insights based on classical generalization arguments. In practice, performance can be much better than these uniform concentration based upper bounds.


\smallskip 
\noindent$\bullet$ \textbf{No harm under overparameteration.} A drawback of Theorem \ref{thm:main} is that, it is an average-risk guarantee over $T$ tasks. In practice, it is possible that some tasks are hurt during MTL because they are isolated or dissimilar to others (see supplementary for examples). Below, we show that, if the supernet achieves zero empirical risk, then, no task will be worse than the scenario where they are individually trained with $N$ samples, i.e.~\MP MTL does not hurt any task.
\begin{lemma}\label{no harm}
    Recall $\hat\f$ is the solution of \eqref{formula:multipath} and $\hat f_t=\hat h_t\circ\hat\bphi_{\hat\alpha_t}$ is the associated task-$t$ hypothesis. Define the excess risk of task $t$ as $\Rt(\hat f_t)=\Lc_t(\hat f_t)-\Lc^\st_t$ where $\Lc_t(f)=\E_{\Dc_t}[\widehat\Lc_t(f)]$ is the population risk of task $t$ and $\Lc^\st_t$ is the optimal achievable test risk for task $t$ over $\Fc$.  With probability at least $1-\delta-\P(\widehat\Lc_{\Scb}(\hat\f)\neq0)$, for all tasks $t\in[T]$,
    \begin{align*}
        \Rt(\hat f_t)\lesssim\Gt_N(\Hc)+\sum_{\ell=1}^L\Gt_N(\Psi_\ell)+\sqrt{\frac{\log(2T/\delta)}{N}}.
    \end{align*}
\end{lemma}
Here, $\P(\widehat\Lc_{\Scb}(\hat\f)=0)$ is the event of interpolation (zero empirical risk) under which the guarantee holds. We call this \emph{no harm} because the bound is same as what one would get by applying union bound over $T$ empirical risk minimizations where each task is optimized individually.

\subsection{\OTFL}\label{sec:main tfl}
Following \MMTL problem, in this section, we discuss guarantees for transfer learning on a supernet. Recall that $\Ac$ is the set of pathways and our goal in \eqref{formula:transfer} is finding the optimal pathway $\alpha\in\Ac$ and prediction head $h\in\Hc_\tgt$ to achieve small target risk. In order to quantify the bias arising from the \MMTL phase, we introduce the following definition.
\begin{definition} [Supernet Bias]\label{def:model distance} 
    Recall the definitions $\Dc_\tgt$, $\Hc_\tgt$, and $\Lc^\st_\tgt$ stated in Section~\ref{sec:setup}. Given a supernet $\bphi$, we define the supernet/representation bias of $\bphi$ for a target $\tgt$ as
    \begin{align*}
        \Bias_{\tgt}(\bphi)=\min_{h\in\Hc_\tgt,\alpha\in\Ac}\Lc_\tgt(h\circ\bphi_\alpha)-\Lc^\star_\tgt.
    \end{align*}
\end{definition}

Definition~\ref{def:model distance} is a restatement of the \bias term in \eqref{TFL risk}. Importantly, it ensures that the optimal pathway-representation over $\bphi$ can not be worse than the optimal performance by $\Bias_\tgt(\bphi)$. 
Following this, we can state a generalization guarantee for {transfer learning problem} (\ref{formula:transfer}).
\begin{theorem}\label{thm:tfl} Suppose Assumptions 1\&2 hold. Let supernet $\hat\bphi$ be the solution of \eqref{formula:multipath} and $\hat f_{\hat{\bphi}}$ be the empirical minima of \eqref{formula:transfer} with respect to supernet $\hat\bphi$. Then with probability at least $1-\delta$,
    \begin{align*}
        \RTFL(\hat f_{\hat\bphi})\lesssim \Bias_{\tgt}(\hat\bphi)+\sqrt{\frac{\log(2|\Ac|/\delta)}{M}}+\Gt_M(\Hc_\tgt),
    \end{align*}
    where input space of $\Gt_M(\Hc_\tgt)$ is given by $\{\hat\bphi_\alpha\circ\Xc\bgl \alpha\in\Ac\}$.
\end{theorem}

Theorem~\ref{thm:tfl} highlights the sample efficiency of transfer learning with optimal pathway. While the derivation is straightforward relative to Theorem \ref{thm:main}, the key consideration is the supernet bias $\Bias_\tgt(\hat\bphi)$. This term captures the excess risk in \eqref{formula:transfer} introduced by using $\hat\bphi$. Let $\bphi^\st$ be the population minima of \eqref{formula:multipath}. Then we can define the \emph{supernet distance} of $\hat\bphi$ and $\bphi^\st$ by $d_\tgt(\hat\bphi;\bphi^\st)=\Bias_\tgt(\hat\bphi)-\Bias_\tgt(\bphi^\st)$. The distance measures how well the finite sample solution $\hat\bphi$ from \eqref{formula:multipath} performs compared to the optimal MTL solution $\bphi^\st$. A plausible assumption is so-called \emph{task diversity} proposed by \citet{chen2021weighted,tripuraneni2020theory,xu2021representation}. Here, the idea (or assumption) is that, if a target task is similar to the source tasks, the distance term for target can be controlled in terms of the excess MTL risk $\RMTL(\hat\f)$ (e.g.~by assuming $d_\tgt(\hat\bphi;\bphi^\st)\lesssim \RMTL(\hat\f)+\eps$). Plugging in this assumption would lead to end-to-end transfer guarantees by integrating Theorems \ref{thm:main} and \ref{thm:tfl}, and we extend the formal analysis to appendix. However, as discussed in Theorem \ref{e2e thm}, in \Mp setting, the problem is a lot more intricate because source tasks can choose totally different task-specific representations making such assumptions unrealistic. In contrast, Theorem \ref{e2e thm} establishes concrete guarantees by probabilistically relating target and source distributions. Finally, $\Bias_\tgt(\bphi^\st)$ term is unavoidable, however, similar to $d_\tgt(\hat\bphi;\bphi^\st)$, it will be small as long as source and target tasks benefit from a shared supernet at the population level.

\section{Guarantees for Linear Representations}\label{sec linear}

As a concrete instantiation of Multipath MTL, consider a linear representation learning problem where each module $\psi_{\ell}^k$ applies matrix multiplications parameterized by $\B_{\ell}^k$ with dimensions $p_{\ell}\times p_{\ell-1}$: $\psi_{\ell}^k(\x)=\B_{\ell}^k\x$. Here $p_\ell$ are module dimensions with input dimension $p_0=p$ and output dimension $p_{L}$. Given a path $\alpha$, {we obtain the linear representation $\B_{\alpha}=\Pi_{\ell=1}^L \B_{\ell}^{\alpha[\ell]}\in\R^{p_L\times p}$ where $p_L$ is the number of rows of the final module $\B_{L}^{\alpha[L]}$}. When $p_L\ll p$, $\B_{\alpha}$ is a fat matrix that projects $\x\in\R^p$ onto a lower dimensional subspace. This way, during few-shot adaptation, we only need to train $p_L\ll p$ parameters with features $\B_{\alpha}\x$. This is also the central idea in several works on linear meta-learning \cite{kong2020robust,sun2021towards,bouniot2020towards,tripuraneni2021provable} which focus on a single linear representation. Our discussion within this section extends these results to the \MP MTL setting. 

Denote  $\fb=\{((\B_{\ell}^k)_{k=1}^{K_\ell})_{\ell=1}^L,(\h_t,\alpha_t)_{t=1}^T\}$ where $\h_t\in\R^{p_{L}}$ are linear prediction heads. Let $\Fc$ be the search space associated with $\fb$. {Follow the similar setting as in Section~\ref{sec:setup} and let $\Xc\subset\R^p$.} Given dataset $\Scb=(\Sc_t)_{t=1}^T$, we study
\vspace{-3pt}
\begin{align}
\hat{\fb}=\min_{\fb\in\Fc} \Lch_{\Scb}(\fb):=\frac{1}{NT}\sum_{t=1}^T\sum_{i=1}^N (y_{ti}-{\h^{\top}_t} \B_{\alpha_t}\x_{ti})^2.\label{LMP-MTL}
\end{align}
Let $\Bc^p(r)\subset\R^p$ be the Euclidean ball of radius $r$. To proceed, we make the following assumption for a constant $C\geq 1$.
\begin{assumption}\label{bounded assume} For all $\ell\in [L]$, $\Psi_\ell$ is the set of matrices with operator norm bounded by $C$ and $\Hc=\Bc^{p_{L}}(C)$.
\end{assumption}
{The result below is a variation of Theorem \ref{thm:main} where the bound is refined for linear representations (with finite parameters).}

\begin{theorem}\label{cor linear} Suppose Assumptions \ref{assum:lip2}\&\ref{bounded assume} hold, and input set $\Xc\subset \Bc^p({R})$ for a constant $R>0$. Then, with probability at least $1-\delta$, 
\[
\RMTL(\hat\f)\lesssim\sqrt{\frac{L\cdot\DoF(\Fc)}{NT}}+\sqrt{\frac{\log|\Ac|}{N}+\frac{\log(2/\delta)}{NT}},
\]
where $\DoF(\Fc)=T\cdot p_L+\sum_{\ell=1}^L K_\ell \cdot p_\ell \cdot p_{\ell-1}$ is the total number of trainable parameters in $\Fc$.
\end{theorem}

We note that Theorem \ref{cor linear} can be stated more generally for neural nets by placing ReLU activations between layers. {Here $\lesssim$ subsumes the logarithmic dependencies, and the sample complexity has linear dependence on $L$ (rather than exponential dependence as in Thm~\ref{thm:main}).} In essence, it implies small task-averaged excess risk as soon as $\text{total sample size}\gtrsim \text{total number of weights}$. 

While flexible, this result does not guarantee that $\hat\f$ can benefit transfer learning for a new task. To proceed, we introduce additional assumptions under which we can guarantee the success of \eqref{formula:transfer}. The first assumption is a realizability condition that guarantees tasks share same supernet representation (so that supernet bias is small).

\begin{assumption} \label{linear dist} \textbf{(A)} Task datasets are generated from a planted model ($\x_t,y_t)\sim\Dc_t$ where $y_t=\x_t^\top \bt_t^\star +z_t$ where $\x_t,z_t$ are zero mean, $\order{1}$-subgaussian and $\E[\x_t\x_t^\top]=\Iden_p$. \\
 \textbf{(B)} Task vectors are generated according to ground-truth supernet $\f^\star=\{((\bar{\B}_{\ell}^k)_{k=1}^{K_\ell})_{\ell=1}^L,(\bar{\h}_t,\bar{\alpha}_t)_{t=1}^T\}$ so that $\bt_t^\star=\bar{\B}^\top_{\bar{\alpha}_t}\bar{\h}_t$. $\f^\star$ is normalized so that $\|\bar{\B}_{\ell}^k\|=\|\bar{\h}_t\|=1$.
\end{assumption}
 Our second assumption is a task diversity condition adapted from \cite{tripuraneni2021provable,kong2020meta} that facilitates the identifiability of the ground truth supernet.
\begin{assumption} [Diversity during MTL] \label{mtldiverse}Cluster the tasks by their pathways via $\Hb_\alpha=\{\bar{\h}_t\bgl \bar{\alpha}_t=\alpha\}$. Define cluster population $\gamma_\alpha={|\Hb_\alpha|}/{p_L}$ and covariance $\bSi_\alpha=\gamma_\alpha^{-1}\sum_{\h\in\Hb_\alpha}\h\h^\top$. For a proper constant $c>0$ and for all pathways $\alpha$ we have $\bSi_\alpha\succeq c\Iden_{p_L}$.
\end{assumption}
Verbally, this condition requires that, if a pathway is chosen by a source task, that pathway should contain diverse tasks so that \eqref{formula:multipath} phase can learn a good representation that can benefit transfer learning. However, this definition is flexible in the sense that pathways can still have sophisticated interactions/intersections and we don't assume anything for the pathways that are not chosen by source. We also have the challenge that, some pathways can be a lot more populated than others and  target task might suffer from poor MTL representation quality over less populated pathways. The following assumption is key to overcoming this issue by enforcing a distributional prior on the target task pathway so that \emph{its pathway is similar to the source tasks in average}.


\begin{assumption}[Distribution of target task] \label{transfer dist}Draw $\alpha_\tgt$ uniformly at random from source pathways $(\bar{\alpha}_t)_{t=1}^T$. Target task is distributed as in Assumption \ref{linear dist}\textbf{(A)} with pathway $\alpha_\tgt$ and $\bt^\st_\tgt=\bar{\B}^\top_{\alpha_\tgt}\h_\tgt$ with $\tn{\h_\tgt}=1$. 
\end{assumption}
With these assumptions, we have the following result that guarantees end-to-end \Mp learning (\eqref{formula:multipath} phase followed by \eqref{formula:transfer} using MTL representation).
\begin{theorem} \label{e2e thm}Suppose Assumptions \ref{bounded assume}--\ref{transfer dist} hold and $\ell(\hat y,y)=(y-\hat y)^2$. {Additionally assume input set $\Xc\subset\Bc^p(R)$ for a constant $R>0$ and $\Hc_\Tc\subset\R^{p_L}$}. Solve MTL problem \eqref{formula:multipath} with the knowledge of ground-truth pathways $(\bar{\alpha}_t)_{t=1}^T$ to obtain a supernet $\hat\bphi$ {and $NT\gtrsim\DoF(\Fc)\log(NT)$}. Solve transfer learning problem \eqref{formula:transfer} with $\hat\bphi$ to obtain a target hypothesis $\hat{f}_{\hat\bphi}$. Then, with probability at least {$1-3e^{-cM}-\delta$}, path-averaged excess target risk \eqref{TFL risk} obeys $\E_{\alpha_{\Tc}}[\RTFL(\hat f_{\hat\bphi})]$
\[
    \lesssim p_L\sqrt{\frac{{L}\cdot\DoF(\Fc)+\log(8/\delta)}{NT}}+\frac{p_L}{M}+\sqrt{\frac{\log(8|\Ac|/\delta)}{M}}.
\]
Here $\DoF(\Fc)=T\cdot p_L+\sum_{\ell=1}^L K_\ell \cdot p_\ell \cdot p_{\ell-1}$, and {$\E_{\alpha_\Tc}$ denotes the expectation over the random target pathways.}
\end{theorem}
In words, this result controls the target risk in terms of the sample size of the target task and sample size during multitask representation learning, and provides a concrete instantiation of discussion following Theorem \ref{thm:tfl}. {In Theorem~\ref{e2e thm2} in appendix, we provide a tighter bound for expected transfer risk when linear head $\h_\Tc$ is uniformly drawn from the unit sphere.} The primary challenge in our work compared to related vanilla MTL results by \cite{tripuraneni2021provable,du2020few,kong2020meta} is the fact that, we deal with exponentially many pathway representations many of which may be low quality. Assumption \ref{transfer dist} allows us to convert task-averaged MTL risk into a transfer learning guarantee over a \emph{random pathway}. Finally, Theorem \ref{e2e thm} assumes that source pathways are known during MTL phase. {In Appendix~\ref{app:tfail2}, we show that} this assumption is indeed necessary: Otherwise, one can construct scenarios where \eqref{formula:multipath} problem admits an alternative solution $\tilde{\f}$ with optimal MTL risk but the resulting supernet $\tilde\bphi$ achieves poor target risk. Supplementary material discusses this challenge and identifies additional conditions that make ground-truth pathways uniquely identifiable when we solve \eqref{formula:multipath}.
%
%
\begin{figure*}[t]
\vspace{-11pt}
\centering
\begin{subfigure}[t]{.33\textwidth}
  \centering
  \includegraphics[width=\linewidth]{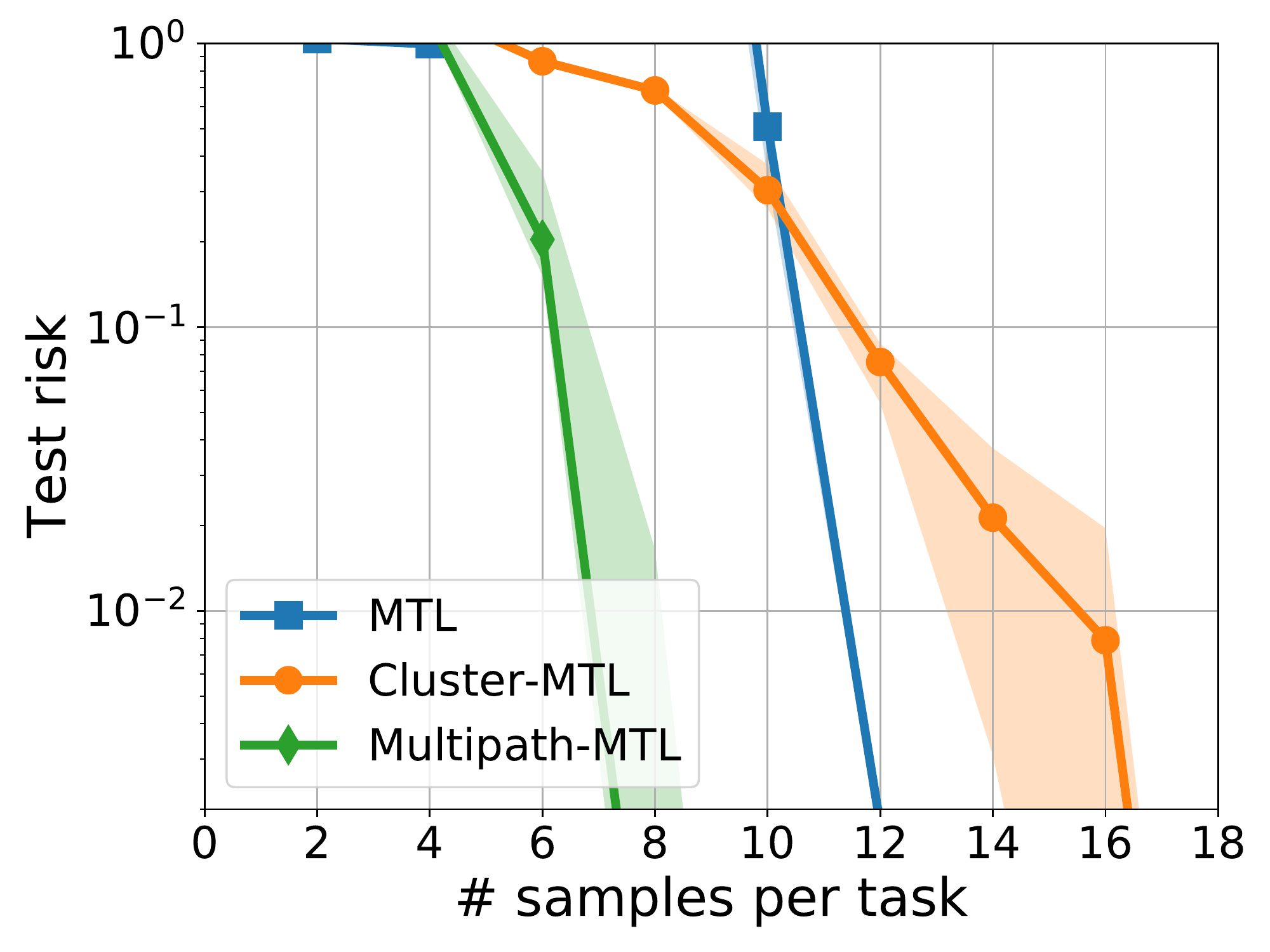}
  \caption{Varying $n$ with $\Tbar=10,K=40$}\label{fig:num_sample}
\end{subfigure}
\begin{subfigure}[t]{.33\textwidth}
  \centering    \includegraphics[width=\linewidth]{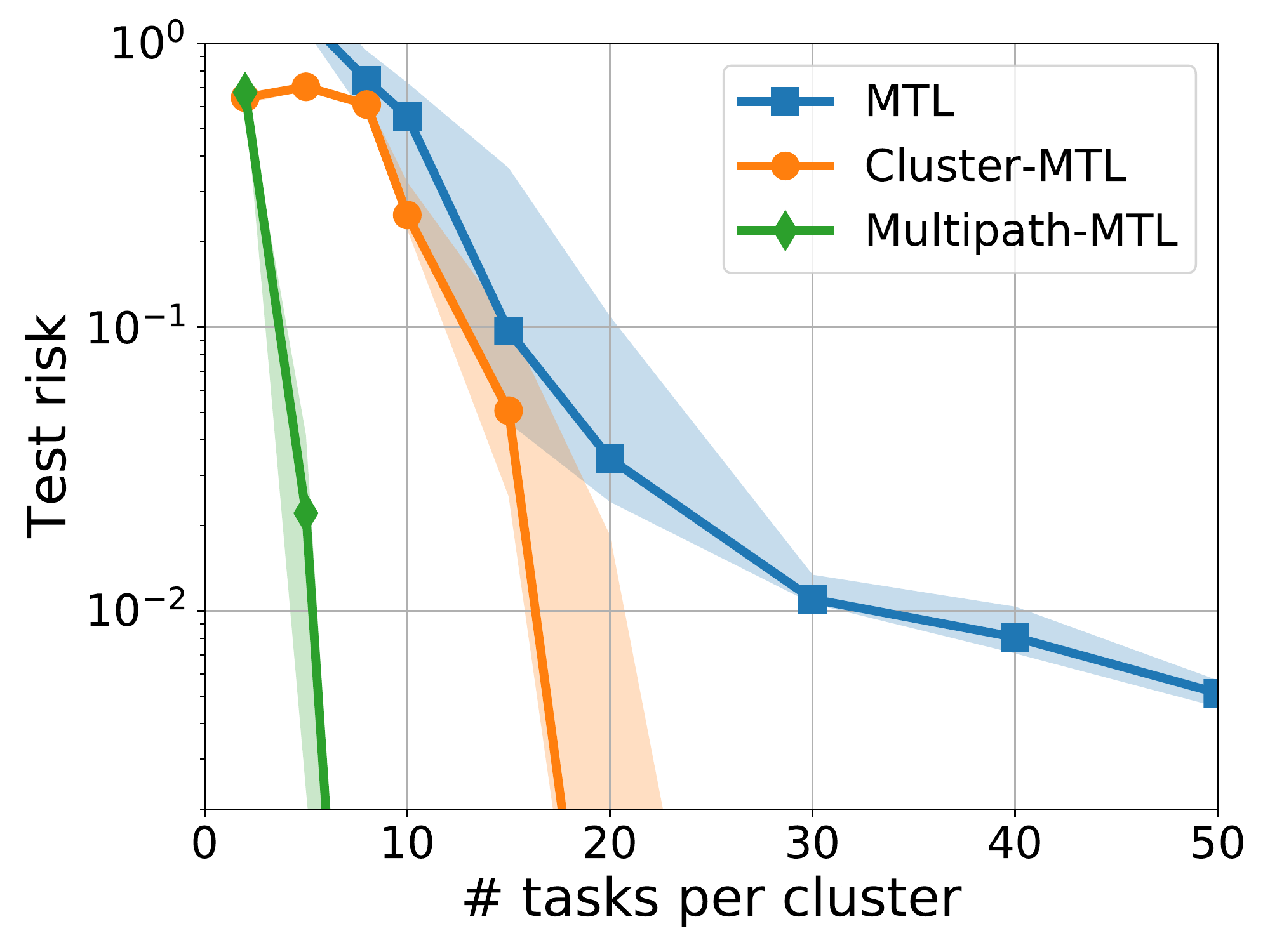}
  \caption{Varying $\Tbar$ with $N=10,K=40$}\label{fig:num_task_per_cluster}
\end{subfigure}
\begin{subfigure}[t]{.33\textwidth}
  \centering    \includegraphics[width=\linewidth]{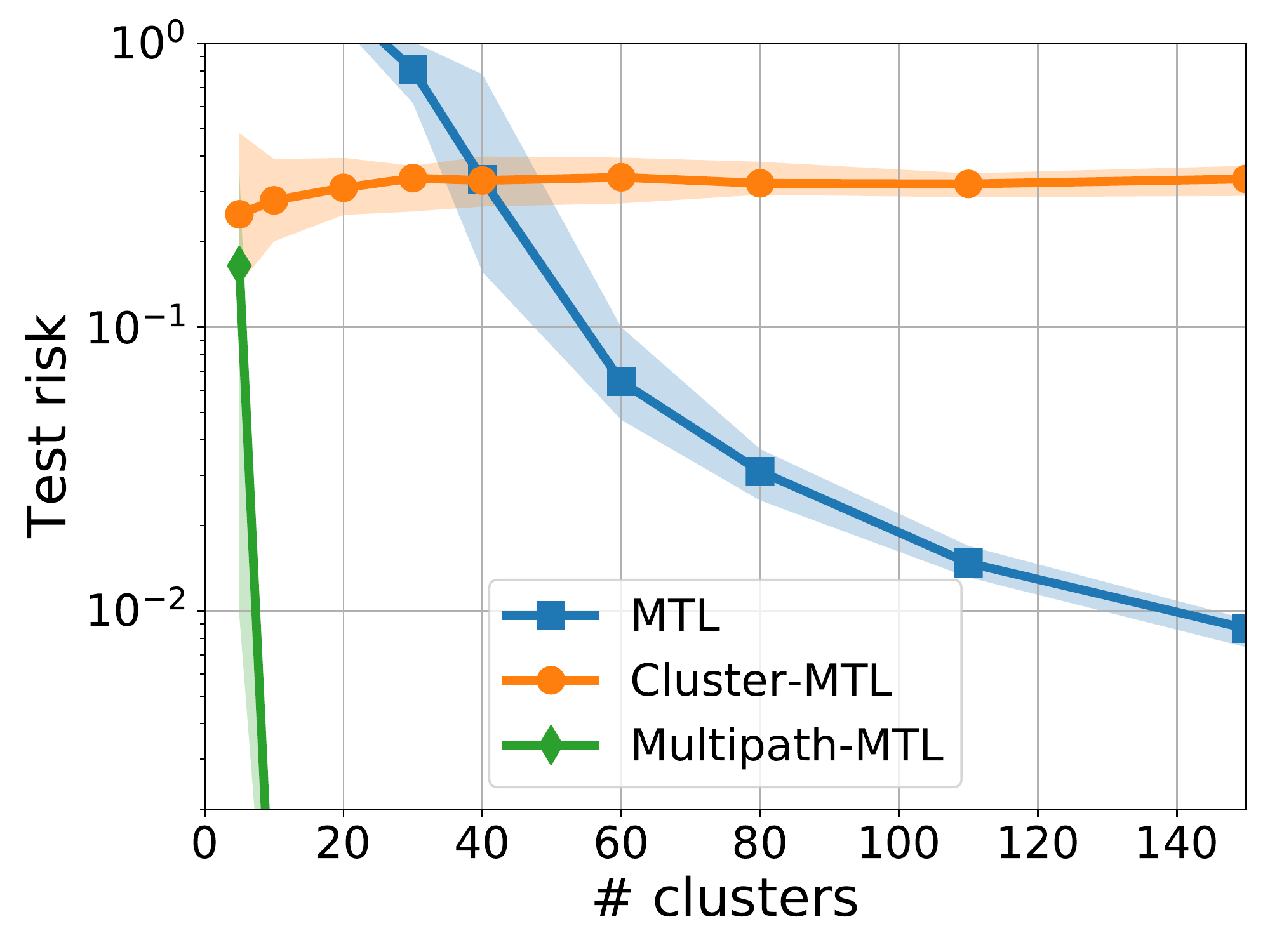}
  \caption{Varying $K$ with $N=10,\Tbar=10$}\label{fig:num_cluster}
\end{subfigure}
\vspace{-7pt}
\caption{We compare the sample complexity of MTL, Cluster-MTL and Multipath-MTL in a noiseless linear regression setting. For each figure, we fix two of the configurations and vary the other one. We find that Multipath-MTL is superior to both baselines of MTL and Cluster-MTL as predicted by our theory. The solid curves are the median risk and the shaded regions highlight the first and third quantile risks. Each marker is obtained by averaging 20 independent realizations.}\label{fig:MTLcomparison}
\vspace{-10pt}
\end{figure*}
\section{Insights from Hierarchical Representations}\label{hierarchy}
We now discuss the special two-layer supernet structure depicted in Figure \ref{fig:intro_tree}. {This setting groups tasks into $K:=K_2$ clusters and first layer module is shared across all tasks ($K_1=1$). Ignoring first layer, pathway $\alpha_t\in [K]$ becomes the clustering assignment for task $t$.} Applying Theorem \ref{thm:main}, we obtain a generalization bound of 
\[
\RMTL(\hat\f)\lesssim {\Gt_{NT}(\Psi_1)}+\sqrt{K}\Gt_{NT}(\Psi_2)+\Gt_N(\Hc)+{\sqrt{\frac{\log K}{N}}}.
\]
Here, $\psi_1\in\Psi_1$ is the shared first layer module, $\psi_2^k\in\Psi_2$ is the module assigned to cluster $k\in[K]$ that personalizes its representation, and we have $|\Ac|=K$. To provide further insights, let us focus on linear representations with the notation of Section \ref{sec linear}: $\psi_1(\x)=\B_1\x$, $\psi_2^k(\x')=\B^k_2\x'$, and $h_t(\x'')=\h_t^\top\x''$ with dimensions $\B_1\in\R^{R\times p}$, $\B^k_2\in\R^{r\times R}$, $\h_t\in\R^{r}$ and $r\leq R\leq p$. 
Our bound now takes the form
\[
\RMTL(\hat\f)\lesssim \sqrt{\frac{Rp+KrR+T(r+\log K)}{NT}},
\]
where $Rp$ and $KrR$ are the number of parameters in supernet layers $1$ and $2$, and $(r+\log K)/N$ is the cost of learning pathway and prediction head per task. Let us contrast this to the shallow MTL approaches with $1$-layer supernets.

\smallskip
\noindent$\bullet$ \textbf{Vanilla MTL:} Learn $\B_1\in\R^{R\times p}$ and learn larger prediction heads $\h^V_t\in\R^R$ (no clustering needed).

\noindent$\bullet$ \textbf{Cluster MTL:} Learn larger cluster modules $\B_2^{C,k}\in\R^{r\times p}$, and learn pathway $\alpha_t$ and head $\h_t\in\R^r$ (no $\B_1$ needed).

\smallskip
\noindent\textbf{Experimental Insights.} Before providing a theoretical comparison, let us discuss the experimental results where we compare these three approaches in a realizable dataset generated according to Figure \ref{fig:intro_tree}. Specifically, we generate $\bar{\B}_1$ and $\{\bar{\B}^{k}_2\}_{k=1}^K$ with orthonormal rows uniformly at random independently. We also generate $\bar{\h}_t$ uniformly at random over the unit sphere independently. Let $\bar{\alpha}_t$ be the cluster assignment of task $t$ where each cluster has same size/number of tasks with $\bar{T}=T/K$ tasks. The distribution $\Dc_t$ associated with task $t$ is generated as
\[
y=\x^\top \bt^\st_t\quad\text{where}\quad \bt^\st_t=(\bar{\h}^\top_t\bar{\B}^{\alpha_t}_2\bar{\B}_1)^\top,~\x\sim\Nn(\boldsymbol{0},\Iden_p),
\]
without label noise. {We evaluate and present results from two scenarios where cluster assignment of each task $\bar\alpha_t$ is known (Figure~\ref{fig:MTLcomparison}) or not (Figure~\ref{fig:clustering_exp}). MTL, Cluster-MTL and Multipath-MTL labels corresponds to our single representation, clustering and hierarchical MTL strategies respectively, in the figures.}

In Figure~\ref{fig:MTLcomparison}, we solve MTL problems with the knowledge of clustering $\bar\alpha_t$. We set ambient dimension $p=32$, shared embedding $R=8$, and cluster embeddings $r=2$. We consider a base configuration of $K=40$ clusters, $\Tbar=T/K=10$ tasks per cluster and $N=10$ samples per task (see supplementary material for further details). Figure \ref{fig:MTLcomparison} compares the performance of three approaches for the task-averaged MTL test risk and demonstrates consistent benefits of \MP MTL for varying $K,\bar{T},N$.

{
We also consider the setting where $\bar\alpha_t$, $t\in[T]$ are unknown during training. Set $p=128$, $R=32$ and $r=2$, and fix number of clusters $K=50$ and cluster size $\bar T=10$. In this experiment, instead of using the ground truth clustering $\bar\alpha_t$, we also learn the clustering assignment $\hat\alpha_t$ for each task. As we discussed and visualized in supplementary material, it is not easy to cluster random tasks even with the hindsight knowledge of task vectors $\bt_t^\st$. To overcome this issue, we add correlation between tasks in the same cluster. Specifically, generate the prediction head by $\bar\h_t'=\gamma\bar\h^k+(1-\gamma)\bar\h_t$ where $\bar\h^k,\bar\h_t$ are random unit vectors corresponding to the cluster $k$ and task $t$ (assuming $\bar\alpha_t=k$). To cluster tasks, we first run vanilla MTL and learn the shared representation $\hat\B_1$ and heads $(\hat\h_t^V)_{t=1}^T$. Next build task vector estimates by $\hat\bt_t:=\hat\B_i^\top\hat\h_t^V$, and get $T\times T$ task similarity matrix using Euclidean distance metric. Applying standard $K$-means clustering to it provides a clustering assignment $\hat\alpha_t$. In the experiment, we set $\gamma=0.6$ to make sure hindsight knowledge of $\bt_t^\st$ is sufficient to correctly cluster all tasks. Results are presented in Figure~\ref{fig:clustering_exp}, where solid curves are solving MTL with ground truth $\bar\alpha_t$ while dashed curves are using $\hat\alpha_t$. We observe that when given enough samples ($N\geq60$), all tasks are grouped correctly even if the MTL risk is not zero. More importantly, \MP MTL does outperform both vanilla MTL and cluster MTL even when the clustering is not fully correct.
}

\begin{figure}[t]
\centering
  \centering
 \includegraphics[width=.33\textwidth]{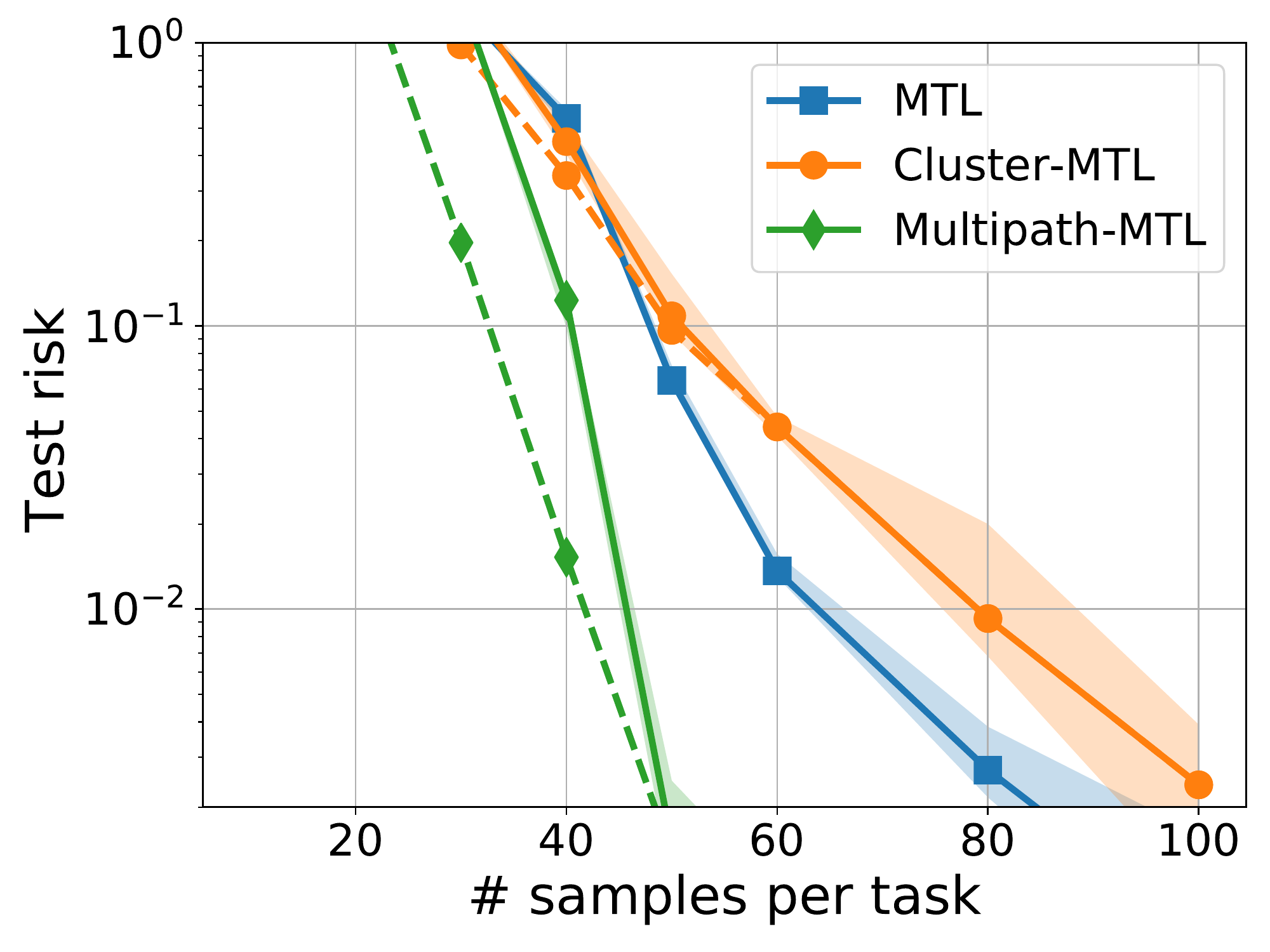}
\vspace{-7pt}
\caption{{We group the $T=500$ tasks into $K=50$ clusters and compare the sample complexity of different MTL strategies. Given different sample size, we cluster tasks based on the trained MTL model and solve Cluster-/\MP-MTL based on the assigned clusters. Solid curves are results using ground truth cluster knowledge $\bar\alpha_t$ and dashed are using the learned clustering $\hat\alpha_t$.} {Experimental setting follows the same setting as in Figure~\ref{fig:MTLcomparison}.}}\label{fig:clustering_exp}
\vspace{-10pt}
\end{figure}

\smallskip
\noindent\textbf{Understanding the benefits of \MP MTL.} Naturally, superior numerical performance of \MP MTL in Figure \ref{fig:MTLcomparison}\&\ref{fig:clustering_exp} partly stems from the hierarchical dataset model we study. This model will also shed light on shortcomings of 1-layer supernets drawing from our theoretical predictions. First, observe that all three baselines are exactly specified: We use the smallest model sizes that capture the ground-truth model so that they can achieve zero test risk as $N,K,T$ grows. For instance, Vanilla MTL achieves zero risk by setting $\B_1=\bar{\B}_1,\h^V_t=(\bar{\B}^{\alpha_t}_2)^\top \bar{\h}_t$ and cluster MTL achieves zero risk by setting $\B_2^{C,k}=\bar{\B}_2^k\bar{\B}_1,\h_t=\bar{\h}_t$. Thus, the benefit of \MP MTL arises from stronger weight sharing across tasks that reduces test risk. In light of Sec.~\ref{sec linear}, the generalization risks of these approaches can be bounded as {$\sqrt{{\DoF(\Fc)}/{NT}}$} where Number-of-Parameters compare as $\textbf{Vanilla:}~Rp+TR$, $\textbf{Cluster:}~Krp+Tr$, $\textbf{\MP:}~Rp+KrR+Tr$. From this, it can be seen that \MP is never worse than the others as long as $Kr\geq R$ and $\bar{T}=T/K\geq r$. These conditions hold under the assumption that multipath model is of minimal size: Otherwise, there would be a strictly smaller zero-risk model by setting $R\gets Kr$ and $r\gets \bar{T}$.

Conversely, \MP shines in the regime $Kr\gg R$ or $\bar{T}\gg r$. As $\frac{Kr}{R},\frac{p}{R}\rightarrow\infty$, \MP strictly outperforms Cluster MTL. This arises from a \emph{cluster diversity} phenomenon that connects to the \emph{task diversity} notions of prior art. In essence, since $r$-dimensional clusters lie on a shared $R$ dimensional space, as we add more clusters beyond $Kr\geq R$, they will collaboratively estimate the shared subspace which in turn helps estimating their local subspaces by projecting them onto the shared one. As $\frac{\bar{T}}{r},\frac{R}{r}\rightarrow\infty$, \MP strictly outperforms Vanilla MTL. $\frac{\bar{T}}{r}$ is needed to ensure that there is enough task diversity within each cluster to estimate its local subspace. Finally, $\frac{R}{r}$ ratio is the few-shot learning benefit of clustering over Vanilla MTL. The prediction heads of vanilla MTL is larger which necessitates a larger $N$, at the minimum $N\geq R$. Whereas \MP works with as little as $N\geq r$. The same argument also implies that clustering/hierarchy would also enable better transfer learning.




\section{Related Work}\label{app related}

Our work is related to a large body of literature spanning efficient architectures and statistical guarantees for MTL, representation learning, task similarity, and subspace clustering. 

\noindent $\bullet$ \textbf{Multitask Representation Learning.} While MTL problems admit multiple approaches, an important idea is building shared representations to embed tasks in a low-dimensional space \cite{zhang2021survey,thrun2012learning,wang2016distributed,baxter2000model}. After identifying this low-dimensional representation, new tasks can be learned in a sample efficient fashion inline with the benefits of deep representations in modern ML applications. While most earlier works focus on linear models, \cite{maurer2016benefit} provides guarantees for general hypothesis classes through empirical process theory improving over \cite{baxter2000model}. More recently, there is a growing line of work on multitask representations that spans tighter sample complexity analysis \cite{garg2020functional,hanneke2020no, du2020few,kong2020meta,xu2021representation,lu2021power}, convergence guarantees \cite{balcan2019provable,khodak2019adaptive,collins2022maml,ji2020convergence,collins2021exploiting,wu2020understanding}, lifelong learning \cite{xu2022statistical,li2022provable}, and decision making problems \cite{yang2020impact,qin2022non,cheng2022provable,sodhani2021multi}. Closest to our work is \cite{tripuraneni2021provable} which provides tighter sample complexity guarantees compared to \cite{maurer2016benefit}. Our problem formulation generalizes prior work (that is mostly limited to single shared representation) by allowing deep compositional representations computed along supernet pathways. To overcome the associated technical challenges, we develop multilayer chain rules for Gaussian Complexity, introduce new notions to assess the quality of supernet representations, and develop new theory for linear representations.

\noindent $\bullet$ \textbf{Quantifying Task Similarity and Clustering.}  We note that task similarity and clustering has been studied by \cite{shui2019principled,nguyen2021similarity,zhou2020task,fifty2021efficiently,kumar2012learning,kang2011learning,aribandi2021ext5,zamir2018taskonomy} however these works do not come with comparable statistical guarantees. Leveraging relations between tasks are explored even more broadly \cite{zhuang2020comprehensive,achille2021information}. Our experiments on linear \MP MTL connects well with the broader subspace clustering literature \cite{vidal2011subspace,parsons2004subspace,elhamifar2013sparse}. Specifically, each learning task $\bt_t$ can be viewed as a point on a high-dimensional subspace. \MP MTL aims to cluster these points into smaller subspaces that correspond to task-specific representations. Our challenge is that we only get to see the points through the associated datasets.





\noindent $\bullet$ \textbf{ML Architectures and Systems.} While traditional ML models tend to be good at a handful of tasks, next-generation of neural architectures are expected to excel at a diverse range of tasks while allowing for multiple input modalities. To this aim, task-specific representations can help address both computational and data efficiency challenges. Recent works \cite{ramesh2021boosting,shu2021zoo,ramesh2021model,fifty2021efficiently,yao2019hierarchically,vuorio2019multimodal,mansour2020three,tan2022towards,ghosh2020efficient,collins2021exploiting} propose hierarchical/clustering approaches to group tasks in terms of their similarities, \cite{qin2020multitask,ye2022eliciting,gupta2022sparsely,asai2022attentional,he2022smile} focus on training mixture-of-experts (MoE) models, and similar to the pathways \cite{strezoski2019many,rosenbaum2017routing,chen2021boosting,ma2019snr} study on task routing.  
In the context of lifelong learning, PathNet, PackNet \cite{fernando2017pathnet,mallya2018packnet} and many other existing methods \cite{parisi2019continual,mallya2018piggyback,hung2019compacting,wortsman2020supermasks,cheung2019superposition} propose to embed many tasks into the same network to facilitate sample/compute efficiency. PathNet as well as SNR \cite{ma2019snr} propose methods to identify pathways/routes for individual tasks and efficiently compute them over the conditional subnetwork. With the advent of large language models, conditional computation paradigm is witnessing a growing interest with architectural innovations such as muNet, GShard, Pathways, and PaLM \cite{gesmundo2022evolutionary,gesmundo2022munet,barham2022pathways,pathways,lepikhin2020gshard,chowdhery2022palm,driess2023palme} and provide a strong motivation for theoretically-grounded \MP MTL methods.
\vspace{-5pt}\section{Discussion}

This work explored novel multitask learning problems which allow for task-specific representations that are computed along pathways of a large supernet. We established generalization bounds under a general setting which proved insightful when specialized to linear or hierarchical representations. We believe there are multiple exciting directions to explore. First, it is desirable to develop a stronger control over the generalization risk of specific groups of tasks. Our Lemma \ref{no harm} is a step in this direction. Second, what are risk upper/lower bounds for \MP MTL as we vary the depth and width of the supernet graph? Discussion in Section \ref{hierarchy} falls under this question where we demonstrate the sample complexity benefits of \MP MTL over traditional MTL approaches.
Finally, following experiments in Section~\ref{hierarchy}, can we establish similar provable guarantees for computationally-efficient algorithms (e.g.~method of moments, gradient descent)? 
\section*{Acknowledgements}
Authors would like to thank Zhe Zhao for helpful discussions and pointing out related works. This work was supported in part by the NSF grants CCF-2046816 and CCF-2212426, Google
Research Scholar award, and Army Research Office grant W911NF2110312.

\bibliography{refs,biblio}

\begin{thebibliography}{90}
\providecommand{\natexlab}[1]{#1}

\bibitem[{Achille et~al.(2021)Achille, Paolini, Mbeng, and
  Soatto}]{achille2021information}
Achille, A.; Paolini, G.; Mbeng, G.; and Soatto, S. 2021.
\newblock The information complexity of learning tasks, their structure and
  their distance.
\newblock \emph{Information and Inference: A Journal of the IMA}, 10(1):
  51--72.

\bibitem[{Aribandi et~al.(2021)Aribandi, Tay, Schuster, Rao, Zheng, Mehta,
  Zhuang, Tran, Bahri, Ni et~al.}]{aribandi2021ext5}
Aribandi, V.; Tay, Y.; Schuster, T.; Rao, J.; Zheng, H.~S.; Mehta, S.~V.;
  Zhuang, H.; Tran, V.~Q.; Bahri, D.; Ni, J.; et~al. 2021.
\newblock Ext5: Towards extreme multi-task scaling for transfer learning.
\newblock \emph{arXiv preprint arXiv:2111.10952}.

\bibitem[{Asai et~al.(2022)Asai, Salehi, Peters, and
  Hajishirzi}]{asai2022attentional}
Asai, A.; Salehi, M.; Peters, M.~E.; and Hajishirzi, H. 2022.
\newblock Attentional Mixtures of Soft Prompt Tuning for Parameter-efficient
  Multi-task Knowledge Sharing.
\newblock \emph{arXiv preprint arXiv:2205.11961}.

\bibitem[{Balcan, Khodak, and Talwalkar(2019)}]{balcan2019provable}
Balcan, M.-F.; Khodak, M.; and Talwalkar, A. 2019.
\newblock Provable guarantees for gradient-based meta-learning.
\newblock In \emph{International Conference on Machine Learning}, 424--433.
  PMLR.

\bibitem[{Barham et~al.(2022)Barham, Chowdhery, Dean, Ghemawat, Hand, Hurt,
  Isard, Lim, Pang, Roy et~al.}]{barham2022pathways}
Barham, P.; Chowdhery, A.; Dean, J.; Ghemawat, S.; Hand, S.; Hurt, D.; Isard,
  M.; Lim, H.; Pang, R.; Roy, S.; et~al. 2022.
\newblock Pathways: Asynchronous distributed dataflow for ML.
\newblock \emph{Proceedings of Machine Learning and Systems}, 4: 430--449.

\bibitem[{Bartlett, Foster, and Telgarsky(2017)}]{bartlett2017spectrally}
Bartlett, P.~L.; Foster, D.~J.; and Telgarsky, M.~J. 2017.
\newblock Spectrally-normalized margin bounds for neural networks.
\newblock In \emph{Advances in Neural Information Processing Systems},
  6241--6250.

\bibitem[{Baxter(2000)}]{baxter2000model}
Baxter, J. 2000.
\newblock A model of inductive bias learning.
\newblock \emph{Journal of artificial intelligence research}, 12: 149--198.

\bibitem[{Bouniot et~al.(2020)Bouniot, Redko, Audigier, Loesch, Zotkin, and
  Habrard}]{bouniot2020towards}
Bouniot, Q.; Redko, I.; Audigier, R.; Loesch, A.; Zotkin, Y.; and Habrard, A.
  2020.
\newblock Towards better understanding meta-learning methods through multi-task
  representation learning theory.
\newblock \emph{arXiv preprint arXiv:2010.01992}.

\bibitem[{Brown et~al.(2020)Brown, Mann, Ryder, Subbiah, Kaplan, Dhariwal,
  Neelakantan, Shyam, Sastry, Askell et~al.}]{brown2020language}
Brown, T.; Mann, B.; Ryder, N.; Subbiah, M.; Kaplan, J.~D.; Dhariwal, P.;
  Neelakantan, A.; Shyam, P.; Sastry, G.; Askell, A.; et~al. 2020.
\newblock Language models are few-shot learners.
\newblock \emph{Advances in neural information processing systems}, 33:
  1877--1901.

\bibitem[{Caruana(1997)}]{caruana1997multitask}
Caruana, R. 1997.
\newblock Multitask learning.
\newblock \emph{Machine learning}, 28(1): 41--75.

\bibitem[{Chen et~al.(2021)Chen, Crammer, He, Roth, and Su}]{chen2021weighted}
Chen, S.; Crammer, K.; He, H.; Roth, D.; and Su, W.~J. 2021.
\newblock Weighted Training for Cross-Task Learning.
\newblock \emph{arXiv preprint arXiv:2105.14095}.

\bibitem[{Chen, Gu, and Fu(2021)}]{chen2021boosting}
Chen, X.; Gu, X.; and Fu, L. 2021.
\newblock Boosting share routing for multi-task learning.
\newblock In \emph{Companion Proceedings of the Web Conference 2021}, 372--379.

\bibitem[{Cheng et~al.(2022)Cheng, Feng, Yang, Zhang, and
  Liang}]{cheng2022provable}
Cheng, Y.; Feng, S.; Yang, J.; Zhang, H.; and Liang, Y. 2022.
\newblock Provable benefit of multitask representation learning in
  reinforcement learning.
\newblock \emph{arXiv preprint arXiv:2206.05900}.

\bibitem[{Cheung et~al.(2019)Cheung, Terekhov, Chen, Agrawal, and
  Olshausen}]{cheung2019superposition}
Cheung, B.; Terekhov, A.; Chen, Y.; Agrawal, P.; and Olshausen, B. 2019.
\newblock Superposition of many models into one.
\newblock \emph{Advances in neural information processing systems}, 32.

\bibitem[{Chowdhery et~al.(2022)Chowdhery, Narang, Devlin, Bosma, Mishra,
  Roberts, Barham, Chung, Sutton, Gehrmann et~al.}]{chowdhery2022palm}
Chowdhery, A.; Narang, S.; Devlin, J.; Bosma, M.; Mishra, G.; Roberts, A.;
  Barham, P.; Chung, H.~W.; Sutton, C.; Gehrmann, S.; et~al. 2022.
\newblock Palm: Scaling language modeling with pathways.
\newblock \emph{arXiv preprint arXiv:2204.02311}.

\bibitem[{Collins et~al.(2021)Collins, Hassani, Mokhtari, and
  Shakkottai}]{collins2021exploiting}
Collins, L.; Hassani, H.; Mokhtari, A.; and Shakkottai, S. 2021.
\newblock Exploiting shared representations for personalized federated
  learning.
\newblock In \emph{International Conference on Machine Learning}, 2089--2099.
  PMLR.

\bibitem[{Collins et~al.(2022)Collins, Mokhtari, Oh, and
  Shakkottai}]{collins2022maml}
Collins, L.; Mokhtari, A.; Oh, S.; and Shakkottai, S. 2022.
\newblock MAML and ANIL provably learn representations.
\newblock \emph{arXiv preprint arXiv:2202.03483}.

\bibitem[{Dean(2021)}]{pathways}
Dean, J. 2021.
\newblock Introducing Pathways: A next-generation AI architecture.
\newblock
  \emph{\url{https://blog.google/technology/ai/introducing-pathways-next-generation-ai-architecture/},
  Google AI Blog}.

\bibitem[{Deng et~al.(2009)Deng, Dong, Socher, Li, Li, and
  Fei-Fei}]{deng2009imagenet}
Deng, J.; Dong, W.; Socher, R.; Li, L.-J.; Li, K.; and Fei-Fei, L. 2009.
\newblock Imagenet: A large-scale hierarchical image database.
\newblock In \emph{2009 IEEE conference on computer vision and pattern
  recognition}, 248--255. Ieee.

\bibitem[{Driess et~al.(2023)Driess, Xia, Sajjadi, Lynch, Chowdhery, Ichter,
  Wahid, Tompson, Vuong, Yu, Huang, Chebotar, Sermanet, Duckworth, Levine,
  Vanhoucke, Hausman, Toussaint, Greff, Zeng, Mordatch, and
  Florence}]{driess2023palme}
Driess, D.; Xia, F.; Sajjadi, M. S.~M.; Lynch, C.; Chowdhery, A.; Ichter, B.;
  Wahid, A.; Tompson, J.; Vuong, Q.; Yu, T.; Huang, W.; Chebotar, Y.; Sermanet,
  P.; Duckworth, D.; Levine, S.; Vanhoucke, V.; Hausman, K.; Toussaint, M.;
  Greff, K.; Zeng, A.; Mordatch, I.; and Florence, P. 2023.
\newblock PaLM-E: An Embodied Multimodal Language Model.
\newblock In \emph{arXiv preprint arXiv:2303.03378}.

\bibitem[{Du et~al.(2020)Du, Hu, Kakade, Lee, and Lei}]{du2020few}
Du, S.~S.; Hu, W.; Kakade, S.~M.; Lee, J.~D.; and Lei, Q. 2020.
\newblock Few-shot learning via learning the representation, provably.
\newblock \emph{arXiv preprint arXiv:2002.09434}.

\bibitem[{Elhamifar and Vidal(2013)}]{elhamifar2013sparse}
Elhamifar, E.; and Vidal, R. 2013.
\newblock Sparse subspace clustering: Algorithm, theory, and applications.
\newblock \emph{IEEE transactions on pattern analysis and machine
  intelligence}, 35(11): 2765--2781.

\bibitem[{Fernando et~al.(2017)Fernando, Banarse, Blundell, Zwols, Ha, Rusu,
  Pritzel, and Wierstra}]{fernando2017pathnet}
Fernando, C.; Banarse, D.; Blundell, C.; Zwols, Y.; Ha, D.; Rusu, A.~A.;
  Pritzel, A.; and Wierstra, D. 2017.
\newblock Pathnet: Evolution channels gradient descent in super neural
  networks.
\newblock \emph{arXiv preprint arXiv:1701.08734}.

\bibitem[{Fifty et~al.(2021)Fifty, Amid, Zhao, Yu, Anil, and
  Finn}]{fifty2021efficiently}
Fifty, C.; Amid, E.; Zhao, Z.; Yu, T.; Anil, R.; and Finn, C. 2021.
\newblock Efficiently identifying task groupings for multi-task learning.
\newblock \emph{Advances in Neural Information Processing Systems}, 34:
  27503--27516.

\bibitem[{Garg and Liang(2020)}]{garg2020functional}
Garg, S.; and Liang, Y. 2020.
\newblock Functional regularization for representation learning: A unified
  theoretical perspective.
\newblock \emph{Advances in Neural Information Processing Systems}, 33:
  17187--17199.

\bibitem[{Gesmundo and Dean(2022{\natexlab{a}})}]{gesmundo2022evolutionary}
Gesmundo, A.; and Dean, J. 2022{\natexlab{a}}.
\newblock An Evolutionary Approach to Dynamic Introduction of Tasks in
  Large-scale Multitask Learning Systems.
\newblock \emph{arXiv preprint arXiv:2205.12755}.

\bibitem[{Gesmundo and Dean(2022{\natexlab{b}})}]{gesmundo2022munet}
Gesmundo, A.; and Dean, J. 2022{\natexlab{b}}.
\newblock muNet: Evolving Pretrained Deep Neural Networks into Scalable
  Auto-tuning Multitask Systems.
\newblock \emph{arXiv preprint arXiv:2205.10937}.

\bibitem[{Ghosh et~al.(2020)Ghosh, Chung, Yin, and
  Ramchandran}]{ghosh2020efficient}
Ghosh, A.; Chung, J.; Yin, D.; and Ramchandran, K. 2020.
\newblock An efficient framework for clustered federated learning.
\newblock \emph{Advances in Neural Information Processing Systems}, 33:
  19586--19597.

\bibitem[{Golowich, Rakhlin, and Shamir(2018)}]{golowich2018size}
Golowich, N.; Rakhlin, A.; and Shamir, O. 2018.
\newblock Size-independent sample complexity of neural networks.
\newblock In \emph{Conference On Learning Theory}, 297--299. PMLR.

\bibitem[{Gupta et~al.(2022)Gupta, Mukherjee, Subudhi, Gonzalez, Jose,
  Awadallah, and Gao}]{gupta2022sparsely}
Gupta, S.; Mukherjee, S.; Subudhi, K.; Gonzalez, E.; Jose, D.; Awadallah,
  A.~H.; and Gao, J. 2022.
\newblock Sparsely activated mixture-of-experts are robust multi-task learners.
\newblock \emph{arXiv preprint arXiv:2204.07689}.

\bibitem[{Hanneke and Kpotufe(2020)}]{hanneke2020no}
Hanneke, S.; and Kpotufe, S. 2020.
\newblock A no-free-lunch theorem for multitask learning.
\newblock \emph{arXiv preprint arXiv:2006.15785}.

\bibitem[{He et~al.(2022)He, Zheng, Zhang, Karypis, Chilimbi, Soltanolkotabi,
  and Avestimehr}]{he2022smile}
He, C.; Zheng, S.; Zhang, A.; Karypis, G.; Chilimbi, T.; Soltanolkotabi, M.;
  and Avestimehr, S. 2022.
\newblock SMILE: Scaling Mixture-of-Experts with Efficient Bi-level Routing.
\newblock \emph{arXiv preprint arXiv:2212.05191}.

\bibitem[{Hung et~al.(2019)Hung, Tu, Wu, Chen, Chan, and
  Chen}]{hung2019compacting}
Hung, C.-Y.; Tu, C.-H.; Wu, C.-E.; Chen, C.-H.; Chan, Y.-M.; and Chen, C.-S.
  2019.
\newblock Compacting, picking and growing for unforgetting continual learning.
\newblock \emph{Advances in Neural Information Processing Systems}, 32.

\bibitem[{Ji et~al.(2020)Ji, Lee, Liang, and Poor}]{ji2020convergence}
Ji, K.; Lee, J.~D.; Liang, Y.; and Poor, H.~V. 2020.
\newblock Convergence of meta-learning with task-specific adaptation over
  partial parameters.
\newblock \emph{Advances in Neural Information Processing Systems}, 33:
  11490--11500.

\bibitem[{Ji and Telgarsky(2018)}]{ji2018gradient}
Ji, Z.; and Telgarsky, M. 2018.
\newblock Gradient descent aligns the layers of deep linear networks.
\newblock \emph{arXiv preprint arXiv:1810.02032}.

\bibitem[{Kang, Grauman, and Sha(2011)}]{kang2011learning}
Kang, Z.; Grauman, K.; and Sha, F. 2011.
\newblock Learning with whom to share in multi-task feature learning.
\newblock In \emph{ICML}.

\bibitem[{Khodak, Balcan, and Talwalkar(2019)}]{khodak2019adaptive}
Khodak, M.; Balcan, M.-F.~F.; and Talwalkar, A.~S. 2019.
\newblock Adaptive gradient-based meta-learning methods.
\newblock \emph{Advances in Neural Information Processing Systems}, 32.

\bibitem[{Kong et~al.(2020{\natexlab{a}})Kong, Somani, Kakade, and
  Oh}]{kong2020robust}
Kong, W.; Somani, R.; Kakade, S.; and Oh, S. 2020{\natexlab{a}}.
\newblock Robust meta-learning for mixed linear regression with small batches.
\newblock \emph{Advances in neural information processing systems}, 33:
  4683--4696.

\bibitem[{Kong et~al.(2020{\natexlab{b}})Kong, Somani, Song, Kakade, and
  Oh}]{kong2020meta}
Kong, W.; Somani, R.; Song, Z.; Kakade, S.; and Oh, S. 2020{\natexlab{b}}.
\newblock Meta-learning for mixed linear regression.
\newblock In \emph{International Conference on Machine Learning}, 5394--5404.
  PMLR.

\bibitem[{Kumar and Daume~III(2012)}]{kumar2012learning}
Kumar, A.; and Daume~III, H. 2012.
\newblock Learning task grouping and overlap in multi-task learning.
\newblock \emph{arXiv preprint arXiv:1206.6417}.

\bibitem[{Lepikhin et~al.(2020)Lepikhin, Lee, Xu, Chen, Firat, Huang, Krikun,
  Shazeer, and Chen}]{lepikhin2020gshard}
Lepikhin, D.; Lee, H.; Xu, Y.; Chen, D.; Firat, O.; Huang, Y.; Krikun, M.;
  Shazeer, N.; and Chen, Z. 2020.
\newblock Gshard: Scaling giant models with conditional computation and
  automatic sharding.
\newblock \emph{arXiv preprint arXiv:2006.16668}.

\bibitem[{Li et~al.(2022)Li, Li, Asif, and Oymak}]{li2022provable}
Li, Y.; Li, M.; Asif, M.~S.; and Oymak, S. 2022.
\newblock Provable and Efficient Continual Representation Learning.
\newblock \emph{arXiv preprint arXiv:2203.02026}.

\bibitem[{Liu, Simonyan, and Yang(2019)}]{liu2018darts}
Liu, H.; Simonyan, K.; and Yang, Y. 2019.
\newblock Darts: Differentiable architecture search.
\newblock \emph{ICLR}.

\bibitem[{Lu, Huang, and Du(2021)}]{lu2021power}
Lu, R.; Huang, G.; and Du, S.~S. 2021.
\newblock On the power of multitask representation learning in linear mdp.
\newblock \emph{arXiv preprint arXiv:2106.08053}.

\bibitem[{Ma et~al.(2019)Ma, Zhao, Chen, Li, Hong, and Chi}]{ma2019snr}
Ma, J.; Zhao, Z.; Chen, J.; Li, A.; Hong, L.; and Chi, E.~H. 2019.
\newblock Snr: Sub-network routing for flexible parameter sharing in multi-task
  learning.
\newblock In \emph{Proceedings of the AAAI Conference on Artificial
  Intelligence}, volume~33, 216--223.

\bibitem[{Mallya, Davis, and Lazebnik(2018)}]{mallya2018piggyback}
Mallya, A.; Davis, D.; and Lazebnik, S. 2018.
\newblock Piggyback: Adapting a single network to multiple tasks by learning to
  mask weights.
\newblock In \emph{Proceedings of the European Conference on Computer Vision
  (ECCV)}, 67--82.

\bibitem[{Mallya and Lazebnik(2018)}]{mallya2018packnet}
Mallya, A.; and Lazebnik, S. 2018.
\newblock Packnet: Adding multiple tasks to a single network by iterative
  pruning.
\newblock In \emph{Proceedings of the IEEE Conference on Computer Vision and
  Pattern Recognition}, 7765--7773.

\bibitem[{Mansour et~al.(2020)Mansour, Mohri, Ro, and
  Suresh}]{mansour2020three}
Mansour, Y.; Mohri, M.; Ro, J.; and Suresh, A.~T. 2020.
\newblock Three approaches for personalization with applications to federated
  learning.
\newblock \emph{arXiv preprint arXiv:2002.10619}.

\bibitem[{Maurer(2016)}]{maurer2016chain}
Maurer, A. 2016.
\newblock A chain rule for the expected suprema of Gaussian processes.
\newblock \emph{Theoretical Computer Science}, 650: 109--122.

\bibitem[{Maurer, Pontil, and Romera-Paredes(2016)}]{maurer2016benefit}
Maurer, A.; Pontil, M.; and Romera-Paredes, B. 2016.
\newblock The benefit of multitask representation learning.
\newblock \emph{Journal of Machine Learning Research}, 17(81): 1--32.

\bibitem[{Mohri, Rostamizadeh, and Talwalkar(2018)}]{mohri2018foundations}
Mohri, M.; Rostamizadeh, A.; and Talwalkar, A. 2018.
\newblock \emph{Foundations of machine learning}.
\newblock MIT press.

\bibitem[{Neyshabur et~al.(2017)Neyshabur, Bhojanapalli, McAllester, and
  Srebro}]{neyshabur2017exploring}
Neyshabur, B.; Bhojanapalli, S.; McAllester, D.; and Srebro, N. 2017.
\newblock Exploring generalization in deep learning.
\newblock \emph{Advances in neural information processing systems}, 30.

\bibitem[{Neyshabur et~al.(2018)Neyshabur, Li, Bhojanapalli, LeCun, and
  Srebro}]{neyshabur2018towards}
Neyshabur, B.; Li, Z.; Bhojanapalli, S.; LeCun, Y.; and Srebro, N. 2018.
\newblock Towards understanding the role of over-parametrization in
  generalization of neural networks.
\newblock \emph{arXiv preprint arXiv:1805.12076}.

\bibitem[{Nguyen, Do, and Carneiro(2021)}]{nguyen2021similarity}
Nguyen, C.; Do, T.-T.; and Carneiro, G. 2021.
\newblock Similarity of classification tasks.
\newblock \emph{arXiv preprint arXiv:2101.11201}.

\bibitem[{Oymak(2018)}]{oymak2018learning}
Oymak, S. 2018.
\newblock Learning compact neural networks with regularization.
\newblock In \emph{International Conference on Machine Learning}, 3966--3975.
  PMLR.

\bibitem[{Parisi et~al.(2019)Parisi, Kemker, Part, Kanan, and
  Wermter}]{parisi2019continual}
Parisi, G.~I.; Kemker, R.; Part, J.~L.; Kanan, C.; and Wermter, S. 2019.
\newblock Continual lifelong learning with neural networks: A review.
\newblock \emph{Neural Networks}, 113: 54--71.

\bibitem[{Parsons, Haque, and Liu(2004)}]{parsons2004subspace}
Parsons, L.; Haque, E.; and Liu, H. 2004.
\newblock Subspace clustering for high dimensional data: a review.
\newblock \emph{Acm sigkdd explorations newsletter}, 6(1): 90--105.

\bibitem[{Qin et~al.(2022)Qin, Menara, Oymak, Ching, and
  Pasqualetti}]{qin2022non}
Qin, Y.; Menara, T.; Oymak, S.; Ching, S.; and Pasqualetti, F. 2022.
\newblock Non-Stationary Representation Learning in Sequential Linear Bandits.
\newblock \emph{IEEE Open Journal of Control Systems}.

\bibitem[{Qin et~al.(2020)Qin, Cheng, Zhao, Chen, Metzler, and
  Qin}]{qin2020multitask}
Qin, Z.; Cheng, Y.; Zhao, Z.; Chen, Z.; Metzler, D.; and Qin, J. 2020.
\newblock Multitask mixture of sequential experts for user activity streams.
\newblock In \emph{Proceedings of the 26th ACM SIGKDD International Conference
  on Knowledge Discovery \& Data Mining}, 3083--3091.

\bibitem[{Ramesh and Chaudhari(2021{\natexlab{a}})}]{ramesh2021boosting}
Ramesh, R.; and Chaudhari, P. 2021{\natexlab{a}}.
\newblock Boosting a model zoo for multi-task and continual learning.
\newblock \emph{arXiv preprint arXiv:2106.03027}.

\bibitem[{Ramesh and Chaudhari(2021{\natexlab{b}})}]{ramesh2021model}
Ramesh, R.; and Chaudhari, P. 2021{\natexlab{b}}.
\newblock Model Zoo: A Growing Brain That Learns Continually.
\newblock In \emph{International Conference on Learning Representations}.

\bibitem[{Rosenbaum, Klinger, and Riemer(2017)}]{rosenbaum2017routing}
Rosenbaum, C.; Klinger, T.; and Riemer, M. 2017.
\newblock Routing networks: Adaptive selection of non-linear functions for
  multi-task learning.
\newblock \emph{arXiv preprint arXiv:1711.01239}.

\bibitem[{Shu et~al.(2021)Shu, Kou, Cao, Wang, and Long}]{shu2021zoo}
Shu, Y.; Kou, Z.; Cao, Z.; Wang, J.; and Long, M. 2021.
\newblock Zoo-tuning: Adaptive transfer from a zoo of models.
\newblock In \emph{International Conference on Machine Learning}, 9626--9637.
  PMLR.

\bibitem[{Shui et~al.(2019)Shui, Abbasi, Robitaille, Wang, and
  Gagn{\'e}}]{shui2019principled}
Shui, C.; Abbasi, M.; Robitaille, L.-{\'E}.; Wang, B.; and Gagn{\'e}, C. 2019.
\newblock A principled approach for learning task similarity in multitask
  learning.
\newblock \emph{arXiv preprint arXiv:1903.09109}.

\bibitem[{Sodhani, Zhang, and Pineau(2021)}]{sodhani2021multi}
Sodhani, S.; Zhang, A.; and Pineau, J. 2021.
\newblock Multi-task reinforcement learning with context-based representations.
\newblock In \emph{International Conference on Machine Learning}, 9767--9779.
  PMLR.

\bibitem[{Strezoski, Noord, and Worring(2019)}]{strezoski2019many}
Strezoski, G.; Noord, N.~v.; and Worring, M. 2019.
\newblock Many task learning with task routing.
\newblock In \emph{Proceedings of the IEEE/CVF International Conference on
  Computer Vision}, 1375--1384.

\bibitem[{Sun et~al.(2021)Sun, Narang, Gulluk, Oymak, and
  Fazel}]{sun2021towards}
Sun, Y.; Narang, A.; Gulluk, I.; Oymak, S.; and Fazel, M. 2021.
\newblock Towards sample-efficient overparameterized meta-learning.
\newblock \emph{Advances in Neural Information Processing Systems}, 34:
  28156--28168.

\bibitem[{Talagrand(2006)}]{talagrand2006generic}
Talagrand, M. 2006.
\newblock \emph{The generic chaining: upper and lower bounds of stochastic
  processes}.
\newblock Springer Science \& Business Media.

\bibitem[{Tan et~al.(2022)Tan, Yu, Cui, and Yang}]{tan2022towards}
Tan, A.~Z.; Yu, H.; Cui, L.; and Yang, Q. 2022.
\newblock Towards personalized federated learning.
\newblock \emph{IEEE Transactions on Neural Networks and Learning Systems}.

\bibitem[{Thrun and Pratt(2012)}]{thrun2012learning}
Thrun, S.; and Pratt, L. 2012.
\newblock \emph{Learning to learn}.
\newblock Springer Science \& Business Media.

\bibitem[{Tripuraneni, Jin, and Jordan(2021)}]{tripuraneni2021provable}
Tripuraneni, N.; Jin, C.; and Jordan, M. 2021.
\newblock Provable meta-learning of linear representations.
\newblock In \emph{International Conference on Machine Learning}, 10434--10443.
  PMLR.

\bibitem[{Tripuraneni, Jordan, and Jin(2020)}]{tripuraneni2020theory}
Tripuraneni, N.; Jordan, M.; and Jin, C. 2020.
\newblock On the theory of transfer learning: The importance of task diversity.
\newblock \emph{Advances in Neural Information Processing Systems}, 33:
  7852--7862.

\bibitem[{Vershynin(2010)}]{vershynin2010introduction}
Vershynin, R. 2010.
\newblock Introduction to the non-asymptotic analysis of random matrices.
\newblock \emph{arXiv preprint arXiv:1011.3027}.

\bibitem[{Vershynin(2018)}]{vershynin2018high}
Vershynin, R. 2018.
\newblock \emph{High-dimensional probability: An introduction with applications
  in data science}, volume~47.
\newblock Cambridge university press.

\bibitem[{Vidal(2011)}]{vidal2011subspace}
Vidal, R. 2011.
\newblock Subspace clustering.
\newblock \emph{IEEE Signal Processing Magazine}, 28(2): 52--68.

\bibitem[{Vuorio et~al.(2019)Vuorio, Sun, Hu, and Lim}]{vuorio2019multimodal}
Vuorio, R.; Sun, S.-H.; Hu, H.; and Lim, J.~J. 2019.
\newblock Multimodal model-agnostic meta-learning via task-aware modulation.
\newblock \emph{Advances in Neural Information Processing Systems}, 32.

\bibitem[{Wainwright(2019)}]{wainwright2019high}
Wainwright, M.~J. 2019.
\newblock \emph{High-dimensional statistics: A non-asymptotic viewpoint},
  volume~48.
\newblock Cambridge University Press.

\bibitem[{Wang, Kolar, and Srebro(2016)}]{wang2016distributed}
Wang, J.; Kolar, M.; and Srebro, N. 2016.
\newblock Distributed multi-task learning with shared representation.
\newblock \emph{arXiv preprint arXiv:1603.02185}.

\bibitem[{Wortsman et~al.(2020)Wortsman, Ramanujan, Liu, Kembhavi, Rastegari,
  Yosinski, and Farhadi}]{wortsman2020supermasks}
Wortsman, M.; Ramanujan, V.; Liu, R.; Kembhavi, A.; Rastegari, M.; Yosinski,
  J.; and Farhadi, A. 2020.
\newblock Supermasks in superposition.
\newblock \emph{Advances in Neural Information Processing Systems}, 33:
  15173--15184.

\bibitem[{Wu, Zhang, and R{\'e}(2020)}]{wu2020understanding}
Wu, S.; Zhang, H.~R.; and R{\'e}, C. 2020.
\newblock Understanding and improving information transfer in multi-task
  learning.
\newblock \emph{arXiv preprint arXiv:2005.00944}.

\bibitem[{Xu and Tewari(2021)}]{xu2021representation}
Xu, Z.; and Tewari, A. 2021.
\newblock Representation learning beyond linear prediction functions.
\newblock \emph{Advances in Neural Information Processing Systems}, 34:
  4792--4804.

\bibitem[{Xu and Tewari(2022)}]{xu2022statistical}
Xu, Z.; and Tewari, A. 2022.
\newblock On the statistical benefits of curriculum learning.
\newblock In \emph{International Conference on Machine Learning}, 24663--24682.
  PMLR.

\bibitem[{Yang et~al.(2020)Yang, Hu, Lee, and Du}]{yang2020impact}
Yang, J.; Hu, W.; Lee, J.~D.; and Du, S.~S. 2020.
\newblock Impact of representation learning in linear bandits.
\newblock \emph{arXiv preprint arXiv:2010.06531}.

\bibitem[{Yao et~al.(2019)Yao, Wei, Huang, and Li}]{yao2019hierarchically}
Yao, H.; Wei, Y.; Huang, J.; and Li, Z. 2019.
\newblock Hierarchically structured meta-learning.
\newblock In \emph{International Conference on Machine Learning}, 7045--7054.
  PMLR.

\bibitem[{Ye, Zha, and Ren(2022)}]{ye2022eliciting}
Ye, Q.; Zha, J.; and Ren, X. 2022.
\newblock Eliciting Transferability in Multi-task Learning with Task-level
  Mixture-of-Experts.
\newblock \emph{arXiv preprint arXiv:2205.12701}.

\bibitem[{Yu, Wang, and Samworth(2015)}]{yu2015useful}
Yu, Y.; Wang, T.; and Samworth, R.~J. 2015.
\newblock A useful variant of the Davis--Kahan theorem for statisticians.
\newblock \emph{Biometrika}, 102(2): 315--323.

\bibitem[{Zamir et~al.(2018)Zamir, Sax, Shen, Guibas, Malik, and
  Savarese}]{zamir2018taskonomy}
Zamir, A.~R.; Sax, A.; Shen, W.; Guibas, L.~J.; Malik, J.; and Savarese, S.
  2018.
\newblock Taskonomy: Disentangling task transfer learning.
\newblock In \emph{Proceedings of the IEEE conference on computer vision and
  pattern recognition}, 3712--3722.

\bibitem[{Zhang and Yang(2021)}]{zhang2021survey}
Zhang, Y.; and Yang, Q. 2021.
\newblock A survey on multi-task learning.
\newblock \emph{IEEE Transactions on Knowledge and Data Engineering}.

\bibitem[{Zhou et~al.(2020)Zhou, Shui, Abbasi, Robitaille, Wang, and
  Gagn{\'e}}]{zhou2020task}
Zhou, F.; Shui, C.; Abbasi, M.; Robitaille, L.-{\'E}.; Wang, B.; and Gagn{\'e},
  C. 2020.
\newblock Task similarity estimation through adversarial multitask neural
  network.
\newblock \emph{IEEE Transactions on Neural Networks and Learning Systems},
  32(2): 466--480.

\bibitem[{Zhuang et~al.(2020)Zhuang, Qi, Duan, Xi, Zhu, Zhu, Xiong, and
  He}]{zhuang2020comprehensive}
Zhuang, F.; Qi, Z.; Duan, K.; Xi, D.; Zhu, Y.; Zhu, H.; Xiong, H.; and He, Q.
  2020.
\newblock A comprehensive survey on transfer learning.
\newblock \emph{Proceedings of the IEEE}, 109(1): 43--76.

\end{thebibliography}
\onecolumn
\appendix

\begin{figure}[t]
\centering
\begin{subfigure}[t]{.3\textwidth}
  \begin{tikzpicture}
  \centering
  \node at (0,0) {\includegraphics[width=\linewidth]{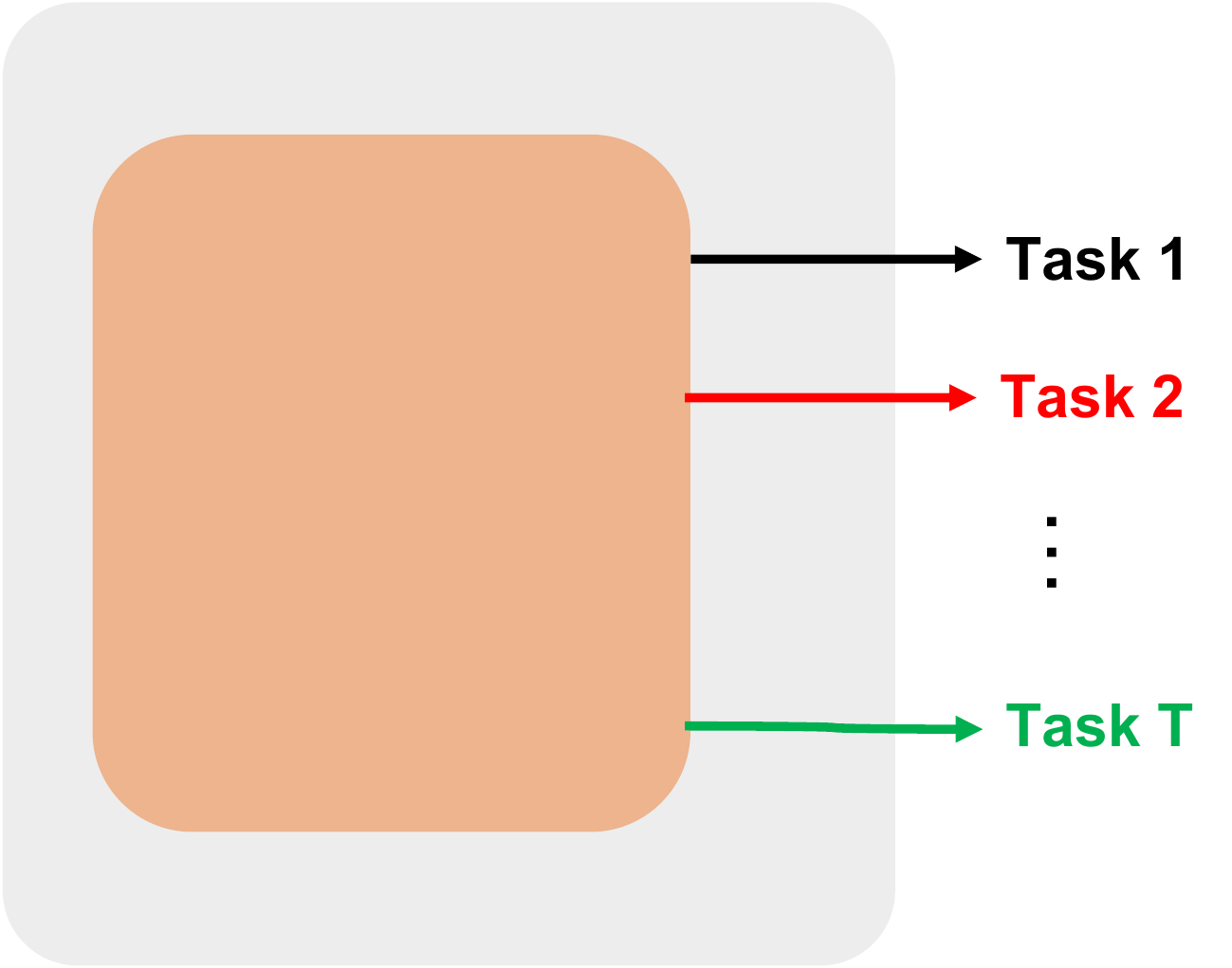}};
  \node at (-0.95, .0) [scale=2] {$\phi$};
  \node at (1., 1.2) [scale=1] {$h_1$};
  \node at (1., 0.6) [scale=1] {$h_2$};
  \node at (1., -0.8) [scale=1] {$h_T$};
  \end{tikzpicture}
  \caption{Vanilla MTL}\label{fig:vanilla}
\end{subfigure}
\hspace{20pt}
\begin{subfigure}[t]{.3\textwidth}
  \begin{tikzpicture}
  \centering
  \node at (0,0) {\includegraphics[width=\linewidth]{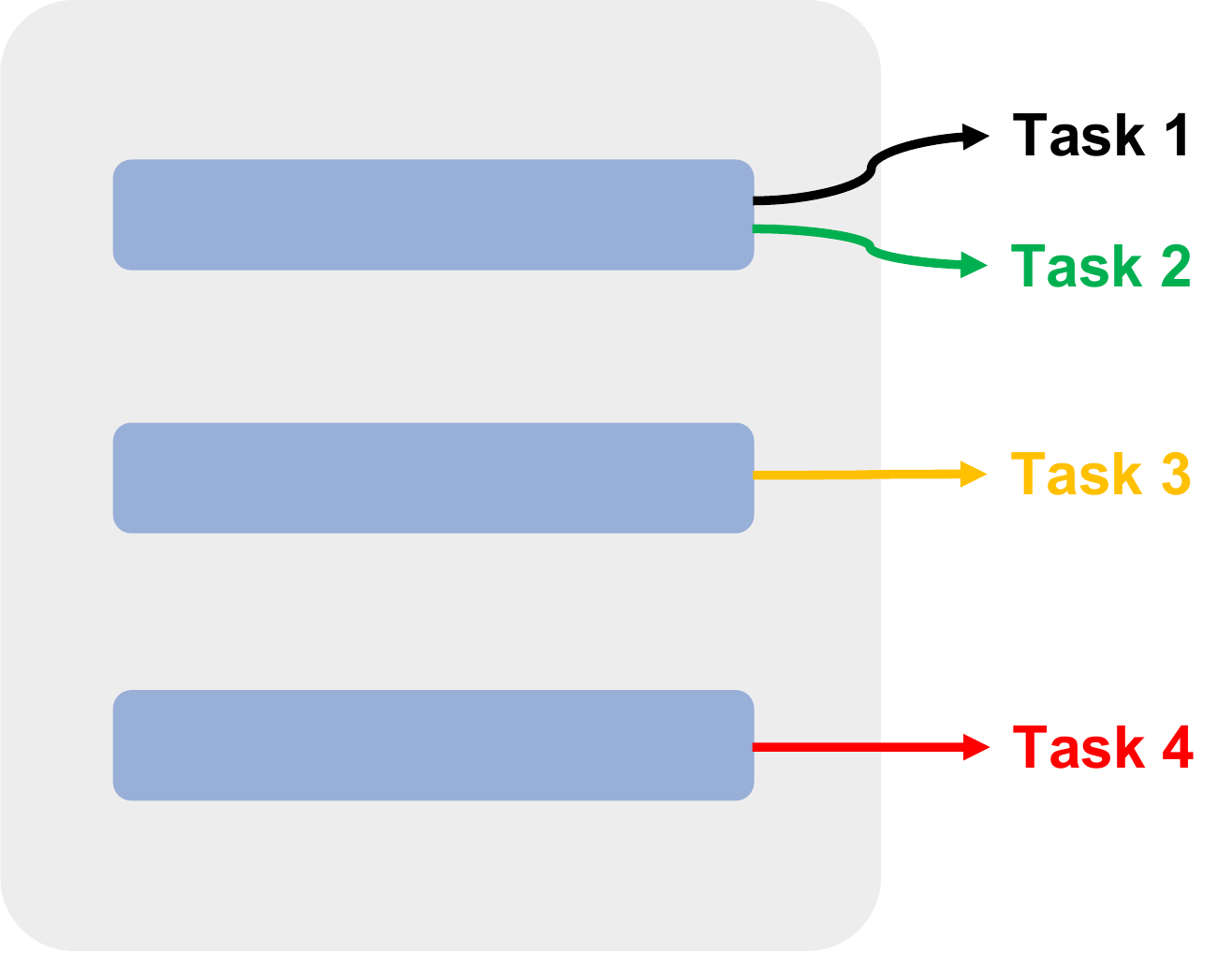}};
  \end{tikzpicture}
  \caption{Cluster MTL}\label{fig:cluster}
\end{subfigure}
\hspace{20pt}
\begin{subfigure}[t]{.3\textwidth}
  \begin{tikzpicture}
  \centering    
  \node at (0,0) {\includegraphics[width=\linewidth]{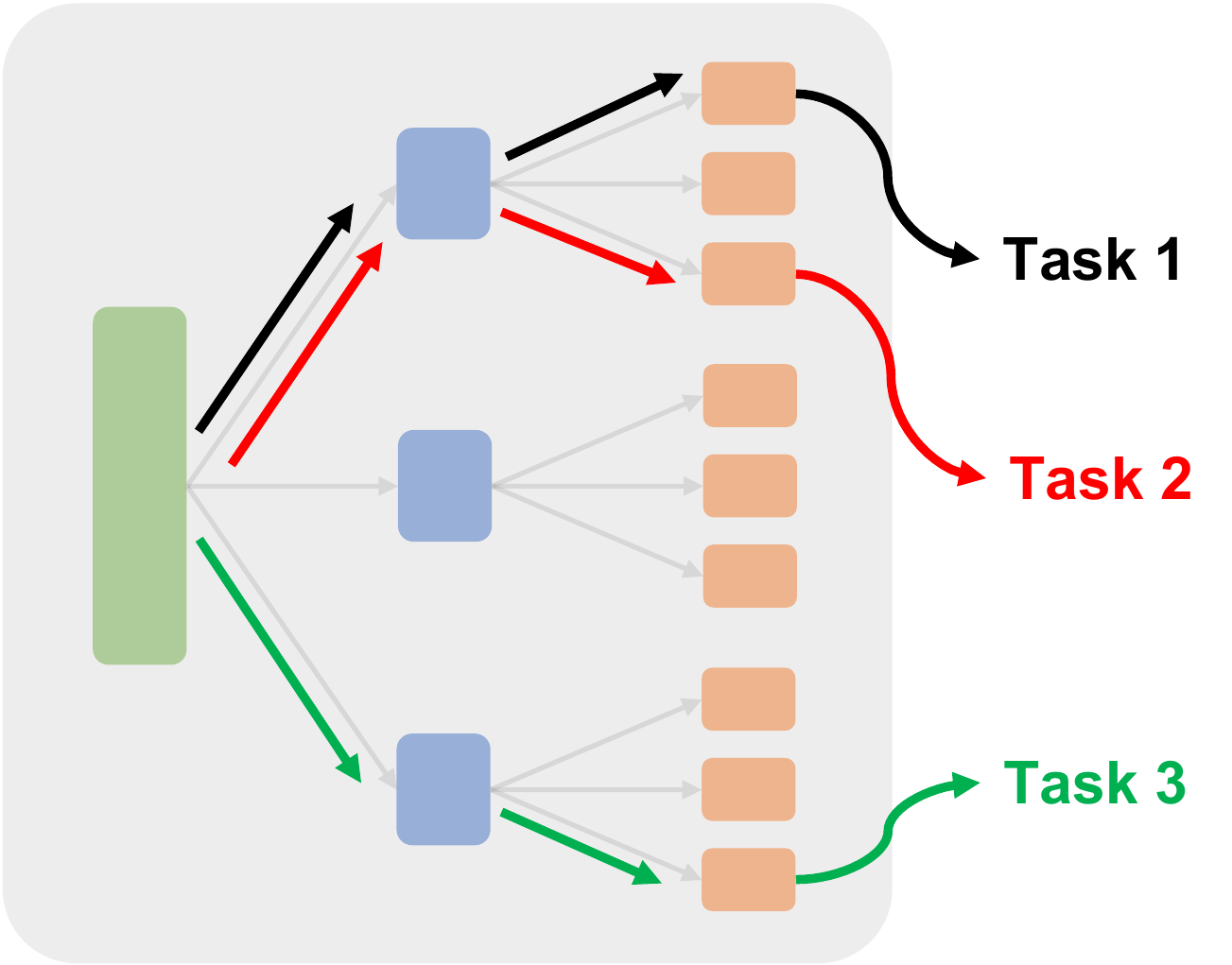}};
  \end{tikzpicture}
  \caption{Hierarchical MTL}\label{fig:hierarchical}
\end{subfigure}
\caption{Three specific MTL settings: Vanilla MTL, Cluster MTL and Hierarchical MTL. In vanilla MTL, all the tasks share the same representation $\phi\in\Phi$, and each task learns its specific head $h_t\in\Hc$. It corresponds to the setting that $|\Ac|=1$, $L=1$ and $K_1=1$. In Cluster MTL, tasks are clustered into groups and  different groups are assigned with different and uncorrelated representations. If we assume there are $K$ clusters, then $|\Ac|=K$, $L=1$ and $K_1=K$.  While, Fig.~\ref{fig:intro_tree} shows the Hierarchical MTL with only two layers, here we present the more general Hierarchical MTL setting. Assume the degree of a hierarchical supernet is $K$ (In Fig.~\ref{fig:hierarchical}, $K=3$), then $|\Ac|=K^{L-1}$ where $L$ is the number of layers in supernet, and $K_\ell=K^{\ell-1}$. }\label{fig:special cases}
\end{figure}
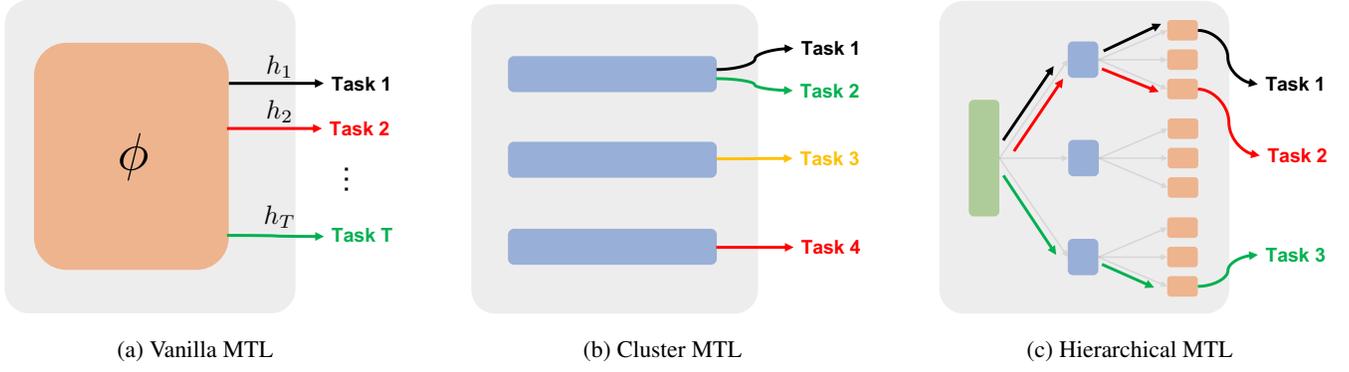

\section*{Organization of the Supplementary Material}
The supplementary material (SM) is organized as follows.
\begin{enumerate}
    \item In Appendix~\ref{app:notion} we introduce additional notions used throughout the supplementary material.
    \item {Appendix~\ref{app:proof main} provides our main proofs in Section~\ref{sec:main} and introduces two direct corollaries of Theorem~\ref{thm:main}. We also provide a data-dependent bound in terms of empirical Gaussian complexity (rather than worst-case). In Appendix~\ref{app:e2etransfer} we also provide end-to-end transfer learning bound by introducing a proper notion of task diversity.}

    \item Appendix~\ref{subexp loss} provides additional guarantees (Thm \ref{thm subgauss}) for parametric classes via non-data-dependent covering argument. The advantages of Theorem \ref{thm subgauss} are: (1) Sample complexity has linear dependence on supernet depth $L$ (rather than exponential), (2) It applies to unbounded loss functions, (3) It is also a supporting result for the proof of Theorem~\ref{cor linear}\&\ref{e2e thm}.

    \item Appendix~\ref{app:proof linear} provides our proofs in Section~\ref{sec linear}. {We also introduce Corollary~\ref{cor linear app}, which is a direct application of Theorem~\ref{thm:main}. 
     Lemma~\ref{e2e fail} proves the necessity of our Assumption~\ref{mtldiverse}.}
        \item In Appendix \ref{app:tfail2}, we include a short discussion on the challenges of transfer learning: Specifically, we provide a lemma/example that shows that, under the assumptions of Theorem \ref{e2e thm}, if ground-truth MTL pathways are not known, there are MTL settings for which transfer learning can provably fail. This construction highlights the (combinatorial) challenge of finding the right task clusterings during MTL phase that are actually useful for transfer phase.

    \item Appendix~\ref{app:numeric} provides further details, algorithms, and results on numerical experiments in Section \ref{hierarchy}. 
\end{enumerate}

\section{Useful Definitions}\label{app:notion}
We will start with some useful notions. Let $\|\cdot\|$ denote the $\ell_2$-norm of a vector, and $[L]$ denote the set $\{1,2,\dots,L\}$. We denote the $K$ times Cartesian product of a hypothesis set $\Qc$ with itself by $\Qc^K$. Now assume we have a hypothesis set $\Qc:\Xc\rightarrow\R^{r}$ and an input dataset of size $n$, defined by $\X=\{\x_1,\dots,\x_n\}$, where $\x_i\in\Xc$. Let $\{\sigma_{ij}\}_{i\in[n],j\in[r]}$ denote Rademacher variables uniformly and independently taking values in $\{-1,~1\}$ and $\{g_{ij}\}_{i\in[n],j\in[r]}$ denote i.i.d. standard random Gaussian variables. Then we can define the empirical and population Rademacher/Gaussian complexities of a hypothesis set $\Qc$ over inputs $\X$ and data space $\Xc$ with sample size $n$ as
\begin{align*}
   & \text{Empirical/Population Rademacher complexities:}~~~\widehat{\Rc}_\X(\Qc)=\E_{\sigma_{ij}}\left[\sup_{q\in\Qc}\frac{1}{n}\sum_{i=1}^n\sum_{j=1}^r\sigma_{ij}q_j(\x_i)\right],~~~~~\Rc_n(\Qc)=\E_{\X}\left[\widehat\Rc_\X(\Qc)\right],\\
    &\text{Empirical/Population Gaussian complexities:}~~~\widehat{\Gc}_\X(\Qc)=\E_{g_{ij}}\left[\sup_{q\in\Qc}\frac{1}{n}\sum_{i=1}^n\sum_{j=1}^rg_{ij}q_j(\x_i)\right],~~~~~\Gc_n(\Qc)=\E_{\X}\left[\widehat\Gc_\X(\Qc)\right],
\end{align*}
where we have $q\in\Qc$ and $q(\x)=[q_1(\x),\dots,q_r(\x)]^\top$. Note that in vector notation one can also write $\widehat{\Rc}_\X(\Qc)=\E_{\bsgm_{i}}\left[\sup_{q\in\Qc}\frac{1}{n}\sum_{i=1}^n\bsgm_{i}^\top q(\x_i)\right]$ and $\widehat{\Gc}_\X(\Qc)=\E_{\g}\left[\sup_{q\in\Qc}\frac{1}{n}\sum_{i=1}^n\g_i^\top q(\x_i)\right]$, where $\bsgm_i$ and $\g_i$ are $r$-dimensional with independent Rademacher/Gaussian variables in each entry. Also recall that worst-case versions $\widetilde{\cal{R}}_n,\Gt_n$ are obtained by taking supremum over the input space.
\section{Proofs in Section~\ref{sec:main}}\label{app:proof main}
We first introduce some lemmas used throughout this section, then provide the proofs of our mean results.
\subsection{Supporting Lemmas}
The following is a seminal contraction lemma due to Talagrand \cite{talagrand2006generic}.
\begin{lemma}[Talagrand's Contraction inequality] \label{tal lem}Let $\beps=(\eps_i)_{i=1}^n$ be i.i.d.~random variables with symmetric sign (e.g.~Rademacher, standard normal). Let $(\phi_i)_{i=1}^n$ be $L$-Lipschitz functions and $\Fc$ be a hypothesis set. We have that
\[
\E_{\beps}\left[\sup_{f\in \Fc} \sum_{i=1}^n \eps_i \phi_i(f(\x_i))\right]\leq L\E_{\beps}\left[\sup_{f\in \Fc} \sum_{i=1}^n \eps_i f(\x_i)\right].
\]
\end{lemma}
As a corollary of this, we can deduce that adjusted empirical Gaussian complexity $n\Gh_{\X}(\Fc)$ is non-decreasing in sample size $n$.
\begin{corollary} \label{cor tal}Let $\Xc$ be a bounded input space and $\Fc:\Xc\to\R$ be a hypothesis set. Let $\X_m$ be a dataset of size $m$ and $\X_n=(\x_i)_{i=1}^n$ be a dataset of size $n$ that contains $\X_m$. We have that
\[
m\Gh_{\X_m}(\Fc)\leq n\Gh_{\X_n}(\Fc).
\]
\end{corollary}
We note that, when $\Fc:\Xc\rightarrow \R^p$ is vector valued and we apply $p\times n$ $L$-Lipschitz functions $\phi_{ij}$, the identical results (Lemmas \ref{tal lem} and Corollary \ref{cor tal}) follow from Sudakov-Fernique inequality under Gaussian $\beps\in\R^{n\times p}$ (e.g.~Exercise 7.2.13 of \cite{vershynin2018high}).

This also implies usual (distributional) and worst-case Gaussian complexities are also non-decreasing.

\begin{proof} Let $(\phi_i)_{i=1}^n$ be functions that are identity for $i\leq m$ and zero for $i>m$. Observe that
\[
m\Gh_{\X_m}=\E_{\beps}\left[\sup_{f\in\Fc}\sum_{i=1}^m\eps_if(\x_i)\right]=\E_{\beps}\left[\sup_{f\in \Fc} \sum_{i=1}^n \eps_i \phi_i(f(\x_i))\right]\leq n\Gh_{\X_n}.
\]
\end{proof}

The following lemma shows that adjusted worst-case Gaussian complexity $\sqrt{n}\Gt_{\X}(\Fc)$ is essentially non-decreasing in sample size $n$.
\begin{lemma} [Worst-case Gaussian Complexity over Input Space and Sample Size]\label{lemma:worst case} For any bounded input space $\Xc$ and hypothesis set $\Fc$, we have that
\[
\sup_{1\leq m\leq n} \sqrt{m}\Gt_m(\Fc)\leq \sqrt{2n}\Gt_n(\Fc).
\]
\end{lemma}
\begin{proof} First suppose $n/2\leq m\leq n$. In this case, from Corollary \ref{cor tal}, we know that $m\Gt_m(\Fc)\leq n\Gt_n(\Fc)\implies \sqrt{m}\Gt_m(\Fc)\leq \frac{\sqrt{2}m}{\sqrt{n}}\Gt_m(\Fc)\leq \sqrt{2n}\Gt_n(\Fc)$. What remains is the scenario $m<n/2$. To do this, we will show monotonicity under doubling $\sqrt{m}\Gt_m(\Fc)\leq \sqrt{2m}\Gt_{2m}(\Fc)$. If this holds, then you can double $m$ until a point $n/2\leq m\leq n$ and apply the first bound.

Consider worst-case dataset for $\Gt_m$ defined as
\[
\Y=\arg\max_{\X\in\Xc^m}\Gh_\X(\Fc).
\]
Let $\Y'$ be a dataset of size $2m$ that repeats the elements of $\Y$ twice so that $\y'_{m+i}=\y'_i=\y_i$. Here, we consider hypothesis set $\Fc:\Xc\to\R^p$, and then $f(\y_i)=[f_1(\y_i),\cdots,f_p(\y_i)]^\top$. Also let $\beps\in\R^{m\times p},\beps'\in\R^{2m\times p}$ where $\beps'_i\sim\Nc(0,\Id_p), i\in[2m]$ and $\beps_i=\frac{\beps'_i+\beps'_{m+i}}{\sqrt{2}}\sim \Nn(0,\Id_p)$. We have that
\begin{align*}
2m\Gt_{2m}(\Fc)&\geq \E_{\beps}\left[{\sup_{f\in\Fc}} \sum_{i=1}^{2m}\sum_{j=1}^p \eps'_{ij} f_j(\y'_{i})\right]\\
&\geq  \E_{\beps}\left[{\sup_{f\in\Fc}} \sum_{i=1}^{m}\sum_{j=1}^p \eps'_{i,j} f_j(\y'_{i})+\eps'_{(m+i),j} f_j(\y'_{m+i})\right]\\
&\geq  \sqrt{2}\E_{\beps}\left[{\sup_{f\in\Fc}} \sum_{i=1}^{m}\sum_{j=1}^p \eps_{ij} f_j(\y_{i})\right]\\
&=  \sqrt{2}m\Gt_{m}.
\end{align*}
Dividing both sides by $\sqrt{2m}$, we conclude with the claim $\sqrt{m}\Gt_m(\Fc)\leq \sqrt{2m}\Gt_{2m}(\Fc)$.
\end{proof}

The following is a model selection argument shows that $\Gt(\Phi)$ can be replaced with $\Gt(\Phi_{\text{used}})$. 
\begin{lemma}[Only utilized supernet matters] \label{lem:only utilized}Observe that $T$ tasks can choose from up to $|\Ac|^T$ supernets in total. Let $\bPhi_{\all}=(\Phi_i)_{i=1}^H$ with $H\leq|\Ac|^T$ be the set of unique supernets (since two supernets that choose same number of modules per layer are identical architectures). Suppose the outcome of empirical risk minimization \eqref{formula:multipath} obeys $\hat\bphi\in \Phi_\used\in \bPhi_{\all}$. Let $\hat{K}_\ell$ be the number of (used) modules in $\Phi_\used$. With probability $1-\delta$, we have that
\begin{align}
&\Lc_{\Dcb}(\hat\f)-\Lch_{\Scb}(\hat\f)\lesssim \genbound,\label{bound1}\\
&\RMTL(\hat\f):=\Lc_{\Dcb}(\hat\f)-\Lc_{\Dcb}^\st\lesssim \genbound.\label{bound2}
\end{align}
\end{lemma}
\begin{proof} Let $\Lc_{\Phi'},\Lch_{\Phi'}$ be the population and empirical risks we achieve when we run the \eqref{formula:multipath} problem over $\Phi'\in \bPhi_\all$ rather than $\Phi$. Additionally, let $K_\ell(\Phi')$ denote the number of modules in the $\ell$th layer of the architecture $\Phi'$. Given $\Phi'$, also define $\genb{\Phi',\delta}$ to be the excess risk bound one obtains via \eqref{app main result 1} (\eqref{app main result 1} in Theorem~\ref{thm:main restate} is obtained without using Lemma~\ref{lem:only utilized}), that is,
\[
\genb{\Phi',\delta}=\Gt_N(\Hc)+\sum_{\ell=1}^L\sqrt{\K_\ell(\Phi')}\Gt_{NT}(\Psi_\ell)+\sqrt{\frac{\log|\Ac|}{N}+\frac{\log(2/\delta)}{NT}}.
\]

To proceed, applying \eqref{app main result 1} over $\Phi'\in\bPhi_\all$ and union bounding over all $H\leq|\Ac|^T$, with probability at least $1-\delta$, we find that, all $\Phi'\in \bPhi_\all$ obeys
\[
{|\Lch_{\Phi'}(\hat \f)-\Lc_{\Phi'}(\hat\f)|\lesssim \genb{\Phi',\delta/H}}.
\]
{Fortunately, $\genb{\Phi',\delta/H}\lesssim \genb{\Phi',\delta}$ since the latter already includes a $\sqrt{\frac{\log|\Ac|}{N}}$ term.} Using this union bound, optimality of $\hat\bphi\in \Phi_\used$ (and that of the associated $\hat\f\in \Fc_\used$), and using $\Lch_{\Phi_\used}(\hat\f)=\Lch_{\Phi}(\hat\f)=\Lch_{\Scb}(\hat\f)$, we find that
\begin{align}
\Lc_{\Phi_\used}(\hat\f)&\leq \Lch_{\Phi_\used}(\hat\f)+\order{\genb{\Phi_\used,\delta}}\\
&\leq \Lch_{\Scb}(\hat\f)+\order{\genb{\Phi_\used,\delta}}.\label{above eq used}
\end{align}
The last line establishes Inequality \eqref{bound1}. To conclude with the second inequality, we control the excess risk error by observing test risk upper bounds the training risk. Namely, let $\f_\st\in \Fc$ be the population minima. First, with $1-\delta$ probability, for this singleton hypothesis, we have that
\[
|\Lc_{\Dcb}(\f_\st)-\Lch_{\Scb}(\f_\st)|\leq \sqrt{\frac{\log(2/\delta)}{NT}}.
\]
Second, we can write
\[
\Lch_{\Scb}(\hat\f)\leq \Lch_{\Scb}(\f_\st)\leq \Lc_{\Dcb}(\f_\st)+\sqrt{\frac{\log(2/\delta)}{NT}}.
\]
Combining this with \eqref{above eq used}, we establish the guarantee against the ground-truth optima $\f_\st$
\[
\Lc_{\Phi_\used}(\hat\f)- \left[\Lc_{\Dcb}(\f_\st)+\sqrt{\frac{\log(2/\delta)}{NT}}\right]\leq \Lc_{\Phi_\used}(\hat\f)- \Lch_{\Scb}(\hat\f)\leq \order{\genb{\Phi_\used,\delta}},
\]
which establishes the claim \eqref{bound2} after subsuming $\sqrt{\frac{\log(2/\delta)}{NT}}$ within $\genb{\Phi_\used,\delta}$.
\end{proof}
\subsection{Proof of Theorem~\ref{thm:main}}\label{app:proof mtl}
Let us define the covering number of a hypothesis as well as natural data-dependent Euclidean distance for ease of reference in the subsequent discussion (see \cite{wainwright2019high}).
\begin{definition} [Covering number]\label{def:cover number} Let $\Qc:\Xc\rightarrow\R^r$ be a family of functions. Given $q,q'\in\Qc$, and a distance metric $d(q,q')\geq 0$, an $\eps$-cover of set $\Qc$ with respect to $d(\cdot,\cdot)$ is a set $\{q^1,q^2,\dots,q^N\}\subset\Qc$ such that for any $q\in\Qc$, there exists some $i\in[N]$ such that $d(q,q^i)\leq\eps$. The $\eps$-covering number $\Nc(\eps;\Qc,d)$ is defined to be the cardinality of the smallest $\eps$-cover.
\end{definition}

\begin{definition}[Data-dependent distance metric $\rho$]\label{def:distance metric} 
    Let $\Qc:\Xc\rightarrow\R^r$ be a family of functions. Given $q,q'\in\Qc$ and an input dataset $\X=\{\x_1,\dots,\x_n\}$ with $\x_i\in\Xc$, we define the dataset-dependent Euclidean distance by $\rho_\X(q,q'):=\sqrt{\frac{1}{n}\sum_{i\in[n],j\in[r]}(q_j(\x_i)-q'_j(\x_i))^2}=\sqrt{\frac{1}{n}\sum_{i\in[n]}\tn{q(\x_i)-q'(\x_i)}^2}$, where $q(\x)=[q_1(\x),\dots,q_r(\x)]^\top$.
\end{definition}

Now we are ready to prove our main theorem which incorporates additional dependencies that were omitted from the original statement.

\begin{theorem}[Theorem~\ref{thm:main} restated]\label{thm:main restate}
    Suppose Assumptions 1\&2 hold. Let $\hat\f$ be the empirical solution of \eqref{formula:multipath}. Let $D_\Xc=\sup_{\x\in\Xc,h\in\Hc,\bphi\in\Phi,\alpha\in\Ac}|h\circ\bphi_{\alpha}(\x)|<\infty$, and set $\GAMMA=\sum_{\ell=0}^L\Gamma^\ell$. Then, with probability at least $1-\delta$, the excess test risk in \eqref{risk mtl} obeys
    \begin{align}
        \RMTL(\hat\f)\leq768\Gamma\left(\frac{D_\Xc}{NT}+D_\Xc\sqrt{\frac{\log|\Ac|}{N}}+\GAMMA\log NT\left(\widetilde\Gc_{N}(\Hc)+\sum_{\ell=1}^L\sqrt{K_\ell}\widetilde\Gc_{NT}(\Psi_\ell)\right)\right)+2\sqrt{\frac{\log\frac{2}{\delta}}{NT}}.\label{app main result 1}
    \end{align}
    Here, the input spaces for $\Hc$ and $\Psi_\ell$ are $\Xc_\Hc=\Psi_L\circ\dots\Psi_1\circ\Xc$, $\Xc_{\Psi_\ell}=\Psi_{\ell-1}\circ\dots\Psi_1\circ\Xc$ for $\ell>1$, and $\Xc_{\Psi_1}=\Xc$. The above is our general results, which we do not focus on the actual modules used in $\hat\f$. Now let $\hat K$ be the number of modules utilized by $\hat\f$, then with probability at least $1-\delta$, we can obtain
    \begin{align}
        \RMTL(\hat\f)\lesssim\Gt_N(\Hc)+\sum_{\ell=1}^L\sqrt{\hat{\K}_\ell}\Gt_{NT}(\Psi_\ell)+\sqrt{\frac{\log|\Ac|}{N}+\frac{\log(2/\delta)}{NT}}.\label{app main result 2}
    \end{align}
    Here, $\lesssim$ suppresses dependencies on $\log NT$, $\GAMMA$ and $D_\Xc$. 
\end{theorem}
\noindent\textbf{Remark.} 
While this result is stated with worst-case Gaussian complexity, the line \eqref{empirical GC bound} states our result in terms of empirical Gaussian complexity which is always a lower bound and is in terms of the training dataset. However, \eqref{empirical GC bound} is more convoluted and involves worst-case hypothesis being applied to the training data. The latter arises from the fact that it is difficult to track the evolution of features across arbitrary pathways and hierarchical layers.

\begin{proof}
To start with, let us recap some notations. Assume we have $T$ tasks each with $N$ training samples i.i.d. drawn from $(\Dc_t)_{t=1}^T$ respectively, and let $\Dcb=\{\Dc_t\}_{t=1}^T$. Denote the training dataset and inputs of $t_{\text{th}}$ task by $\Sc_t=\{(\x_{ti},y_{ti})\}_{i=1}^N$ and $\X_t=\{\x_{ti}\}_{i=1}^N$, and define the union by $\Scb=\bigcup_{t=1}^T\Sc_t$ and $\X=\bigcup_{t=1}^T\X_t$. Let $\h=[ h_1,\dots, h_T]\in\Hc^T,~\bal=[\alpha_1,\dots,\alpha_T]\in\Ac^T,~\bpsi_\ell=[\psi_{\ell}^1,\dots,\psi_{\ell}^{K_\ell}]\in\Psi_\ell^{K_\ell}, \ell\in[L]$, and $\bphi=[\bpsi_1,\dots,\bpsi_L]\in\Phi=\Psi_1^{K_1}\times\dots\Psi_L^{K_L}$. $\hat\f:=(\hat\h,\hat\bal,\hat\bphi)$ is the empirical solution of (\ref{formula:multipath}) and $\f^\star:=(\h^\star,\bal^\star,\bphi^\star)$ is the population solution of (\ref{formula:multipath}) when each task has infinite i.i.d training samples ($N=\infty$). Let $\Fc$ denote the hypothesis set of functions $\f$. 
Since multitask problem is task-aware, that is, the task identification of each data is given during training and test, we can rewrite samples in $\Sc_t$ as $\{(\x_i,y_i,t_i\equiv t)\}_{i=1+(t-1)N}^{tN}$ and the overall multitask training dataset can be seen as $\Scb=\{(\x_i,y_i,t_i)\}_{i=1}^{NT}$. Letting $\f(\x,t)=f_t(\x)=h_t\circ\bphi_{\alpha_t}(\x)$, the loss functions can be rewritten by $\widehat\Lc_{\Scb}(\f)=\frac{1}{NT}\sum_{i=1}^{NT}\ell(\f(\x_i,t_i),y_i)$ and $\Lc_{\Dcb}(\f)=\E[\widehat\Lc_{\Scb}(\f)]$. In the following, we drop the subscript $\Dcb$ and $\Scb$ for cleaner notations.
Then we have
\begin{align}
\underset{\RMTL(\hat\f)}{\underbrace{\Lc(\hat \f)-\Lc(\f^\star)}}=\underset{a}{\underbrace{\Lc(\hat \f)-\widehat\Lc(\hat \f)}} + \underset{b}{\underbrace{\widehat\Lc(\hat \f)-\widehat\Lc(\f^\star)}} + \underset{c}{\underbrace{\widehat\Lc(\f^\star)-\Lc(\f^\star)}},
\end{align}
where $b\leq0$ because of the fact that $\hat \f$ is the empirical risk minimizer of $\widehat\Lc(\f)$. 
Then, following the proof of Theorem~3.3 of \cite{mohri2018foundations}, we make two observations: 1) Their Equation (3.8) in the proof still holds when we restrict $N$ i.i.d samples in each task instead of $NT$ i.i.d. samples over distribution $\Dcb$. Therefore, the symmetrization augment does not change, and this theorem holds under our setting. 2) The identical results hold for any function set mapping to $[-1,1]$.
In this work, based on these two observations, following Assumption~\ref{assum:lip2} and Theorem~11.3 in \cite{mohri2018foundations}, we have that with probability at least $1-\delta/2$, $a,c\leq2\Gamma\Rc_{NT}(\Fc)+\sqrt{\frac{\log(2/\delta)}{2NT}}$. Therefore, we can conclude that with probability at least $1-\delta$,

\begin{align}
    \RMTL(\hat\f)&\leq4\Gamma\Rc_{NT}(\Fc)+\sqrt{\frac{2\log\frac{2}{\delta}}{NT}}, \\
    \text{and similarly,}~~~\RMTL(\hat\f)&\leq4\Gamma\widehat\Rc_{\X}(\Fc)+3\sqrt{\frac{2\log\frac{4}{\delta}}{NT}},
\end{align}
where $\widehat\Rc_{\X}(\Fc)$ is the empirical complexity with respect to the inputs $\X$ and $\Rc_{NT}(\Fc)$ is the Rademacher complexity with respect to the sample size $NT$. 
Exercise~5.5 in \cite{wainwright2019high} shows that Rademacher complexity can be bounded in terms of Gaussian complexity, that is $\widehat\Rc_\X(\Fc)\leq\sqrt{\frac{\pi}{2}}\widehat\Gc_\X(\Fc)$ and $\Rc_{NT}(\Fc)\leq\sqrt{\frac{\pi}{2}}\Gc_{NT}(\Fc)$. 
Combining them together, we have that with probability at least $1-\delta$,
\begin{align}
    \RMTL(\hat\f)\leq6\Gamma\Gc_{NT}(\Fc)+2\sqrt{\frac{\log\frac{2}{\delta}}{NT}},~~\text{and}~~\RMTL\leq6\Gamma\widehat\Gc_\X(\Fc)+6\sqrt{\frac{\log\frac{4}{\delta}}{NT}}.\label{RMTL gaussian bound}
\end{align}
In what follows, we will move to Gaussian complexity instead. 
Now, it remains to decompose the Gaussian complexity of a set of composition functions $\Fc$ into basic function sets $\Hc,~\Ac$ and $\{\Psi_{\ell}\}_{\ell=1}^L$.
We will first bound the empirical Gaussian complexity with respect to any training inputs $\X$, which turns to be worst-case Gaussian complexity defined in Definition~\ref{def:worst-case}. Then, population complexity is simply bounded by the worst-case Gaussian complexity.

Inspired by \cite{tripuraneni2020theory}, we use the Dudley's entropy integral bound showed in \cite{wainwright2019high} (Theorem~5.22) to derive the upper bound.  Define $Z_{\f}:=\frac{1}{\sqrt{NT}}\sum_{i=1}^{NT}g_{i}\f(\x_{i},t_i)$ where $\f\in\Fc$ and $g_{i}$s are standard random Gaussian variables. 
Sine $Z_{\f}$ has zero-mean, we have $\widehat\Gc_{\X}(\Fc)=\frac{1}{\sqrt{NT}}\E_{\g}[\sup_{\f\in\Fc}Z_{\f}]\leq\frac{1}{\sqrt{NT}}\E_{\g}[\sup_{\f,\f'\in\Fc}(Z_{\f}-Z_{\f'})]$. Following Definition~\ref{def:distance metric}, let $\rho_{\X}(\f,\f')=\sqrt{\frac{1}{NT}\sum_{i=1}^{NT}(\f(\x_{i},t_i)-\f'(\x_{i},t_i))^2}$. 
Define $D_\X=\sup_{\f,\f'\in\Fc}\rho_\X(\f,\f'){\leq 2D_\Xc}$. Following Theorem~5.22 in \cite{wainwright2019high}, we have that for any $\eps\in[0,D_\X]$,
\begin{align}
    \E_{\g}\left[\sup_{\f,\f'\in\Fc}(Z_{\f}-Z_{\f'})\right]\leq2\E_{\g}\left[\underset{\underset{\rho_\X(\f,\f')\leq\eps}{\f,\f'\in\Fc}}{\sup}(Z_{\f}-Z_{\f'})\right]+32\int_{\eps/4}^{D_\X}\sqrt{\log \Nc(u;\Fc,\rho_\X)}du,\label{dudley bound}
\end{align}
where $\Nc(u;\Fc,\rho_\X)$ is the $u$-covering number of function set $\Fc$ with respect to metric $\rho_\X(\cdot,\cdot)$ following Definition~\ref{def:cover number}. 

The first term in the right hand side above is easy to bound. As shown in proof of Theorem~7 in \cite{tripuraneni2020theory}, we have $\E_\g[\sup_{\rho_\X(\f,\f')\leq\eps}(Z_{\f}-Z_{\f'})]\leq\E_\g[\sup_{\|\vb\|_2\leq\eps}\g^\top\vb]\leq \E_\g[\sup_{\|\vb\|_2\leq\eps}\|\g\|_2\|\vb\|_2]=\sqrt{NT}\eps$. Next, it remains to bound the integral term. Here, since $\f\in\Fc$ is a sophisticated function composed with $\psi_\ell^k\in\Psi_{\ell}, \alpha_t\in\Ac$ and $h_t\in\Hc$, its covering number is not well-defined. Hence, instead, we relate the cover of $\Fc$ to the covers of basic function sets, $\Psi_\ell$, $\Ac$ and $\Hc$. To this end, we need to decompose the distance metric $\rho_\X$ into distances over basic sets. Since $\Ac$ is a discrete set with cardinality $|\Ac|$. Let $\Fc^{\bal}\subset\Fc$ be the function set given pathways of all tasks $\bal$. Then we have $\log\Nc (u;\Fc,\rho_\X)\leq T\log|\Ac|+\max_{\bal\in\Ac^T}\log\Nc (u;\Fc^{\bal},\rho_\X)$. For any $\f,\f'\in\Fc^\bal$, we have
\begin{align*}
    \rho_\X(\f,\f')&=\sqrt{\frac{1}{NT}\sum_{i=1}^{NT}\left(\f(\x_{i},t_i)-\f'(\x_{i},t_i)\right)^2}
    =\sqrt{\frac{1}{NT}\sum_{t=1}^T\sum_{i=1}^{N}\left(h_{t}\circ\bphi_{\alpha_{t}}(\x_{ti})-h'_{t}\circ\bphi'_{\alpha_{t}}(\x_{ti})\right)^2}\\
    &\leq\underset{d}{\underbrace{\sqrt{\frac{1}{NT}\sum_{t=1}^T\sum_{i=1}^{N}\left(h_t\circ\bphi_{\alpha_t}(\x_{ti})-h'_t\circ\bphi_{\alpha_t}(\x_{ti})\right)^2}}}
    +\underset{e}{\underbrace{\sqrt{\frac{1}{NT}\sum_{t=1}^T\sum_{i=1}^{N}\left(h_{t}'\circ\bphi_{\alpha_{t}}(\x_{ti})-h'_{t}\circ\bphi'_{\alpha_{t}}(\x_{ti})\right)^2}}}.
\end{align*}
To proceed, let us introduce some notations. For any function $\phi$ with inputs $\X=\{\x_1,\dots,\x_n\}$, define output set w.r.t. the inputs $\X$ by $\phi(\X)=\{\phi(\x_1),\dots,\phi(\x_n)\}$. In the multipath setting, since different tasks have different pathways, different modules are chosen by different set of tasks. Given $\bal$, the task clustering methods in different layers are determined. Let $\Ic_{\ell}^k$ denote the union of task IDs who select $(\ell,k)$'th module, and $\Ic_{\ell}^k,\ell\in[K_\ell]$ are disjoint sets satisfying $\bigcup_{k=1}^{K_\ell}\Ic_{\ell}^k=[T]$. What's more, let $\Z_\ell^k$ denote the latent inputs of $(\ell,k)$'th module, where we have
\begin{align}
    \Z_{\ell}^k=\bigcup_{t\in\Ic_{\ell}^k}\psi_{\ell-1}^{\alpha_t}\dots\circ\psi_{1}^{\alpha_t}(\X_t),~~~~1<\ell\leq L,\label{Z}
\end{align} 
and $\Z_1^k=\bigcup_{t\in\Ic_1^k}\X_t$. 
In short, $(\ell,k)$'th module (whose function is $\psi_\ell^k$) is utilized by tasks $\Ic_\ell^k$ with latent inputs $\Z_\ell^k$. The inputs of heads are 
\begin{align*}
    \Z_{\Hc}^t=\psi_{L}^{\alpha_t}\dots\circ\psi_{1}^{\alpha_t}(\X_t)=\bphi_{\alpha_t}(\X_t),~~~~\forall~t\in[T].
\end{align*} 
Then we can obtain that 
\begin{align*}
    (d)&=\sqrt{\frac{1}{T}\sum_{t=1}^T\frac{1}{N}\sum_{i=1}^{N}\left(h_t\circ\bphi_{\alpha_t}(\x_{ti})-h'_t\circ\bphi_{\alpha_t}(\x_{ti})\right)^2}\leq\sqrt{\frac{1}{T}\sum_{t=1}^T\rho^2_{\Z_\Hc^t}(h_t,h_t')},\\
    (e)&\leq\Gamma\sqrt{\frac{1}{NT}\sum_{t=1}^T\sum_{i=1}^N\left\|\bphi_{\alpha_t}(\x_{ti})-\bphi'_{\alpha_t}(\x_{ti})\right\|^2}\\
    &\leq\Gamma\sum_{\ell=1}^L\Gamma^{L-\ell}\sqrt{\frac{1}{K_\ell}\sum_{k=1}^{K_\ell}\frac{1}{|\Z_\ell^k|}\sum_{\z_i\in\Z_\ell^k}\left\|\psi_\ell^k(\z_i)-{\psi'}_\ell^k(\z_i)\right\|^2}\\
    &\leq\sum_{\ell=1}^L\Gamma^{L-\ell+1}\sqrt{\frac{1}{K_\ell}\sum_{k=1}^{K_\ell}\rho^2_{\Z_\ell^k}(\psi_\ell^k,{\psi'}_\ell^k)}.
\end{align*}
Here $|\Z_\ell^k|=|\Ic_\ell^k|N$ is the number of samples used in training $(\ell,k)$'th module.  
The result follows the fact that all functions $h\in\Hc$, $\psi_\ell^k\in\Psi_\ell,\ell\in[L],k\in[K_\ell]$ are $\Gamma$-Lipschitz, and it also applies an implicit chain rule for composition Lipschitz functions. Now, we decompose distance $(d)$ into distances of each head function $h_t,t\in[T]$, with inputs $\Z_\Hc^t$, and decompose distance $(e)$, which captures the distance of composition functions $\bphi$ and $\bphi'$, into distances of module functions $\psi_\ell^k,{\psi'}_\ell^k,\ell\in[L],k\in[K_\ell]$, w.r.t. inputs of $\psi_\ell^k$, $\Z_\ell^k$. Combining them together and assuming $\rho_{\Z_\Hc^t}(h_t,h_t')\leq\eps'$ and $\rho_{\Z_\ell^k}(\psi_\ell^k,{\psi'}_\ell^k)\leq\eps'$ for all $t\in[T],\ell\in[L]$ and $k\in[K^\ell]$, we can obtain
\begin{align*}
    \rho_\X(\f,\f')\leq\sqrt{\frac{1}{T}\sum_{t=1}^T\rho^2_{\Z_\Hc^t}(h_t,h_t')}+\sum_{\ell=1}^L\Gamma^{L-\ell+1}\sqrt{\frac{1}{K_\ell}\sum_{k=1}^{K_\ell}\rho^2_{\Z_\ell^k}(\psi_\ell^k,{\psi'}_\ell^k)}\leq\left(1+\sum_{\ell=1}^L\Gamma^{L-\ell+1}\right)\eps':=\GAMMA\eps'.
\end{align*}
It shows that given pathway assignments $\bal$ and inputs $\X$, $\eps'$-covers of all heads and modules result in $(\GAMMA\eps)$-cover of $\Fc^\bal$.
Recalling that $\log\Nc (u;\Fc,\rho_\X)\leq T\log|\Ac|+\max_{\bal\in\Ac^T}\log\Nc (u;\Fc^{\bal},\rho_\X)$, we have
\begin{align}
    \log\Nc \left(\GAMMA\eps';\Fc,\rho_\X\right)&\leq T\log|\Ac|+\max_{\bal\in\Ac^T}\log\Nc \left(\GAMMA\eps';\Fc^\bal,\rho_\X\right)\\
    &\leq T\log|\Ac|+\max_{\bal\in\Ac^T}\left(\sum_{t=1}^T\log\Nc \left(\eps';\Hc,\rho_{\Z_\Hc^t}\right)+\sum_{\ell=1}^L\sum_{k=1}^{K_\ell}\log\Nc \left(\eps';\Psi_\ell,\rho_{\Z_\ell^k}\right)\right).\label{log cover bound}
\end{align}

Till now, we have decomposed the covering number of $\Fc^\bal$ into product of covering numbers of all basic function sets $\Hc,\Psi_\ell,\ell\in[L]$. Next, following \cite{tripuraneni2020theory}, and the Sudakov minoration theorem (Theorem~5.30) and Lemma~5.5 in \cite{wainwright2019high}, and recalling Definition~\ref{def:worst-case}, we have that for any $\eps'>0$,
\begin{align*}
    \max_{\bal\in\Ac^T}\sum_{t=1}^T{\log\Nc\left(\eps';\Hc,\rho_{\Z_\Hc^t}\right)}
    &\leq\max_{\bal\in\Ac^T}\sum_{t=1}^T\left(\frac{2\sqrt{N}}{\eps'}\widehat\Gc_{\Z_\Hc^t}(\Hc)\right)^2\leq T\left(\frac{2\sqrt{N}}{\eps'}\widetilde\Gc_N^{\Xc_\Hc}(\Hc)\right)^2,\\
    \max_{\bal\in\Ac^T}\sum_{k=1}^{K_\ell}{\log\Nc \left(\eps';\Psi_\ell,\rho_{\Z_\ell^k}\right)}
    &\leq\max_{\bal\in\Ac^T}\sum_{k=1}^{K_\ell}\left(\frac{2\sqrt{|\Z_\ell^k|}}{\eps'}\widehat\Gc_{\Z_\ell^k}(\Psi_\ell)\right)^2\leq\max_{\bal\in\Ac^T}\sum_{k=1}^{K_\ell}\left(\frac{2\sqrt{|\Z_\ell^k|}}{\eps'}\widetilde\Gc^{\Xc_{\Psi_{\ell}}}_{|\Z_\ell^k|}(\Psi_\ell)\right)^2\\
    &\leq K_\ell\left(\frac{2\sqrt{2NT}}{\eps'}\widetilde\Gc^{\Xc_{\Psi_{\ell}}}_{NT}(\Psi_\ell)\right)^2,
\end{align*}
where the input spaces for $\Hc$ and $\Psi_\ell$ are $\Xc_\Hc=\Psi_L\circ\dots\Psi_1\circ\Xc$, $\Xc_{\Psi_\ell}=\Psi_{\ell-1}\circ\dots\Psi_1\circ\Xc$ for $\ell>1$ and $\Xc_{\Psi_1}=\Xc$. The last inequality is drawn from Lemma~\ref{lemma:worst case}, which shows $\sqrt{|\Z_\ell^k|}\widetilde\Gc^{\Xc_{\Psi_{\ell}}}_{|\Z_\ell^k|}(\Psi_\ell)\leq\sqrt{2NT}\Gc^{\Xc_{\Psi_{\ell}}}_{NT}(\Psi_\ell)$.
Since Definition~\ref{def:worst-case} eliminates the input($\X$)-dependency, the inequalities hold for any valid inputs $\X$.  In what follows, we drop the superscripts from the worst-case Gaussian complexities for cleaner exposition as they are clear from context. Then, setting $\eps'=\frac{u}{\GAMMA}$, applying triangle inequality, we can obtain that for any $\X$,
\begin{align}
    \sqrt{\log\Nc \left(u;\Fc,\rho_\X\right)}\leq\sqrt{T\log|\Ac|}+\frac{2\GAMMA\sqrt{NT}}{u}\widetilde\Gc_N(\Hc)+\sum_{\ell=1}^L\frac{2\GAMMA\sqrt{2K_\ell NT}}{u}\widetilde\Gc_{NT}(\Psi_\ell).\label{cover bound}
\end{align}

Now it is time to combine everything together! Recall \eqref{RMTL gaussian bound}, \eqref{dudley bound} and \eqref{cover bound}. Since, $D_\X{\leq 2D_\Xc}$ for any inputs $\X$, choosing $\eps=\frac{8D_\Xc}{NT}$, we can obtain that with probability at least $1-\delta$,
\begin{align*}
    \RMTL(\hat\f)&\leq6\Gamma\Gc_{NT}(\Fc)+2\sqrt{\frac{\log\frac{2}{\delta}}{NT}}\\
    &\leq12\Gamma\left(\eps+32D_\Xc\sqrt{\frac{\log|\Ac|}{N}}+32\GAMMA\left(\widetilde\Gc_{N}(\Hc)+\sum_{\ell=1}^L\sqrt{2K_\ell}\widetilde\Gc_{NT}(\Psi_\ell)\right)\int_{\eps/4}^{2D_\Xc}\frac{1}{u}du\right)+2\sqrt{\frac{\log\frac{2}{\delta}}{NT}}\\
    &\leq768\Gamma\left(\frac{D_\Xc}{NT}+D_\Xc\sqrt{\frac{\log|\Ac|}{N}}+\GAMMA\log NT\left(\widetilde\Gc_{N}(\Hc)+\sum_{\ell=1}^L\sqrt{K_\ell}\widetilde\Gc_{NT}(\Psi_\ell)\right)\right)+2\sqrt{\frac{\log\frac{2}{\delta}}{NT}}.
\end{align*} 
Till now, we have obtained the result for general $\hat\f$. Finally, consider the case that $\hat\f$ might not utilize all the modules in the supernet. Let $\hat\K_\ell\leq\K_\ell$ be the number of modules used by the empirical solution $\hat\f$. Applying Lemma~\ref{lem:only utilized}, we can now replace $\Phi$ with $\Phi_\used$ which replaces $K_\ell$ with $\hat K_\ell$ for $\ell\in [L]$, which concludes our final result.
\end{proof}

\noindent $\bullet$ \textbf{{Developing an input-dependent bound.}} In Theorem~\ref{thm:main}, we present the bound of Multipath MTL problem based on the worst-case Gaussian complexity. However, as shown in Definition~\ref{def:worst-case}, it computes the complexity of a function set by searching for the worst-case latent inputs, which ignores the data distribution and how the data collected as tasks. In the following argument, we present an input-based guarantee that bounds the excess risk of Multipath MTL problem tightly. To begin with, recall that $\X=\{\X_t\}_{t=1}^T$ and $\X_t=\{\x_{ti}\}_{i=1}^N$ denote the actual raw feature sets. Given inputs in $T$ tasks, we can define the \emph{empirical} worst-case Gaussian complexities of $\Hc$ and $\Psi_\ell,\ell\in[L]$ as follows.
\begin{align*}
    C_\X^\Hc=\max_{t\in[T]}\sup_{\Z\in\Zc_t}\Gh_{\Z}(\Hc),~~~&\text{where}~~\Zc_t=\Psi_L\circ\dots\Psi_1\left(\X_t\right),\\
    C_\X^{\Psi_\ell}=\max_{\Ic\subset[T]}\sup_{\Z\in\Zc_\Ic}\sqrt{\frac{|\Ic|}{T}}\Gh_\Z(\Psi_\ell),~~~&\text{where}~~\Zc_\Ic=\bigcup_{t\in\Ic}\Psi_{\ell-1}\circ\dots\Psi_1\left(\X_t\right),
\end{align*}
where $\Gh_{\Z}(\Hc)$ and $\Gh_\Z(\Psi_\ell)$ are empirical Gaussian complexities and input spaces of $\Hc$ and $\Psi_\ell$ are corresponding to the raw input $\X$. Then, such statement provide another method to bound \eqref{log cover bound}. That is, we have for any $\eps'>0$,
\begin{align*}
    \max_{\bal\in\Ac^T}\sum_{t=1}^T{\log\Nc\left(\eps';\Hc,\rho_{\Z_\Hc^t}\right)}
    &\leq\sum_{t=1}^T\left(\frac{2\sqrt{N}}{\eps'}\max_{\bal\in\Ac^T}\widehat\Gc_{\Z_\Hc^t}(\Hc)\right)^2\leq T\left(\frac{2\sqrt{N}}{\eps'}C_\X^\Hc\right)^2,\\
    \max_{\bal\in\Ac^T}\sum_{k=1}^{K_\ell}{\log\Nc \left(\eps';\Psi_\ell,\rho_{\Z_\ell^k}\right)}
    &\leq\sum_{k=1}^{K_\ell}\left(\frac{2\sqrt{NT}}{\eps'}\max_{\bal\in\Ac^T}\sqrt{\frac{|\Z_\ell^k|}{NT}}\widehat\Gc_{\Z_\ell^k}(\Psi_\ell)\right)^2\leq K_\ell\left(\frac{2\sqrt{NT}}{\eps'}C_\X^{\Psi_\ell}\right)^2.
\end{align*}
The statements provided to prove Theorem~\ref{thm:main} utilize the worst-case Gaussian complexity, and it bounds both empirical and population Gaussian complexities. Here, $C_X^\Hc$ and $C_\X^{\Psi_\ell}$ depend on the input $\X$, and by construction, they are larger than their corresponding empirical complexities, however there is no guarantee that they will be larger than the corresponding population Gaussian complexities. Combining the result with \eqref{RMTL gaussian bound}, we can obtain that with probability at least $1-\delta$,
\begin{align}
    \RMTL(\hat\f)\leq384\Gamma\left(\frac{D_\X}{NT}+D_\X\sqrt{\frac{\log|\Ac|}{N}}+\GAMMA\log NT\left(C_\X^\Hc+\sum_{\ell=1}^L\sqrt{K_\ell}C_\X^{\Psi_\ell}\right)\right)+6\sqrt{\frac{\log\frac{4}{\delta}}{NT}},\label{empirical GC bound}
\end{align}
where $D_\X=\sup_{\f,\f'\in\Fc}\rho_\X(\f,\f')$. Here we consider complexity of each task-specific head separately and bound it using the task with the largest head complexity ($C_\X^\Hc$). As for the complexity of each layer, in the general case (as shown in Theorem~\ref{thm:main}), all the modules in the same layer share the same input space $\Xc_{\Psi_\ell}$ by assuming raw input space $\Xc$, and because of Lemma~\ref{lemma:worst case}, the sample complexity of $\ell_\tth$ layer is bounded by $\order{{\sqrt{K_\ell}\widetilde\Gc_{NT}(\Psi_\ell)}}$. When given actual training data $\X$, we need to search to find the worst-case cluster method of $\ell_\tth$ layer, which results in $C_\X^{\Psi_\ell}$.

Below, we extend our theoretical result of Multipath MTL to two specific settings, vanilla MTL and hierarchical MTL.
\begin{corollary}[Vanilla MTL]\label{corol:vanilla}
    Given the same data setting described in Section~\ref{sec:setup}, consider a vanilla MTL problem as depicted in Figure~\ref{fig:vanilla}, which can be formulated as follows. 
    \begin{align*}
        \{\hat h_t\}_{t=1}^T,\hat\phi=\underset{h_t\in\Hc,\phi\in\Phi}{\arg\min}\frac{1}{NT}\sum_{t=1}^T\sum_{i=1}^N\ell(h_t\circ\phi(\x_{ti}),y_{ti}).
    \end{align*}
    Suppose $\Hc$, $\Phi$ are sets of $\Gamma$-Lipschitz functions with respect to Euclidean norm, and $\ell(\cdot,y):\R\times\R\rightarrow[0,1]$ is also $\Gamma$-Lipschitz with respect to Euclidean norm. Define $\Dc_\Xc=\sup_{\x\in\Xc,h\in\Hc,\phi\in\Phi}|h\circ\phi(\x)|<\infty$. 
    Let $\Lc(\{h_t\}_{t=1}^T,\phi)=\E_\Dcb[\ell(h_t\circ\phi(\x),y)]$ and $\Lc^\st=\min_{h_t\in\Hc,\phi\in\Phi}\E_\Dcb[\ell(h_t\circ\phi(\x),y)]$. Then we have that with probability at least $1-\delta$,
    \begin{align*}
        \Lc(\{\hat h_t\}_{t=1}^T,\hat\phi)-\Lc^\st\leq384\Gamma\left(\frac{D_\Xc}{NT}+(\Gamma+1)\log NT\left(\widetilde\Gc_{N}(\Hc)+\Gc_{NT}(\Phi)\right)\right)+2\sqrt{\frac{\log\frac{2}{\delta}}{NT}}.
    \end{align*}
    Here, the input space for $\Hc$ is $\Psi\times\Xc$.
\end{corollary}

This corollary is consistent with \cite{tripuraneni2020theory}, and it can be simply deduced following the statement of Theorem~\ref{thm:main restate}, by setting $L=1$, $K_1=1$. Since there is only one pathway selection, $|\Ac|=1$ and $\log|\Ac|=0$. Here, the input space for representation $\Phi$ is $\Xc$, and its complexity is shown in Gaussian complexity fashion.

\begin{corollary}[Hierarchical MTL] Consider the hierarchical MTL problem depicted in Fig.~\ref{fig:hierarchical} and consider a hierarchical supernet with degree $K$. Follow the same settings in Section~\ref{sec:setup}. Suppose Assumptions 1\&2 hold. Let $\hat\f$ be the empirical solution of \eqref{formula:multipath}. Let $D_\Xc=\sup_{\x\in\Xc,h\in\Hc,\bphi\in\Phi,\alpha\in\Ac}|h\circ\bphi_\alpha(\x)|<\infty$ and $\GAMMA=\sum_{\ell=0}^L\Gamma^\ell$. Then, with probability at least $1-\delta$, the excess test risk in \eqref{risk mtl} obeys
    \begin{align*}
        \RMTL(\hat\f)\leq768\Gamma\left(\frac{D_\Xc}{NT}+D_\Xc\sqrt{\frac{(L-1)\log K}{N}}+\GAMMA\log NT\left(\widetilde\Gc_{N}(\Hc)+\sum_{\ell=1}^LK^{\frac{\ell-1}{2}}\widetilde\Gc_{NT}(\Psi_\ell)\right)\right)+2\sqrt{\frac{\log\frac{2}{\delta}}{NT}}.
    \end{align*}
    Here, the input spaces for $\Hc$ and $\Psi_\ell$ are $\Xc_\Hc=\Psi_L\circ\dots\Psi_1\circ\Xc$, $\Xc_{\Psi_\ell}=\Psi_{\ell-1}\circ\dots\Psi_1\circ\Xc$ for $\ell>1$, and $\Xc_{\Psi_1}=\Xc$.
    Now if we consider a two-layer hierarchical representations as depicted in Fig.~\ref{fig:intro_tree}, we can immediately obtain the result by setting $L=2$ ($\GAMMA=1+\Gamma+\Gamma^2$). Then with probability at least $1-\delta$,
    \begin{align*}
        \RMTL(\hat\f)\leq768\Gamma\left(\frac{D_\Xc}{NT}+D_\Xc\sqrt{\frac{\log K}{N}}+\GAMMA\log NT\left(\widetilde\Gc_{N}(\Hc)+\Gc_{NT}(\Psi_1)+\sqrt{K}\widetilde\Gc_{NT}(\Psi_2)\right)\right)+2\sqrt{\frac{\log\frac{2}{\delta}}{NT}}.
    \end{align*}
\end{corollary}

The result is consistent with Section~\ref{hierarchy}, and proof can be immediately done by setting $|\Ac|=K^{L-1}$ and $K_\ell=K^{\ell-1}$ in Theorem~\ref{thm:main restate}. Here we observe that if the complexity of $\Psi_\ell$ decreasing exponentially as $\comp(\Psi_\ell)\propto{K^{-\frac{\ell}{2}}}$, then each layer has a constant complexity. We believe this and similar bounds can potentially provide guidelines on how we should design hierarchical supernets.


\subsection{Proof of Lemma~\ref{no harm}}
\begin{lemma}[Lemma~\ref{no harm} restated]\label{no harm2}
    Recall $\hat\f$ is the solution of \eqref{formula:multipath} and $\hat f_t=\hat h_t\circ\hat\bphi_{\hat\alpha_t}$ is the associated task-$t$ hypothesis. Define the excess risk of task $t$ as $\Rt(\hat f_t)=\Lc_t(\hat f_t)-\Lc^\st_t$ where $\Lc_t(f)=\E_{\Dc_t}[\widehat\Lc_t(f)]$ is the population risk of task $t$ and $\Lc^\st_t$ is the optimal achievable test risk for task $t$ over $\Fc$.  With probability at least $1-\delta-\P(\widehat\Lc_{\Scb}(\hat\f)\neq0)$, for all tasks $t\in[T]$,
    \begin{align}
        \Rt(\hat f_t)\lesssim\Gt_N(\Hc)+\sum_{\ell=1}^L\Gt_N(\Psi_\ell)+\sqrt{\frac{\log(2T/\delta)}{N}}.\label{individual bounds}
    \end{align}
\end{lemma}
\begin{proof} Let $\Fc_{\text{IND}}$ be the hypothesis class of a single task induced by a pathway in the supernet. Since modules are same, $\Fc_{\text{IND}}$ is same regardless of pathway. First, applying our main theorem (Thm \ref{thm:main}) for a single supernet with $K_\ell=1$ (i.e.~on $\Fc_{\text{IND}}$), for a single task $t$, we end up with the uniform concentration guarantee, for all $f\in \Fc_{\text{IND}}$, with probability at least $1-\delta$,
\[
|\Lch_{\Sc_t}( f)-\Lc_t( f)|\lesssim\Gt_N(\Hc)+\sum_{\ell=1}^L\Gt_N(\Psi_\ell)+\sqrt{\frac{\log(2/\delta)}{N}}.
\]
Union bounding, for all $f_t\in\Fc_{\text{IND}}$, $t\in [T]$, with probability at least $1-\delta$, we obtain
\begin{align}
|\Lch_{\Sc_t}( f_t)-\Lc_t( f_t)|\lesssim\Gt_N(\Hc)+\sum_{\ell=1}^L\Gt_N(\Psi_\ell)+\sqrt{\frac{\log(2T/\delta)}{N}}.\label{worst guarantee}
\end{align}
Let us call this intersection event $\Ec_\all$. Intersecting this with the events $\min_{f_t\in\Fc_{\text{IND}}} \Lch_{\Sc_t}(f_t)=0$ for $t\in [T]$, we exactly end up with \eqref{individual bounds}. Thus, the statement is indeed what one would obtain by union bounding individualized training.

To proceed, we argue that same bound holds when solving \eqref{formula:multipath}. We know \eqref{worst guarantee} holds for all $f_t$ chosen from $\Fc_\text{IND}$, therefore it holds for $\hat f_t$, $t\in[T]$. Consider its intersection with the event $\P(\widehat\Lc_{\Scb}(\hat\f)\neq0)$. Given that $\Lch_{\Sc_t}(\hat f_t)=0$, we obtain $\Rt(\hat f_t)\leq \Lc_t(\hat f_t)$ upper bounded by the RHS of \eqref{worst guarantee}.

\end{proof}

\subsection{Proof of Theorem~\ref{thm:tfl}}
\begin{theorem}[Theorem~\ref{thm:tfl} restated]\label{thm:tfl restate} Suppose Assumptions 1\&2 hold. Let supernet $\hat\bphi$ be the solution of \eqref{formula:multipath} and $\hat f_{\hat{\bphi}}$ be the empirical minima of \eqref{formula:transfer} with respect to supernet $\hat\bphi$. Let $D_\Xc=\sup_{\x\in\Xc,\alpha\in\Ac,h\in\Hc_\tgt}|h\circ\hat\bphi_\alpha(\x)|<\infty$. Then with probability at least $1-\delta$,
    \begin{align*}
        \RTFL(\hat f_{\hat\bphi})\leq\Bias_{\tgt}(\hat\bphi)+768\Gamma\left(\frac{D_\Xc}{M}+D_\Xc\sqrt{\frac{\log|\Ac|}{M}}+\log M\cdot\widetilde\Gc_{M}(\Hc_\tgt)\right)+2\sqrt{\frac{\log\frac{2}{\delta}}{M}},
    \end{align*}
    where input space of $\Gt_M(\Hc_\tgt)$ is given by $\{\hat\bphi_\alpha\circ\Xc\bgl \alpha\in\Ac\}$.
\end{theorem}
\begin{proof}
	For short notation, let $\Hc:=\Hc_\tgt$. 
	We consider the transfer learning problem over a target task, with distribution $\Dc_\tgt$ and training dataset $\Sc_\tgt=\{(\x_i,y_i)\}_{i=1}^M$ with $M$ samples i.i.d. drawn from $\Dc_\tgt$. 
	Let $\hat\bphi$ and $\bphi^\star$ denote the empirical and population solution of (\ref{formula:multipath}). 
	Then, we can recap the excess transfer learning risk 
\begin{align}
    \RTFL&(\hat f_{\hat\bphi})=\Lc_{\tgt}(\hat f_{\hat\bphi})-\Lc_{\tgt}^\st=\underset{\text{variance}(a)}{\underbrace{\Lc_{\tgt}(\hat f_{\hat\bphi})-\Lc_{\tgt}(f^\star_{\hat\bphi})}}+\underset{\text{\bias}(b)}{\underbrace{\Lc_{\tgt}( f^\star_{\hat\bphi})-\Lc_{\tgt}^\st}}.\nn
\end{align}
Following Definition~\ref{def:model distance}, $b=\Bias_\Tc(\hat\bphi)$, and it remains to bound variance $(a)$. Let $\hat f_{\hat\bphi}:=(\hat h_{\hat\bphi},\hat\alpha_{\hat\bphi})$ and $f^\st_{\hat\bphi}:=(h^\st_{\hat\bphi},\alpha^\st_{\hat\bphi})$. For short notations, we remove the subscript $\hat\bphi$, and we assume supernet $\hat\bphi$ is implied. Following the similar statements in Appendix~\ref{app:proof mtl}, we can decompose variance as follows.
\begin{align*}
	a=\Lc_{\tgt}(\hat f)-\Lc_{\tgt}(f^\st)=\underset{c}{\underbrace{\Lc_{\tgt}(\hat f)-\widehat\Lc_{\tgt}(\hat f)}}+\underset{d}{\underbrace{\widehat\Lc_{\tgt}(\hat f)-\widehat\Lc_{\tgt}(f^\st)}} +\underset{e}{\underbrace{\widehat\Lc_{\tgt}(f^\st)-\Lc_{\tgt}(f^\st)}}
\end{align*}
where $\Lc_\tgt(f)=\E_{\Dc_\tgt}[\ell(h\circ\hat\bphi_\alpha(\x),y)]$ and $\widehat \Lc_\tgt(f)=\frac{1}{M}\sum_{i=1}^M\ell(h\circ\hat\bphi_\alpha(\x_i),y_i)$ where $f=(h,\alpha)$ and $(\x_i,y_i)\in\Sc_\tgt$. Since $\hat f$ minimizes the training loss given $\hat\bphi$, $d\leq0$. Let $\X$ denote the input dataset, that is, $\X=\{\x_i\}_{i=1}^M$. Same as Inequality~\eqref{RMTL gaussian bound} in Appendix~\ref{app:proof mtl}, we derive the similar result that with probability at least $1-\delta$,
\begin{align*}
	&\RTFL(\hat f)\leq\Bias_{\tgt}(\hat\bphi)+6\Gamma\Gc_{M}(\Hc\circ\hat\bphi(\Ac))+2\sqrt{\frac{\log\frac{2}{\delta}}{M}},\\~~\text{and}~~&\RTFL(\hat f)\leq\Bias_{\tgt}(\hat\bphi)+6\Gamma\widehat\Gc_{\X}(\Hc\circ\hat\bphi(\Ac))+6\sqrt{\frac{\log\frac{2}{\delta}}{M}},
\end{align*}
where $\widehat\Gc_{\X}(\Hc\circ\hat\bphi(\Ac))=\E_\g\left[\sup_{h\in\Hc,\alpha\in\Ac}\frac{1}{M}\sum_{i=1}^{M}g_ih\circ\hat\bphi_\alpha(\x_i)\right]$ and $\Gc_M(\Hc\circ\hat\bphi(\Ac))=\E_{\Dc_\tgt}\left[\widehat\Gc_{\X}(\Hc\circ\hat\bphi(\Ac))\right]$. Following the Definition~\ref{def:distance metric}, let $D=\sup_{h,h'\in\Hc,\alpha,\alpha'\in\Ac}\rho_{\X}(h\circ\hat\bphi_\alpha,h'\circ\hat\bphi_{\alpha'})\leq 2D_\Xc$. By applying the Dudley's theorem, and following the same statements in Appendix~\ref{app:proof mtl}, we obtain that given any $\eps\in[0, D]$
\begin{align*}
	\widehat\Gc_{\X}(\Hc\circ\hat\bphi(\Ac))\leq2\eps+\frac{32}{\sqrt{M}}\int_{\eps/4}^{D}\sqrt{\log\Nc\left(u;\Hc\circ\hat\bphi(\Ac),\rho_{\X}\right)}du.
\end{align*}
Now we need to decompose the covering number of $\Hc\circ\hat\bphi(\Ac)$ into the covering numbers of separate hypothesis sets $\Hc$ and $\Ac$. For short notations, let $\Hc(\Ac):=\Hc\circ\hat\bphi(\Ac)$ and $\Hc(\alpha):=\Hc\circ\hat\bphi_\alpha$, and we omit the subscript $\X$ from $\rho$. Since pathway set $\Ac$ is discrete with cardinality $|\Ac|$, the covering number of $\Hc(\Ac)$ is the product of covering number of $\Hc(\alpha)$ for all $\alpha\in\Ac$, and can be bounded by the $|\Ac|$ times product of the worst-case covering number of $\Hc(\alpha)$, that is $\Nc(u;\Hc(\Ac),\rho)=\Pi_{\alpha\in\Ac}\Nc(u;\Hc(\alpha),\rho)\leq\max_{\alpha\in\Ac}\Nc^{|\Ac|}(u;\Hc(\alpha),\rho)$. Logarithm of it results in $\log\Nc(u;\Hc(\Ac),\rho)\leq\log|\Ac|+\max_{\alpha\in\Ac}\Nc(u;\Hc(\alpha),\rho)$.  Now let $\Z_\alpha=\hat\bphi_\alpha(\X)=\{\hat\bphi_\alpha(\x_i):\x_i\in\X\}$, which is the set of latent inputs of prediction head. Then for any given $\alpha\in\Ac$, 
\begin{align*}
	&\rho_{\X}(h\circ\hat\bphi_\alpha,h'\circ\hat\bphi_{\alpha})=\sqrt{\frac{1}{M}\sum_{i=1}^M\left(h\circ\hat\bphi_\alpha(\x_i)-h'\circ\hat\bphi_{\alpha}(\x_i)\right)^2}=\sqrt{\frac{1}{M}\sum_{i=1}^M\left(h(\z_i)-h'(\z_i)\right)^2}=\rho_{\Z_\alpha}(h,h'),
\end{align*}
where $\z_i=\hat\bphi_\alpha(\x_i)$ and then $\Z_\alpha=\{\z_1,\dots,\z_M\}$. 
Such equality states that if pathway $\alpha$ is fixed, $u$-cover of head $\Hc$ results in $u$-cover of the prediction function, and simply, $\Nc(u;\Hc(\alpha),\rho_{\X})=\Nc(u;\Hc,\rho_{\Z_\alpha})$. 
Next, following the same statements in Appendix~\ref{app:proof mtl}, if we utilize the Sudakov minoration theorem in \cite[]{wainwright2019high}, we obtain $\sqrt{\log\Nc(u;\Hc,\rho_{\Z_\alpha})}\leq\frac{2\sqrt{M}}{u}\widehat\Gc_{\Z_\alpha}(\Hc)$. Finally, combining all we have together obtains
\begin{align*}
	\widehat\Gc_{\X}(\Hc(\Ac))&\leq2\eps+\frac{32}{\sqrt{M}}\int_{\eps/4}^{D}\sqrt{\log\Nc\left(u;\Hc(\Ac),\rho_{\X}\right)}du\leq2\eps+32D\sqrt{\frac{\log|\Ac|}{M}}+64\max_{\alpha\in\Ac}\widehat\Gc_{\Z_\alpha}(\Hc)\int_{\eps/4}^{D}\frac{1}{u}du\\
	&\leq2\eps+32D\sqrt{\frac{\log|\Ac|}{M}}+64\log\frac{4D}{\eps}\max_{\alpha\in\Ac}\widehat\Gc_{\Z'_\alpha}(\Hc)\leq 64\left(\frac{D}{M}+D\sqrt{\frac{\log|\Ac|}{M}}+\log M\max_{\alpha\in\Ac}\widehat\Gc_{\Z_\alpha}(\Hc)\right),
\end{align*}
by choosing $\eps=\frac{4D}{M}$. 

\noindent$\bullet$ \textbf{Input-dependent bound.} If we define the worst case \emph{empirical} Gaussian complexity based on the raw input data $\X$ and given supernet $\hat\bphi$, that is $C_\X^\Hc:=\max_{\alpha\in\Ac}\widehat\Gc_{\Z_\alpha}(\Hc)$, where $\Z_\alpha$ shows as above with respect to $\hat\bphi$ and $\alpha$, we have that with probability at least $1-\delta$,
\begin{align*}
	\RTFL(\hat f_{\hat\bphi})\leq\Bias_{\tgt}(\hat\bphi)+384\Gamma\left(\frac{D}{M}+D\sqrt{\frac{\log|\Ac|}{M}}+\log M\cdot C_\X^\Hc\right)+6\sqrt{\frac{\log\frac{4}{\delta}}{M}}.
\end{align*}

Furthermore, let input space be $\Xc$. If we define the worst case Gaussian complexity independent to the specific training dataset and supernet, that is, $\widetilde\Gc_M^{\Xc_\Hc}(\Hc):=\sup_{\Z\in\Xc_\Hc^M}\widehat\Gc_{\X}(\Hc)$, where $\Xc_\Hc=\{\hat\bphi_\alpha\circ\Xc|\alpha\in\Ac\}$, then we have that 
\begin{align*}
	\Gc_{M}(\Hc(\Ac))\leq64\left(\frac{D_\Xc}{M}+D_\Xc\sqrt{\frac{\log|\Ac|}{M}}+\log M\cdot\widetilde\Gc_M^{\Xc_\Hc}(\Hc)\right),
\end{align*}
which leads to the result that with probability at least $1-\delta$,
\begin{align*}
	\RTFL(\hat f_{\hat\bphi})\leq\Bias_{\tgt}(\hat\bphi)+384\Gamma\left(\frac{D_\Xc}{M}+D_\Xc\sqrt{\frac{\log|\Ac|}{M}}+\log M\cdot\widetilde\Gc_{M}(\Hc)\right)+2\sqrt{\frac{\log\frac{2}{\delta}}{M}}.
\end{align*}
Here input space of $\Hc$ is given by $\Xc_\Hc=\{\hat\bphi_\alpha\circ\Xc|\alpha\in\Ac\}$.
\end{proof}

\subsection{End-to-End Transfer Learning}\label{app:e2etransfer}
In this section, we present an end-to-end transfer learning guarantee based on task diversity. We start with two useful definitions: supernet distance and task diversity. Here, supernet distance has been mentioned in Section~\ref{sec:main tfl} and following provides the intact definition. It measures the performance gap of two supernets. Similar to the previous work \cite{chen2021weighted,tripuraneni2020theory,xu2021representation}, we define task diversity in Definition~\ref{def:task diversity}. It captures the similarity of target task to source tasks over a supernet by comparing their representation distance over it. Finally, using the task diversity argument, we can immediately obtain the theoretical guarantee for transfer learning risk.

\begin{definition}[Supernet Distance]\label{def:app supernet distance}Consider a transfer learning with optimal pathway \eqref{formula:transfer} problem. Recall the definitions $\Dc_\tgt$ and $\Hc_\tgt$ stated in Section~\ref{sec:setup}. Given two supernets $\bphi$ and $\bphi'$, define the supernet/representation distance of $\bphi$ from $\bphi'$ for a target $\tgt$ as
\begin{align*}
	\Dist_\tgt(\bphi;\bphi')=\Bias_\tgt(\bphi)-\Bias_\tgt(\bphi')=\min_{h\in\Hc_\tgt,\alpha\in\Ac}\Lc_\tgt(h\circ\bphi_\alpha)-\min_{h\in\Hc_\tgt,\alpha\in\Ac}\Lc_\tgt(h\circ\bphi'_\alpha).
\end{align*}	
\end{definition}
Here, we do not restrict the supernet distance to target task $\tgt$ only. Given source task $t\in[T]$, we can still define the corresponding supernet distance of $\bphi$ from $\bphi'$ as
\begin{align}
	\Dist_t(\bphi;\bphi')=\min_{h\in\Hc,\alpha\in\Ac}\Lc_t(h\circ\bphi_\alpha)-\min_{h\in\Hc,\alpha\in\Ac}\Lc_t(h\circ\bphi'_\alpha),\label{def:source supernet distance}
\end{align} 
and the hypothesis set for head is $\Hc$ instead.

\begin{definition}[Task Diversity]\label{def:task diversity}
    For any supernets $\bphi$ and $\bphi'$, given $T$ source tasks with distribution $(\Dc_t)_{t=1}^T$ and a target task with distribution $\Dc_\tgt$, we say that the source tasks are $(\nu,\epsilon)$-diverse over the target task for a supernet $\bphi'$ if for any $\bphi\in\Phi$,
    \begin{align*}
        \Dist_\tgt(\bphi;\bphi')\leq\left(\frac{1}{T}\sum_{t=1}^T\Dist_t(\bphi;\bphi')\right)/\nu+\epsilon,
    \end{align*}
    where we assume that head hypothesis sets $\Hc$, $\Hc_\tgt$ are implied for source and target distances.
\end{definition}
\begin{theorem}[End-to-end transfer learning]
	Suppose Assumption 1\&2 hold. Let supernet $\hat\bphi$ and $\bphi^\st$ be the empirical and population solutions of \eqref{formula:multipath} and $\hat f_{\hat\bphi}$ be the empirical minima of \eqref{formula:transfer} with respect to supernet $\hat\bphi$. Assume the source tasks used in Multipath MTL phase are $(\nu,\epsilon)$-diverse over target task $\tgt$ for the optimal supernet $\bphi^\st$. Then with probability at least $1-2\delta$,
	\begin{align*}
		\scalemath{0.9}{\RTFL(\hat f_{\hat\bphi})\lesssim\Bias_\tgt(\bphi^\st)+\frac{1}{\nu}\left(\Gt_N(\Hc)+\sum_{\ell=1}^L\sqrt{{\hat\K}_\ell}\Gt_{NT}(\Psi_\ell)+\sqrt{\frac{\log|\Ac|}{N}}\right)+\sqrt{\frac{\log|\Ac|}{M}}+\Gt_M(\Hc_\tgt)+\frac{1}{\nu}\sqrt{\frac{\log\frac{2}{\delta}}{NT}}+\sqrt{\frac{\log\frac{2}{\delta}}{M}}+\epsilon}.
	\end{align*}
	Here, the input spaces for $\Hc$, $\Psi_\ell$ and $\Hc_\tgt$ are same to the statements in Theorem~\ref{thm:main} and Theorem~\ref{thm:tfl}.
\end{theorem} 

\begin{proof}
	Recall Theorem~\ref{thm:tfl restate}. To state end-to-end transfer learning risk, we need to bound supernet bias $\Bias_\tgt(\hat\bphi)$. Following Definition~\ref{def:app supernet distance}, we have that $\Bias_\tgt(\hat\bphi)=\Dist_\tgt(\hat\bphi;\bphi^\st)+\Bias_\tgt(\bphi^\st)$. Next, from Definition~\ref{def:task diversity}, since we assume source tasks are $(\nu,\epsilon)$-diverse over target task $\tgt$ for the supernet $\bphi^\st$, we can obtain $\Dist_\tgt(\hat\bphi;\bphi^\st)\leq\left(\frac{1}{T}\sum_{t=1}^T\Dist_t(\hat\bphi;\bphi^\st)\right)/\nu+\epsilon$. To process, following \eqref{def:source supernet distance}, we have
	\begin{align*}
		\frac{1}{T}\sum_{t=1}^T\Dist_t(\hat\bphi;\bphi^\st)&=\frac{1}{T}\sum_{t=1}^T\left(\min_{h\in\Hc,\alpha\in\Ac}\Lc_t(h\circ\hat\bphi_\alpha)-\min_{h\in\Hc,\alpha\in\Ac}\Lc_t(h\circ\bphi^\st_\alpha)\right)\\
		&\leq\frac{1}{T}\sum_{t=1}^T\left(\Lc_t(\hat h_t\circ\hat\bphi_{\hat\alpha_t})-\Lc_t(h_t^\st\circ\bphi^\star_{\alpha_t^\star})\right)=\Lc_\Dcb(\hat\f)-\Lc^\st_\Dcb=\RMTL(\hat\f).
	\end{align*}
	Here, $(\{\hat h_t,\hat\alpha_t\}_{t=1}^T,\hat\bphi)$ and $(\{h_t^\st,\alpha^\st_t\}_{t=1}^T,\bphi^\st)$ are the empirical and population solutions of \eqref{formula:multipath}, and we set $\hat\f:=(\{\hat h_t,\hat\alpha_t\}_{t=1}^T,\hat\bphi)$. The inequality term holds from the fact that: 1) $\min_{h\in\Hc,\alpha\in\Ac}\Lc_t(h\circ\hat\bphi_\alpha)\leq\Lc_t(\hat h_t\circ\hat\bphi_{\hat\alpha_t})$, and 2) $\min_{h\in\Hc,\alpha\in\Ac}\Lc_t(h\circ\bphi^\st_\alpha)=\Lc_t(h_t^\st\circ\bphi^\star_{\alpha_t^\star})$ since $h_t^\st$ and $\alpha_t^\st$ can be seen as the optimal solutions given supernet $\bphi^\st$. Combining them together with Theorem~\ref{thm:main} and Theorem~\ref{thm:tfl} completes the proof.
\end{proof}






\section{\MP MTL under Subexponential Loss Functions}\label{subexp loss}

The goal of this section is proving an MTL result under unbounded loss functions (e.g.~least-squares). 
The high-level proof strategy is essentially a simplified version of proof of Theorem \ref{thm:main}, where we use a more naive covering argument for parametric classes that have $\order{\log(1/\eps)}$ covering numbers. For this reason, we will make some simplifications in the proof to avoid repetitions. Instead, we will highlight key differences such as how the concentration argument changes due to unbounded losses. We first make the following assumptions.
\begin{assumption} \label{subexp ass1}For any task distribution $(\x,y)\sim\Dc_t$ and for any task hypothesis $f_t\in\Fc_t$ (induced by $\alpha_t,h_t,\bphi$), we have that $\ell(y,f_t(\x))$ is a $\Xi$ subexponential random variable for some $\Xi>0$. Additionally, assume loss is a $\Gamma>0$ Lipschitz function of $\hat y=f_t(\x)$.
\end{assumption}
We also assume a standard covering assumption. Note that, unlike the proof of Theorem \ref{thm:main}, we focus on parametric classes and use data-agnostic covers.
\begin{assumption} \label{subexp ass2}All modules $\psi_\ell^k\in\Psi_\ell$ are $\Gamma$ Lipschitz and map $\psi_\ell^k:\R^{p_{\ell-1}}\rightarrow\R^{p_\ell}$. For the sake of simplicity assume $\psi_\ell^k(0)=0$ (e.g.~neural net layer with ReLU activation). Additionally, for any Euclidean ball of radius $R$, the covering dimension of $\Psi_\ell$ follows the parametric classes, namely, 
\[
\Nc(\eps;\Psi_\ell,R)\leq d_\ell\log(\frac{3R}{\eps}),
\]
where $d_\ell$ is the covering dimension of $\Psi_\ell$. Verbally, there exists a cover $\Psi^\eps_\ell$, $|\Psi^\eps_\ell|\leq\Nc(\eps;\Psi_\ell,R)$, such that for any $\tn{\x}\leq R$ and for any $\psi\in \Psi_\ell$, there exists $\psi'\in \Psi^\eps_\ell$ such that $\tn{\psi'(\x)-\psi(\x)}\leq \eps$. Additionally, let head $\Hc$ be $\Gamma$ Lipschitz and $d_\Hc$ be the covering dimension for $\Hc$. 
\end{assumption}

\begin{theorem} \label{thm subgauss}Suppose $\Xc\subset\Bc^p(R)$ and Assumptions \ref{subexp ass1} and \ref{subexp ass2} hold. Suppose we have a \eqref{formula:multipath} problem with $\NT$ samples in total where all training samples are independent, however, task sample sizes are arbitrary\footnote{In words, tasks don't have to have $N$ samples each. We will simply control the gap between empirical and population. If tasks have different sizes, then their population weights will similarly change.}. Note that, in the specific setting of Theorem \ref{thm:main}, we have $\NT=NT$ with identical sample sizes. {Assume that $\NT\gtrsim\DoF(\Fc)\log(\NT)+T\log|\Ac|$}. Declare population risk $\Lc_{\Dcb}(\f)=\E[\Lch_{\Scb}(\f)]$. We have that with probability at least $1-\delta$, for all \MP hypothesis $\f\in\Fc$
\begin{align}
|\Lch_{\Scb}(\f)-\Lc_{\Dcb}(\f)|\lesssim   \Xi \sqrt{\frac{L\cdot\DoF(\Fc)+T\log|\Ac|+\log(2/\delta)}{\NT}}.\label{non exp bound}
\end{align}
The right hand side bounds similarly hold for the excess risk $\RMTL(\hat\f)$ where $\hat\f$ is the ERM solution. {$\lesssim$ subsumes the logarithmic dependence on $R,\Gamma,L,\NT$.} The exact bound is below \eqref{exct dep}. Finally, if we solve \eqref{formula:multipath} with fixed pathway choices (rather than searching over $\Ac$), with same probability {and assuming $\NT\gtrsim\DoF(\Fc)\log(\NT)$} we have the simplified bound
\begin{align}
|\Lch_{\Scb}(\f)-\Lc_{\Dcb}(\f)|\lesssim   \Xi \sqrt{\frac{L\cdot\DoF(\Fc)+\log(2/\delta)}{\NT}}.\label{non exp bound2}
\end{align}
\end{theorem}
The theorem is automatically applicable to loss functions bounded by $\Xi>0$. Also note that, this theorem avoids exponential depth dependence compared to Theorem \ref{thm:main}. This is primarily because of the strong coverability of parametric classes which (essentially) applies a log operation to the Lipschitz constant of $\Fc$.

\begin{proof} Let $R_0=R$ and note that, at the $\ell$th layer, the input(output) space has radius $R_{\ell-1}(R_\ell)$ where $R_\ell=\Gamma^{\ell} R$. Let $\Psi_\ell$ denote the hypothesis set of the modules of $\ell$th layer. Fix an $\eps$ cover $\Fc_\eps$ for the sets $(\Psi^{K_\ell}_\ell)_{\ell=1}^L, \Hc^T$ and $\Ac^T$, such that $\Psi_\ell$ is covered according to its input space radius $R_{\ell-1}$ with resolution $\eps_\ell=\frac{\eps}{\Gamma^{L-\ell+1}}$, where $\ell$ is layer depth and prediction head is layer $L+1$. This implies that
\begin{align}
\log |\Fc_\eps|&\leq Td_\Hc\log\frac{3R_L}{\eps}+T\log|\Ac|+\sum_{\ell=1}^L K_\ell d_\ell\log\frac{3R_{\ell-1}}{\eps_\ell}\label{F decomposition}\\
&= (Td_\Hc+\sum_{\ell=1}^L K_\ell d_\ell)\log\frac{3R_L}{\eps}+T\log|\Ac|\\
&\leq \DoF(\Fc)\log\frac{3R_L}{\eps}+T\log|\Ac|\\
&=\DoF(\Fc)\left(\log\frac{3R}{\eps}+L\log \Gamma\right)+T\log|\Ac|.
\end{align}

\noi\textbf{$\bullet$ Step 1: Union bound over the cover.} We now show a uniform concentration argument over this cover. Since each sample is independent of others and each loss is $\Xi$ subexponential, using subexponential Bernstein inequality (e.g.~Prop 5.16 of \cite{vershynin2010introduction}), we have that
\[
\Pro\left(|\Lch_{\Scb}(\f)-\Lc_{\Dcb}(\f)|\geq \frac{t}{\sqrt{\NT}}\right)\leq 2\exp\left(-c\min\left\{\frac{t^2}{\Xi^2},\frac{t\sqrt{\NT}}{\Xi}\right\}\right)
\]
{Let $\eps=\frac{1}{\Gamma {(L+1)} \NT}$, and recall that we assumed $\NT\gtrsim \log |\Fc_\eps|+\tau$.} Now, setting $t\propto \sqrt{ \log |\Fc_\eps|+\tau}$ and union bounding over all $\f\in \Fc_\eps$, we find that, uniformly over $\Fc_\eps$, 
\begin{align}
\Pro\left(|\Lch_{\Scb}(\f)-\Lc_{\Dcb}(\f)|\geq \Xi\sqrt{\frac{\log |\Fc_\eps|+\tau}{\NT}}\right)\leq 2e^{-\tau}.\label{con bound}
\end{align}


\noi\textbf{$\bullet$ Step 2: Perturbation analysis.} Now that covering analysis is done, we proceed with controlling the perturbation. Let $\fb\in \Fc$ be a \MP MTL hypothesis. We choose $\fb'\in\Fc_\eps$ such that:
\begin{itemize}
\item $\f'$ chooses the same pathways.
\item $\f'$ chooses heads $(h'_t)_{t=1}^T$ and modules $(({\psi'}^k_\ell)_{k=1}^{K_\ell})_{\ell=1}^L$ such that these hypotheses are $\eps$ close over their respective input spaces to the hypotheses of $\f$ denoted by $(h_t)_{t=1}^T$ and modules $((\psi^k_\ell)_{k=1}^{K_\ell})_{\ell=1}^L$.
\end{itemize}
Fix an arbitrary $\x\in\Xc$ and task $t\in [T]$. Set the short-hand notation $\psit_\ell=\psi^{\alpha_t}_{\ell}$ and $\psit'_\ell={\psi'}^{\alpha_t}_{\ell}$. Along the pathway $\alpha_t$, define the functions
\[
f^\ell_t(\x)=\begin{cases}f_t(\x)\quad\text{if}\quad \ell=L+1,\\f'_t(\x)\quad\text{if}\quad \ell=0,\\h'_t\circ \psit'_L\circ\dots \psit'_{\ell+1}\circ\psit_{\ell}\circ\dots\circ\psit_1(\x)\quad\text{if}\quad 1\leq \ell\leq L.\end{cases}
\]
Let $\x_\ell=\psit_{\ell}\circ\dots\circ\psit_1(\x)$. Recall that, $\Psi_\ell$ is covered with resolution $\eps_\ell$. Now, through a standard perturbation decomposition, we find that
\begin{align}
|f_t(\x)-f'_t(\x)|&\leq \sum_{\ell=0}^{L} |f^{\ell+1}_t(\x)-f^\ell_t(\x)|\\
&\leq \sum_{\ell=0}^{L} |h'_t\circ \psit'_L\circ\dots \psit_{\ell+1}(\x_\ell)-h'_t\circ \psit'_L\circ\dots \psit'_{\ell+1}(\x_\ell)|\\
&\leq \sum_{\ell=0}^{L} \Gamma^{L-\ell}\eps_{\ell+1}\\
&={(L+1)}\eps.
\end{align}

This establishes that if tasks choose the same pathways, proposed $\eps$ cover ensures that for all $\x\in\Xc$ and task $t$, $\Fc_\eps$ is an ${(L+1)}\eps$ cover of $\Fc$. To conclude, using $\Gamma$ Lipschitzness of the loss function, we obtain
\begin{align}
&\Lc_{\Dcb}(\fb)-\Lc_{\Dcb}(\fb')\leq \sup_{\x,t} |\ell(y,f_t(\x))-\ell(y,f'_t(\x))|\leq \sup_{\x,t}\Gamma |f_t(\x)-f'_t(\x)|\leq \Gamma {(L+1)}\eps\\
&\Lch_{\Scb}(\fb)-\Lch_{\Scb}(\fb')\leq \sup_{\x,t} |\ell(y,f_t(\x))-\ell(y,f'_t(\x))|\leq \sup_{\x,t}\Gamma |f_t(\x)-f'_t(\x)|\leq \Gamma {(L+1)}\eps.
\end{align}
Combining with uniform concentration, we found that, for all $\fb\in \Fc$, with probability $1-\delta$,
\begin{align}
|\Lch_{\Scb}(\f)-\Lc_{\Dcb}(\f)|&\lesssim  \Xi \sqrt{\frac{\log |\Fc_\eps|+\log(2/\delta)}{\NT}}+\Gamma {(L+1)}\eps\\
&\lesssim  \Xi \sqrt{\frac{\DoF(\Fc)\left(\log\frac{3R}{\eps}+L\log \Gamma\right)+T\log|\Ac|+\log(2/\delta)}{\NT}}+\Gamma {(L+1)}\eps.
\end{align}
Recall that $\eps=\frac{1}{\Gamma {(L+1)} \NT}$, then we obtain the advertised uniform concentration guarantee
\begin{align}
|\Lch_{\Scb}(\f)-\Lc_{\Dcb}(\f)|\lesssim   \Xi \sqrt{\frac{\DoF(\Fc)\left(\log(3R\Gamma {(L+1)}\NT)+L\log {\Gamma}\right)+T\log|\Ac|+\log(2/\delta)}{\NT}}.\label{exct dep}
\end{align}
We get the simplified statement \eqref{non exp bound} after ignoring the log factors. Finally, \eqref{non exp bound2} arises by repeating above argument step-by-step while ignoring $|\Ac|$ term in \eqref{F decomposition}.
\end{proof}
\section{Proofs in Section~\ref{sec linear}}\label{app:proof linear}

\subsection{{A direct corollary of Theorem~\ref{thm:main} to linear representations}}
We start with a lemma that controls the worst-case Gaussian complexity of linear models. The proof is standard and stated for completeness.
\begin{lemma} [Linear models]\label{lemma:linear model} Let $\Bc\subset \R^{d\times p}$ be a set of matrices with operator norm bounded by a constant $C>0$ and let $\Xc\subset\Bc^p(R)$ (subset of $\ell_2$ ball of radius $R$) . Then
\[
\Gt_{n}^\Xc(\Bc)\leq CR\sqrt{\frac{dp}{n}}.
\] 
\end{lemma}
\begin{proof} Set $\X_{\eps}=\sum_{i=1}^n\x_i\g_i^\top=\X^\top\Gb$ where $\X\in\R^{n\times p}$ is dataset and $\Gb\distas\Nn(0,1)\in\R^{n\times d}$. Applying Cauchy-Schwarz, we write
\begin{align}
\Gt_n^{\Xc}(\Bc)&= \frac{1}{n}\sup_{\X\in \Xc^n}\E\left[\sup_{\B\in\Bc} \sum_{i=1}^n\g_i^\top \B\x_i\right]=\frac{1}{n}\sup_{\X\in \Xc^n}\E\left[\sup_{\B\in\Bc} \text{trace}(\X_\eps\B)\right]\\
&{\leq C\frac{\sqrt{p}}{n}\sup_{\X\in \Xc^n}\E[\tf{\X_\eps}]\leq C\frac{\sqrt{p}}{n}\sup_{\X\in \Xc^n}\sqrt{\E[\tf{\X^\top\Gb}^2]}\leq CR\sqrt{\frac{dp}{n}}}.
\end{align}
\end{proof}
\begin{corollary}\label{cor linear app} Suppose {Assumptions~\ref{assum:lip2}\&\ref{bounded assume}} hold and input set $\Xc\subset \Bc^p(c\sqrt{p})\footnote{Observe that, this input space is rich enough to capture a random vector with $\order{1}$ subgaussian norm. For instance, a standard normal vector would fall into this set with exponentially high probability as soon as $c>1$.}$ for a constant $c>0$. Let $\hat\f$ be empirical solution of \eqref{LMP-MTL}. Then, with probability at least $1-\delta$, 
    \[
    \RMTL(\hat\f)\lesssim\sqrt{\frac{{p}\cdot\DoF(\Fc)}{NT}}+\sqrt{\frac{\log|\Ac|}{N}+\frac{\log(2/\delta)}{NT}},
    \]
    where $\DoF(\Fc)=T\cdot p_L+\sum_{\ell=1}^L K_\ell \cdot p_\ell \cdot p_{\ell-1}$ is the total number of trainable parameters in $\Fc$.
\end{corollary}
\begin{proof}
    This proof is immediately done by following Theorem~\ref{thm:main} and Lemma~\ref{lemma:linear model}. Since $\Xc\subset\Bc^p(c\sqrt{p})$, and we assume $\Psi_\ell,\ell\in[L]$ have bounded operator norm $C$, for each layer, the input space $\Xc_{\Psi_\ell}\subset\Bc^{p_{\ell-1}}(C^{\ell-1}c\sqrt{p})$ and then following Lemma~\ref{lemma:linear model}, $\Gt_{NT}(\Psi_\ell)\leq C^\ell c\sqrt{p}\sqrt{\frac{p_{\ell}p_{\ell-1}}{NT}}$, $\ell\in[L]$.  Since $\Hc=\Bc^{p_L}(C)$ and $\Xc_\Hc\subset\Bc^{p_L}(C^{L}c\sqrt{p})$, we have $\Gt_N(\Hc)\leq C^{L+1}c\sqrt{p}\sqrt{\frac{p_L}{N}}$. Then we obtain
    \begin{align*}
        \Gt_N(\Hc)+\sum_{\ell=1}^L\sqrt{K_\ell}\Gt_{NT}(\Psi_\ell)\lesssim{{c\sqrt{p}}}\cdot{\sqrt{\frac{C^{L+1}\cdot T\cdot p_L+\sum_{\ell=1}^LC^\ell\cdot K_\ell\cdot p_\ell\cdot p_{\ell-1}}{NT}}}.
    \end{align*}
    Combining it with Theorem~\ref{thm:main} finishes the proof.
\end{proof}

\subsection{Proof of Theorem \ref{cor linear}}
{Corollary~\ref{cor linear app} directly follows by applying Theorem~\ref{thm:main} to the linear representation setting, and therefore $\lesssim$ subsumes dependencies on $\log NT$ and $\Gamma^L$. Instead in Theorem~\ref{cor linear} we establish a tighter bound for parametric hypothesis classes and the sample complexity is only logarithmic in the input space radius $R$ ($R=c\sqrt{p}$ in Corollary~\ref{cor linear app}) and linearly dependent on the number of layers $L$.\\
\begin{proof}
    The theorem is a direct application of Theorem~\ref{thm subgauss} after verifying the assumptions. First, bounded loss function $\ell:\R\times\R\to[0,1]$ implies it is sub-exponential, which verifies Assumption~\ref{subexp ass1}. Second, all module/head functions have bounded spectral/Euclidean norms, which verifies Assumption~\ref{subexp ass2}. One remark (compared to Theorem~\ref{thm subgauss}) is that, since the loss function is bounded, by applying Hoeffding's inequality, \eqref{con bound} holds without enforcing a lower bound constraint on $\NT$.
\end{proof}
}

\subsection{The Need for Well-Populated Source Tasks in Assumption \ref{mtldiverse}}
\begin{lemma}\label{e2e fail} {Consider a weaker version of Assumption \ref{mtldiverse} where we enforce $\bSi_\alpha\succeq c\Iden_{p_L}$ over all clusters \underline{with two or more tasks}}\footnote{The relaxation is not enforcing anything on pathways containing a single task.}  (i.e.~only when $\gamma_\alpha\geq 2/p_L$). Then, there exists a (\eqref{formula:multipath}, \eqref{formula:transfer}) problem pair such that the excess transfer learning risk obeys $\RTFL(\hat f_{\hat\bphi})\geq 1$ as soon as $N\geq p$.
\end{lemma}
\begin{proof} The idea is packing supernet with isolated MTL tasks that are uncorrelated with target while achieving zero MTL risk. We consider a simple supernet construction where $T$ tasks will be processed in parallel and all layers have exactly $K_\ell=T$ modules. Specifically, task $t$ will use the pathway $\alpha_t=[t,t,\dots,t]$ by selecting $t$th module from each layer. This way each task will use a unique pathway and supernet will be fully occupied. Set noise level $\sigma=0$. Observe that as soon as $N\geq p$, $\bt^\star_t$ minimizes both empirical and population risks. Consequently, for any $\tn{\bar{\h}_t}=1$, $\B_{\alpha_t}=\h_t (\bt^\star_t)^\top$ is a valid (and minimum norm) minimizer of empirical and population risks. Here, we highlight the minimum norm aspect because this solution is what gradient descent would converge during MTL phase (while we acknowledge the existence of infinitely-many solutions) \cite{ji2018gradient}. To wrap up the proof, suppose transfer task is orthogonal to all source tasks and observe that, regardless of the transfer prediction head $\hat\h_\Tc$ and pathway choice $t$, we have
\begin{align*}
\RTFL(\hat {f}_{\hat\bphi})&=\E\left[(y-\hat {f}_{\hat\bphi}(\x))^2\right]=\E\left[(\bt_\Tc^\top \x-\hat\h_\Tc^\top \h_t (\bt^\star_t)^\top\x )^2\right]\\
&\geq \tn{\bt_\Tc-(\hat\h_\Tc^\top \h_t) \bt^\star_t}^2\geq \tn{\bt_\Tc}^2=1.
\end{align*}
This concludes the proof. We note that, if $\sigma\neq 0$ same argument would work as $N\rightarrow\infty$. Additionally, through same argument with $\sigma=0$, it can be observed that, a more general lower bound on excess transfer risk is $\min_{t\in [T]}\tn{\bt_\Tc-\bt_t^\star}^2/2$.
\end{proof}

\subsection{Proof of Theorem \ref{e2e thm} and Supporting Results}

We start with a useful lemma to show excess risk of linear least squares problem with dependent noise.
\begin{lemma}[Linear least squares risk with dependent noise] \label{simple LS}Let $\Sc=(\x_i,y_i)_{i=1}^n\distas\Dc$ where $y=\bt^\top \x+z$ where $\x$ is $\order{1}$ subgaussian vector with isotropic covariance and $z$ is $\order{\sigma}$ subgaussian noise. Here, we assume that $\x\& z$ can be dependent, however, orthogonal (i.e. $\E[\x z]=0$). Let $\X=[\x_1~\cdots~\x_n]^\top\in\R^{n\times p}$ and $\X^\dagger$ be the Moore-Penrose pseudoinverse of $\X$. Let $\wedge$ be the minimum symbol. For $n\geq Cp$ for a sufficiently large constant $C\geq 1$, the excess least squares risk and population-empirical risk gap of $\hat\bt=\X^\dagger\y$ is given by
    \begin{align}
    \Lc_{\Dc}(\hat\bt)-\sigma^2\leq C\sigma^2\frac{p+t}{n}~~~~~~&\text{with probability at least $\prb$}\label{prb1}\\
    \Lc_{\Dc}(\hat\bt)-\Lch_{\Sc}(\hat\bt)\leq C\sigma^2\left(\frac{p}{n}+\sqrt{\frac{t}{n}}\right)~~~~~~&\text{with probability at least $\prbb$}.\label{prb2}
    \end{align}
    \end{lemma}
    \begin{proof} Let $\z=[z_1~\dots~z_n]^\top$ and $\sigma_{\min}(\cdot),\sigma_{\max}(\cdot)$ return the smallest and biggest singular value of a matrix. We can write
    \[
    \Lc_{\Dc}(\hat\bt)-\E[z^2]=\tn{\bt-\hat\bt}^2=\tn{(\X^\top\X)^{-1}\X^\top \z}^2\leq \frac{\tn{\X^\top\z}^2}{\sigma_{\min}(\X)^4}.
    \]
    Following \cite{vershynin2010introduction}, we have $\sqrt{2n}\geq \sigma_{\max}(\X)$, $ \sigma_{\min}(\X)\geq \sqrt{n/2}$ each with probability at least $1-e^{-cn}$. The crucial term of interest is $\tn{\X^\top\z}$. To control this, observe that $\X^\top\z=\sum_{i=1}^n z_i\x_i$. Since $z_i\x_i$ is $\order{\sigma}$-subexponential (multiplication of two subgaussians), the summand $\X^\top\z$ has a mixed subgaussian/subexponential tail. Specifically, it obeys \cite[Lemma D.7]{oymak2018learning}
    \[
    {\Pro\left(\tn{\X^\top\z}^2\gtrsim \sigma^2(p+t){n}\right)\leq 2e^{-\sqrt{tn}\wedge t}.}
    \]
    Combining both, with advertised probability we establish the first claim.
    \[
    \frac{\tn{\X^\top\z}^2}{\sigma_{\min}(\X)^4}\lesssim \sigma^2\frac{p+t}{n}.
    \]
    {
    For the second claim, observe that 
    \begin{align}
    \Lch_{\Sc}(\hat\bt)-\frac{1}{n}\tn{\z}^2&{=\frac{1}{n}\tn{\y-\hat{\y}}^2-\frac{1}{n}\tn{\z}^2}=\frac{1}{n}[\tn{(\Iden-\X\X^\dagger)\z}^2-\tn{\z}^2]\\
    &=\frac{1}{n}\tn{\X\X^\dagger\z}^2\leq \frac{\sigma_{\max}(\X)^{2}}{n}\tn{(\X^\top\X)^{-1}\X^\top \z}^2\\
    &\leq 2\tn{(\X^\top\X)^{-1}\X^\top \z}^2\lesssim \sigma^2\frac{p+t}{n}. 
    \end{align}
    Here, the first and second inequalities of last line hold with respective probabilities at least $1-e^{-cn}$ and $\prb$. Additionally, since $z^2$ is $\order{\sigma^2}$-subexponential, $|\frac{1}{n}\tn{\z}^2-\E[z^2]|\lesssim\sigma^2\sqrt{t/n}$ with probability at least $1-2e^{-\sqrt{tn}\wedge t}$.
    }
    Combining all provides the final equation bounding the gap between empirical and population risks.
    \end{proof}

Then we present the following lemma that converts an MTL guarantee into a transfer learning guarantee on a single subspace.

\begin{lemma} \label{MTL->TL}Let $\B\in\R^{r\times p}$ be a matrix with orthonormal rows and fix $\{\h_t\}_{t=1}^T\in\R^r$ with unit covariance and declare distributions $(\x,y)\sim\Dc_t$ obeying $y=\h_t^\top \B\x+z$ with $\E[z^2]=\sigma^2$ and $\E[\x\x^\top]=\Iden_p$. Form $\Hb=[\h_1~\dots~\h_t]^\top$ and assume $C\frac{1}{r}\Iden_r\succeq\frac{1}{T}\Hb^\top\Hb\succeq c\frac{1}{r}\Iden_r$. Now, for some $\eps>0$, suppose that $\hat\f=(\hB,\{\hhb_t\}_{t=1}^T)$ with orthonormal $\hat\B$ achieves small population risk in average that is 
\[
\Lc_{\Dcb}(\hat\f)-\Lc_{\Dcb}(\f_\st)=\frac{1}{T}\sum_{t=1}^T\E_{\Dc_t}[(\h_t^\top \B\x-\hhb_t^\top\hB \x)^2]\leq \eps
\]
where $\Lc_{\Dcb}(\f_\st)=\sigma^2$ is the optimal risk achieved by $\f_\st=(\B,\{\h_t\}_{t=1}^T)$. Let $\Dc_\Tc$ be a new distribution with $y=\h_\Tc^\top \B\x+z$ where $\x,z$ are independent $\order{1},\order{\sigma}$ subgaussian respectively and $\E[\x\x^\top]=\Iden_p$. With probability at least $\prbM$, the transfer learning risk on $\hat\B$ with $M$ samples is bounded as
\[
{\Lc_{\Tc}(\hat f)-\sigma^2}\lesssim r\eps\wedge 1+ C\frac{r+t}{M}.
\]
where $\wedge$ is the minimum symbol. Additionally, if target task vector $\h_\Tc$ is uniformly drawn from unit Euclidean sphere, in expectation over $\h_\Tc$ and in probability over target training datasets (with probability at least $\prbM$), we have the tighter bound
\[
{\E_{\h_\Tc}[\Lc_{\Tc}(\hat f)]-\sigma^2}\lesssim \eps+ C\frac{r+t}{M}.
\]
Finally, in both cases, population-empirical transfer gaps $|\Lc_{\Tc}(\hat f)-\Lch_{\Sc_{\Tc}}(\hat f)|$, $\E_{\h_\Tc}[|\Lc_{\Tc}(\hat f)-\Lch_{\Sc_{\Tc}}(\hat f)|]$ are bounded by $\order{\frac{r+t}{M}}$ with same probability.
\end{lemma}
\begin{proof} Let $\bt_t=\B^\top \h_t$ and $\bth_t=\hB^\top \hhb_t$. We first observe that task $t$ risk is simply 
\[
\Lc_t(\bth_t)=\E_{\Dc_t}[(y-\hhb_t^\top\hB \x)^2]=\sigma^2+\E_{\Dc_t}[(\bt_t\x-\bth_t \x)^2]=\sigma^2+\tn{\bth_t-\bt_t}^2.
\]
Thus, the excess MTL risk is simply
\[
\Lc_{\Dcb}(\hat\f)-\Lc_{\Dcb}(\f_\st)=\frac{1}{T}\tf{\bT-\bTh}^2\leq\eps,
\]
where $\bT,\bTh\in\R^{T\times p}$ are the concatenated task vectors. 

Now, we aim to obtain the transfer learning risk over $\hat\B$. We first write the target regression task $(y,\x)\sim\Dc_\Tc$ with $y= \x^\top\bt_\Tc+z$ (for some $\h$) as 
\begin{align}
y=\x^\top \hB^\top \h+z+\x^\top \Pi_{\hB^\perp}(\bt_\Tc).\label{do you see}
\end{align}
Here set $\x'=\hB\x$ and $z'=\x^\top (\Iden-\hB^\top\hB)\bt_\Tc$, {and then $y=\h^\top\x'+z+z'$}. Note that
\[
\E[\x'z']=\E[\hB\x\x^\top (\Iden-\hB^\top\hB)\bt_\Tc]=0,
\]
verifying that we can treat the representation mismatch as a dependent but orthogonal subgaussian noise. Combined with Lemma~\ref{simple LS}, with probability $\prbM$ (conditioned on $\hB$) this leads to a transfer learning risk of 
\[
{\Lc_{\Tc}(\hat f)-\sigma^2}\leq \tn{\Pi_{\hB^\perp}(\bt_\Tc)}^2+C\frac{r+t}{M}.
\]
{Following the proof of Lemma~\ref{simple LS}, the $\sigma^2$ term on the right hand side of \eqref{prb1} is related to the inputs ($\x'$) and noise ($z+z'$) levels, which are $\order{1}$.}



To proceed, observe that {$\E[z'^2]=\tn{\Pi_{\hB^\perp}(\bt_\Tc)}^2=\tn{\B\bt_\Tc}^2-\tn{\hB\bt_\Tc}^2=\bt_\Tc^\top (\B^\top\B -\hB^\top\hB)\bt_\Tc$}. 
In the worst case, this risk is equal to 
\[
\sup_{\tn{\bt}=1,\B^\top\B\bt=\bt} \tn{\Pi_{\hB^\perp}(\bt)}^2=\|\B^\top\B -\hB^\top\hB\|.
\]
Recall that we are given $\frac{1}{T}\|\bT-\bTh\|^2\leq \frac{1}{T}\tf{\bT-\bTh}^2\leq \eps$. Additionally, {$\bT^\top\bT/T$} is a well-conditioned matrix over the subspace $\text{Range}(\B)$ with minimum nonzero eigenvalue at least $c/r>0$ and condition number upper bounded by $C/c$ (equal to that of $\Hb$). 
If $\eps\leq c/2r$, this also implies {$\la_{\min}(\bTh^\top\bTh/T)\geq c/2r$} and condition number at most $3C/c$. Consequently, applying Davis-Kahan theorem \cite{yu2015useful} on $\bT,\bTh$ pair implies that the eigenspaces $\B,\hB$ of $\bT,\bTh$ obey
\[
\|\B^\top\B -\hB^\top\hB\|\leq \order{\frac{\eps}{c/r}}=\order{r \eps}.
\]
If $\eps r\geq c/2$, we can simply use the tighter estimate {$\|\B^\top\B -\hB^\top\hB\|\leq 1$}, which completes the proof of first part of the lemma.

Secondly, consider the average case scenario where $\h_\Tc\sim\text{unif\_over\_sphere}$. In this case, we observe that, the target-averaged transfer risk follows
\begin{align}
\E_{\bt_\Tc}[\min_{\h}\tn{\bt_\Tc-\hB^\top\h}^2]&=\E[\tn{(\Iden-\hB^\top\hB)\bt_\Tc}^2]=\frac{1}{2r}\tf{\B^\top\B -\hB^\top\hB}^2.
\end{align}
This time, Davis-Kahan theorem yields the tighter estimate (for our purposes) $\frac{1}{2r}\tf{\B^\top\B -\hB^\top\hB}^2\lesssim \frac{1}{T}\tf{\bT-\bTh}^2\leq \eps$. To proceed, we find that, the expected transfer learning risk over task distribution obey the tighter guarantee (with probability at least~$\prbM$)
\[
{\E_{\h_\Tc}[\Lc_{\Tc}(\hat f)]-\sigma^2}\lesssim \eps+ C\frac{r+t}{M}.
\]
The final claim arises as a direct result of our application of Lemma \ref{simple LS} in \eqref{do you see}.
\end{proof}

The following corollary is a \MP MTL guarantee for least-squares regression obtained by specializing the more general Theorem \ref{thm subgauss}.
\begin{corollary} \label{thm lin subgauss}Suppose $\Xc\subset\Bc^p(R)$, $\ell(\hat y,y)$ is quadratic, and Assumptions \ref{bounded assume}\&\ref{linear dist} hold. Solving \eqref{formula:multipath} with the fixed choice of ground-truth pathways {and $NT\gtrsim\DoF(\Fc)\log(NT)$}, with probability at least $1-\delta$, we have that
\begin{align}
\RMTL(\hat\f)\lesssim    \sqrt{\frac{L\cdot\DoF(\Fc)+\log(2/\delta)}{NT}}. 
\end{align}
\end{corollary}
\begin{proof} We need to verify the assumptions of Theorem \ref{thm subgauss}. Observe that $\x,z$ are subgaussian and ground-truth model {$\tn{\bt^\st_t}\leq1$} and all feasible task hypothesis $\bt$ obeys $\tn{\bt}\leq C^{L+1}$ which we treat as a constant (i.e.~fixed depth $L$). Consequently, subexponential norm obeys $\te{(y-\bt^\top \x)^2}=\te{(z+\x^\top (\bt-\bts))^2}\leq \order{C^{2(L+1)}}$ which verifies $\order{1}$ subexponential condition. Similarly, loss function is Lipschitz with $\Gamma=\sup_{y,\x} |y-f_t(\x)|\leq 2C^{L+1}R$. Together, these verify Assumption \ref{subexp ass1}. Note that, Theorem \ref{thm subgauss} has logarithmic dependence on $\Gamma$ which is subsumed within $\lesssim$. Finally, each module is $C$ Lipschitz (due to spectral norm bounds) and log-covering number of $d\times p$ matrices with $C$-bounded spectral norm obeys {$dp\log(3CR/\eps)$}. These two verify Assumption \ref{subexp ass2}.
\end{proof}
\subsubsection{Finalizing the Proof of Theorem \ref{e2e thm}}

{Following the discussion above, we provide a proof of Theorem \ref{e2e thm}. The result below is a formal restatement of the theorem with a few caveats. 
First, we state two closely-related guarantees. First guarantee is when target head $\h_\Tc$ is arbitrary (worst-case) and second one is for when it is uniformly distributed over unit sphere (average case). The latter shaves a factor of $p_L$ in the MTL risk term. Second, the probability term in Theorem \ref{e2e thm} is chosen to be approximate for notational simplicity. Namely, we ignored the $\log(1/\delta)/NT$ term and second order effects. We state the full dependence here which is a bit more convoluted.
}
\begin{theorem} \label{e2e thm2}Suppose Assumptions \ref{bounded assume}--\ref{transfer dist} hold and $\ell(\hat y,y)=(y-\hat y)^2$. {Additionally assume input space is $\Bc^p(c\sqrt{p})$ and $\Hc_\tgt=\R^{p_L}$}\footnote{We make this assumption (no norm constraint unlike MTL phase) since during transfer learning, we simply solve least-squares. Thanks to this, we achieve faster rates.}. Solve MTL problem \eqref{formula:multipath} with the knowledge of ground-truth pathways $(\bar{\alpha}_t)_{t=1}^T$ to obtain a supernet $\hat\bphi$ {and assume $NT\gtrsim\DoF(\Fc)\log(NT)$}. Solve transfer learning problem \eqref{formula:transfer} with $\hat\bphi$ to obtain a target hypothesis $\hat{f}_{\hat\bphi}$. Then, with probability at least {$1-3e^{-cM}-4\delta$}, excess target risk \eqref{TFL risk} of \ref{formula:transfer} obeys
\begin{align}
\E_{\alpha_{\Tc}}[\RTFL(\hat f_{\hat\bphi})]\lesssim \frac{p_L}{M}+p_L\sqrt{\frac{L\cdot\DoF(\Fc)+\log(2/\delta)}{NT}}+\left[\sqrt{\frac{\log(2|\Ac|/\delta)}{M}}\right]_+.\label{adv worst}
\end{align}
Here, the probability is over the source datasets and the (input, noise) pairs of the target dataset i.e.~$(\x^\Tc_i,z^\Tc_i)_{i=1}^M$, and we used the short hand $[x]_+=x+x^2$. Additionally, if target distribution follows the same generative model with prediction head $\h_\Tc$ drawn uniformly at random over the unit sphere, we obtain the tighter bound
\begin{align}
\E_{\alpha_{\Tc},\h_\Tc}[\RTFL(\hat f_{\hat\bphi})]\lesssim \frac{p_L}{M}+\sqrt{\frac{L\cdot\DoF(\Fc)+\log(2/\delta)}{NT}}+\left[\sqrt{\frac{\log(2|\Ac|/\delta)}{M}}\right]_+.\label{adv avg}
\end{align}
\end{theorem}
\noindent\textbf{Remark.} Note that, above, we split probability space into three independent variables. Source datasets $\Scb$, (input, noise) pairs of the target dataset i.e.~$(\x^\Tc_i,z^\Tc_i)_{i=1}^M$, and finally target path $\alpha_\Tc$. The result is with high probability over the former two and expectation over the latter.

\begin{proof} In this proof, we aim to reduce the \MP MTL guarantee to a Vanilla MTL scenario so that we can utilize Lemma \ref{MTL->TL}. Assumptions \ref{mtldiverse} and \ref{transfer dist} will be critical towards this goal. Recall that, we have the \MP MTL guarantee from Corollary \ref{thm lin subgauss} so that, with probability $1-\delta$,
\[
\RMTL(\hat\f)\lesssim\sqrt{\frac{L\cdot\DoF(\Fc)}{NT}}+\sqrt{\frac{\log(2/\delta)}{NT}}.
\]
Let us call this event $\Ec_1$. Here, we omitted the $\log|\Ac|/N$ term because our transfer guarantee will require the knowledge of ground-truth pathways for sources (even if it is not required for the target). The main idea is to show that small $\RMTL(\hat\f)$ implies that target will fall on a pathway with small source-averaged risk. This way, we can apply Lemma~\ref{MTL->TL} to provide a guarantee for the target. To proceed, we gather all unique ground-truth pathways via $\Gamma=\{\gamma_i\}_{i=1}^S$. Additionally, let $C(\gamma)$ be the number of tasks that chooses pathway $\gamma$. 

Finally let $\Lc'_i(\f)$ be the excess task-averaged population risk over $\gamma_i$, that is $\Lc'_i(\f)=\frac{1}{C(\gamma_i)}\sum_{\bar\alpha_t=\gamma_i}\{\Lc_t(f_t)-\sigma^2\}$. With this definition, we can write MTL excess risk as
\[
\frac{1}{T} \sum_{i=1}^S C(\gamma_i)\Lc'_i(\f)\lesssim\sqrt{\frac{L\cdot\DoF(\Fc)}{NT}}+\sqrt{\frac{\log(2/\delta)}{NT}}.
\]
To proceed, we will view each $\Lc'_i$ as a vanilla MTL problem over pathway $\gamma_i$. Following Assumption~\ref{transfer dist}, we draw the random pathway $\alpha_\Tc$ of the target task and it is equal to $\alpha_\Tc=\gamma_i \in\Gamma$. Note that this event happens with probability $\Pro(\alpha_\Tc=\gamma_i)=C(\gamma_i)/T$. Conditioned on this, let us control the transfer risk.

Note that during \ref{formula:transfer} we will search over all pathways $\alpha\in\Ac$. Denote $\bar{\B}_\alpha,\hB_\alpha\in\R^{p_L\times p}$ are the ground-truth and empirical weights of the linear model induced by $\alpha$. Denote the transfer learning model over $\hB_\alpha$ via $\hat f_\alpha$. For any choice of $\alpha$, applying Lemma \ref{simple LS}, we know that empirical-population transfer gap $\Lc_{\Dc_\Tc}(\hat f_\alpha)-\Lch_{\Sc_\Tc}(\hat f_\alpha)$ is bounded by $\order{\frac{p_L}{M}+\left[\sqrt{\frac{\log(2/\delta)}{M}}\right]_+}$ with probability at least $1-2e^{-cM}-2\delta$. This is over the input/noise distribution of target samples (arbitrary $\alpha_\Tc=\gamma_i$ and associated ground-truth $\bt_\Tc$). Union bounding over all potential pathways target task may use, we obtain that, 
\[
\sup_{\alpha\in\Ac}|\Lc_{\Dc_\Tc}(\hat f_\alpha)-\Lch_{\Sc_\Tc}(\hat f_\alpha)|\leq \order{\frac{p_L}{M}+\left[\sqrt{\frac{\log(2|\Ac|/\delta)}{M}}\right]_+}.
\]
Consequently, empirical risk minimization over all pathways will choose a target model $\hat f_{\hat\bphi}$ guaranteeing {with probability at least $1-2e^{-cM}-2\delta$}
\begin{align}
\RTFL(\hat f_{\hat\bphi})&\leq \min_{\alpha}\RTFL(\hat f_{\alpha})+\order{\frac{p_L}{M}+\left[\sqrt{\frac{\log(2|\Ac|/\delta)}{M}}\right]_+}\\
&\leq\RTFL(\hat f_{\gamma_i})+\order{\frac{p_L}{M}+\left[\sqrt{\frac{\log(2|\Ac|/\delta)}{M}}\right]_+}.
\end{align}
The latter line is reasonable because we know that ground-truth pathway $\alpha_\Tc=\gamma_i$ is a great candidate for being population minima. {Applying Lemma \ref{MTL->TL}} again over the path $\gamma_i$, with probability $1-e^{-cM}-\delta$, we obtain
\[
\RTFL(\hat f_{\gamma_i})\leq p_L\Lc'_i(\f)\wedge 1+C\frac{p_L}{M}+\left[\frac{\log(2/\delta)}{M}\right]_+.
\]
Combining with above, with probability {$1-3e^{-cM}-3\delta$}, the ERM solution over all pathways obeys
\[
\RTFL(\hat f_{\hat\bphi})\leq p_L\Lc'_i(\f)\wedge 1+\order{\frac{p_L}{M}+\left[\sqrt{\frac{\log(2|\Ac|/\delta)}{M}}\right]_+}.
\]
Note that above holds for worst-case prediction head $\h_\Tc$. Additionally, applying Lemma \ref{MTL->TL} again and assuming $\h_\Tc$ is generated uniformly over the unit sphere, on the same event, we find
\begin{align}
\E_{\h_\Tc}[\RTFL(\hat f_{\hat\bphi})]\leq \Lc'_i(\f)+\order{\frac{p_L}{M}+\left[\sqrt{\frac{\log(|\Ac|/\delta)}{M}}\right]_+}.\label{avg control}
\end{align}
Now, for fixed MTL dataset, taking expectation over $\alpha_\Tc$, with same probability over the input/noise distribution
\begin{align}
\E_{\alpha_{\Tc}}[\RTFL(\hat f_{\hat\bphi})]&\leq \sum_{i=1}^S \frac{C(\gamma_i)}{T}p_L\Lc'_i(\f)\wedge 1+\order{\frac{p_L}{M}+\left[\sqrt{\frac{\log(2|\Ac|/\delta)}{M}}\right]_+}\\
&\leq p_L\RMTL(\hat\f)\wedge 1+\order{\frac{p_L}{M}+\left[\sqrt{\frac{\log(2|\Ac|/\delta)}{M}}\right]_+}.
\end{align}
Let us call this event $\Ec_2$ which is independent of $\alpha_\Tc$. Union bounding $\Ec_1$ and $\Ec_2$ (both independent of $\alpha_\Tc$), we obtain the advertised bound \eqref{adv worst}. We obtain \eqref{adv avg} through same argument following the average-case control \eqref{avg control}.
\end{proof}

%





\section{Not All Optimal MTL Pathways are Good for Transfer Learning}\label{app:tfail2}

Ideally we would like to prove Theorem \ref{e2e thm} without assuming that MTL phase is solved with the knowledge of ground-truth pathways. While we believe this may be possible under stronger assumptions, here, we discuss why this problem is pretty challenging with a simple example on linear representations.

\noindent \textbf{Setting:} Suppose we have a single layer linear supernet with $K=2$ modules each with size $2R\times p$. This corresponds to the Cluster MTL model where we simply wish to group the tasks into two clusters and train vanilla MTL over individual clusters. This simple setting will already highlight the issue.

$\bullet$ \textbf{Source tasks:} Consider four groups of tasks $(\bT_i)_{i=1}^4$ where $\bT_i=(\bt_{ij})_{j=1}^{T/4}$. We assume that $T/4$ task vectors from $\bT_i$ perfectly span an $R$ dimensional subspace $S_i$ (at least $T\geq 4R$). Additionally, set $(S_i)_{i=1}^4$ to be perfectly orthogonal over $\R^p$. Also assume that the tasks are linear and noiseless i.e.~$y_{ij}=\x_{ij}^\top\bt_{ij}$.

\begin{lemma} Suppose representation modules $\B_1,\B_2\in\R^{2R\times p}$ are constrained to have orthonormal rows. Define the ground-truth pathways where $\bT_1,\bT_2$ are on pathway $1$ and $\bT_3,\bT_4$ are on pathway $2$. Now, assume that transfer learning task $\bt_\Tc$ is drawn uniformly at random from the $2R$ dimensional subspace of one of these pathways. Assume target is linear \& isotropic: $(\x,y)\sim\Dc_{\Tc}$ obeys $y=\x^\top \bt_\Tc+z$ where $\E[\x\x^\top]=\Iden_p$ and $\x,z$ are orthogonal. Then, regardless of the source sample/task sizes $N,T$ and target sample size $M$, there exists an MTL solution such that, excess transfer risk of final target hypothesis $\hat f_\Tc$ obeys 
\[
\E_{\bt_\Tc}[\RTFL(\hat f_\Tc)]\geq c.
\]
for some absolute constant $c>0$. Additionally, $\RTFL(\hat f_\Tc)\geq 0.5$ almost surely as $R\rightarrow\infty$.
\end{lemma}
\begin{proof} As the reader might have noticed, the argument is straightforward. Create the following MTL solution: Let $\B_1$ be an orthonormal basis for $\bT_1,\bT_3$ and let $\B_2$ be an orthonormal basis for $\bT_2,\bT_4$. Without losing generality, for $\B_1$, let us set it so that first $R$ rows are assigned to $\bT_1$ and last $R$ assigned to $\bT_3$ (same for $\B_2$). Note that, we simply swapped $\bT_2$ with $\bT_3$ in pathway assignments.

Observe that $\B_1$ and $\B_2$ achieves zero MTL risk because they contain all task vectors $\bT_i=(\bt_{ij})_{j=1}^{T/4}$ in their range and problems are noiseless. What remains to show is that $\B_1,\B_2$ assignments are poor choices for the target task drawn from either $\B_1^\st$ induced by $\bT_1,\bT_2$ or $\B_2^\st$ induced by $\bT_3,\bT_4$. Without losing generality, suppose $\bt_\Tc$ is drawn from $\B_1^\st$. Observing $\bt_\Tc$ lies on the combined range of $\B_1,\B_2$, and using properties of linear regression with isotropic features, we bound the target transfer risk via
\begin{align*}
\Lc_\Tc(\hat f_\Tc)-\E[z^2]&=\min_{i\in \{1,2\}}\Lc_\Tc(\B_i^\top \hhb_{\Tc})-\E[z^2]\\
&= \min_{i\in \{1,2\}}\tn{\B_i^\top \hhb_{\Tc}-\bt_\Tc}^2\\
&\geq\min_{i\in \{1,2\}}\min_{\h}\tn{\B_i^\top \h-\bt_\Tc}^2\\
&=\min_{i\in \{1,2\}}\tn{\B_{3-i}\bt_\Tc}^2=\min_{i\in \{1,2\}}\tn{\B_i\bt_\Tc}^2\\
&=\min_{i\in \{1,2\}} \tn{\text{Proj}_{\Sc_i}(\bt_\Tc)}^2.
\end{align*}
{The last line highlights the fact that $S_i$ lies on $\B_i$ and projection of $\bt_\Tc$ on $\B_i$ is exactly equal to its projection on $S_i$ by pathway assignments.} Since $\bt_\Tc$ is uniformly drawn, the last line is equivalent to $X(\g,\g')=\frac{\tn{\g}^2}{\tn{\g}^2+\tn{\g'}^2}\wedge \frac{\tn{\g'}^2}{\tn{\g}^2+\tn{\g'}^2}$ for $\g,\g'\distas\Nn(0,\Iden_{R})$. Observing $\tn{\g}^2,\tn{\g'}^2$ are Chi-squared, it is clear that, for all $R$ and for some constant $c_0>0$, we have $\Pro(0.5\leq \frac{\tn{\g}^2}{{R}} \leq 1.5)\geq c_0$. On these events on $\g,\g'$, we have $X(\g,\g')\geq 1/4$ and $\E[X(\g,\g')]\geq c=c_0^2/4$. Finally, as dimension $R\rightarrow\infty$, we have $\tn{\g}^2/\tn{\g'}^2\rightarrow 1$ almost surely, which similarly implies $X(\g,\g')\rightarrow 0.5$.
\end{proof}

\section{Experimental Details on Section~\ref{hierarchy}}\label{app:numeric}

We provide further details on the experiments in Section~\ref{hierarchy} as well as incorporate additional experiments. 
\subsection{Algorithms for Vanilla MTL, Cluster MTL, and \MP~MTL}\label{sec:algo}

To facilitate faster and more stable convergence of all three algorithms, we used a conventional approach from nonconvex optimization literature which has also been proposed in the context of linear representation learning \cite{kong2020robust,sun2021towards,bouniot2020towards,tripuraneni2021provable}. Specifically, linear representation learning with Vanilla MTL has a bilinear form similar to matrix factorization. Thus, first-order method to solve Vanilla MTL benefits from proper initialization of the representation. In our experiments, we use such a two-step procedure:

$\bullet$ \textbf{Initialization:} At the start of MTL, build an initialization for the representation.

$\bullet$ \textbf{Alternating least-squares (ALS):} Train prediction heads and representation layers through alternating least-squares.

Here, we note that ALS is same as alternating gradient descent (AGD) however we are essentially running infinitely many gradient iterations before alternating. The reason we use this procedure for all three algorithms is to provide a fair comparison without the worry of tuning learning rates for each algorithm individually. Initialization plays a useful role in further stabilizing ALS.

While prior works provide initialization methods for MTL, we will also develop a novel initialization algorithm for \MP MTL. We believe this may be an interesting future direction for providing provable computational guarantees for \MP MTL.

\noindent \textbf{Initialization procedures:} We first revise the procedure for Vanilla MTL. Suppose we are given $T$ tasks with dataset $\Scb$ where input features have isotropic covariance. We will use the procedure discussed in \cite{sun2021towards} where the authors claim improvement over \cite{tripuraneni2021provable,kong2020meta}. 

$\bullet$ \textbf{Vanilla MTL:} initialization is a method-of-moments procedure as follows:
\begin{enumerate}
\item Form the $\bth_t$ estimates via $\bth_t=\frac{1}{N}\sum_{i=1}^N y_{ti}\x_{ti}$.
\item Form the moment matrix $\M=\sum_{t=1}^T \bth_t\bth_t^\top$.
\item Set $\hB_0\in\R^{R\times p}$ to be the top $R$ eigenvectors of $\M$.
\end{enumerate}
At this point, we can start running our favorite choice of first order method starting from the initialization $\hB_0$. In our implementation, we run ALS where we estimate $\{\hat\h_t\}_{t=1}^T$ (by fitting LS given $\hB$), then re-estimate $\hB$ and keep going.

$\bullet$ \textbf{Cluster MTL:} In our experiments, we assumed clusters (i.e.~pathways) are known. This is in order to decouple the challenge of task-clustering from the comparisons in Figure \ref{fig:MTLcomparison}. We note that task clustering has been studied by \cite{fifty2021efficiently,kumar2012learning,kang2011learning} (Leveraging relations between tasks are explored even more broadly \cite{zhuang2020comprehensive}.) however these works don't come with comparable statistical guarantees. In our setup, Cluster MTL simply runs $K$ Vanilla MTL algorithms in parallel over individual clusters using ground-truth pathways.

$\bullet$ \textbf{\MP MTL:} We propose an initialization algorithm which is inspired from the Vanilla MTL algorithm as follows. Again, we assume knowledge of clustering/pathways.
\begin{enumerate}
\item Estimate shared first layer $\hB_1\in\R^{R\times p}$ via \textbf{Vanilla MTL} initialization using all data.
\item Estimate cluster-specific representations $(\tilde{\hB}_2^k)_{k=1}^K\in\R^{r\times p}$ via \textbf{Vanilla MTL} initialization over each cluster data.
\item Estimate the second layer $({\hB}_2^k)_{k=1}^K\in\R^{r\times R}$ by projecting $\tilde{\hB}_2^k$ onto the $R$-dimensional first layer as follows
\[
\hB_2^k= \tilde{\hB}_2^k\hB_1^\top.
\]
\end{enumerate}
We then run ALS where we go in the order: Prediction heads, second layers, first layer (repeat).

\noindent\textbf{Remark on unknown clusters:} We note that a simple approach to identifying clusters when they are unknown is by solving Vanilla MTL and then clustering the resulting weight vectors $\{\hat\bt_t\}_{t=1}^T$ of the Vanilla MTL solution (e.g.~via $K$-means). The reason is that, the ground-truth weights $\{\bt^\st_t\}_{t=1}^T$ are simply points that lie on $r$-dimensional latent cluster-subspaces that we would like to recover. Naturally, the (random) points on the same subspace will have higher correlation. This viewpoint (restricted to linear setting) also connects well with the broader subspace clustering literature where each learning task is a point on a high-dimensional subspace \cite{vidal2011subspace,parsons2004subspace,elhamifar2013sparse}. The challenge in our setting is we only get to see the points through the associated datasets. {Figure~\ref{fig:clustering_exp} shows our results assuming unknown source pathways.} 

In the next section, we discuss a few more experiments comparing these three approaches.
\begin{figure*}[t]
  \centering
  \begin{subfigure}[t]{.35\textwidth}
    \centering
    \includegraphics[width=\linewidth]{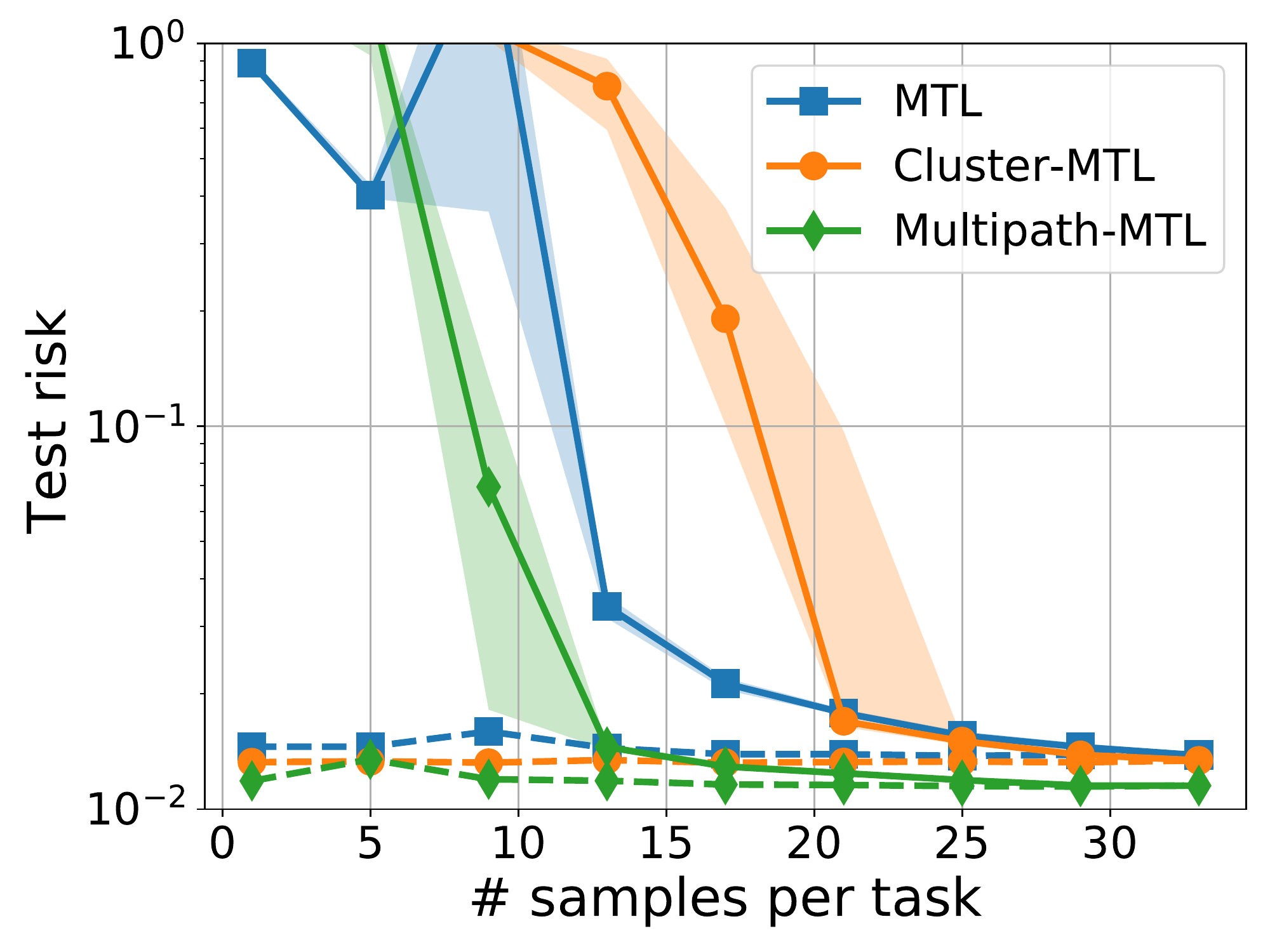}
    \caption{Varying $N_1$ with $\Tbar=10,K=20,N_2=33$}\label{fig:app1}
  \end{subfigure}
  \hspace{20pt}
  \begin{subfigure}[t]{.35\textwidth}
    \centering    \includegraphics[width=\linewidth]{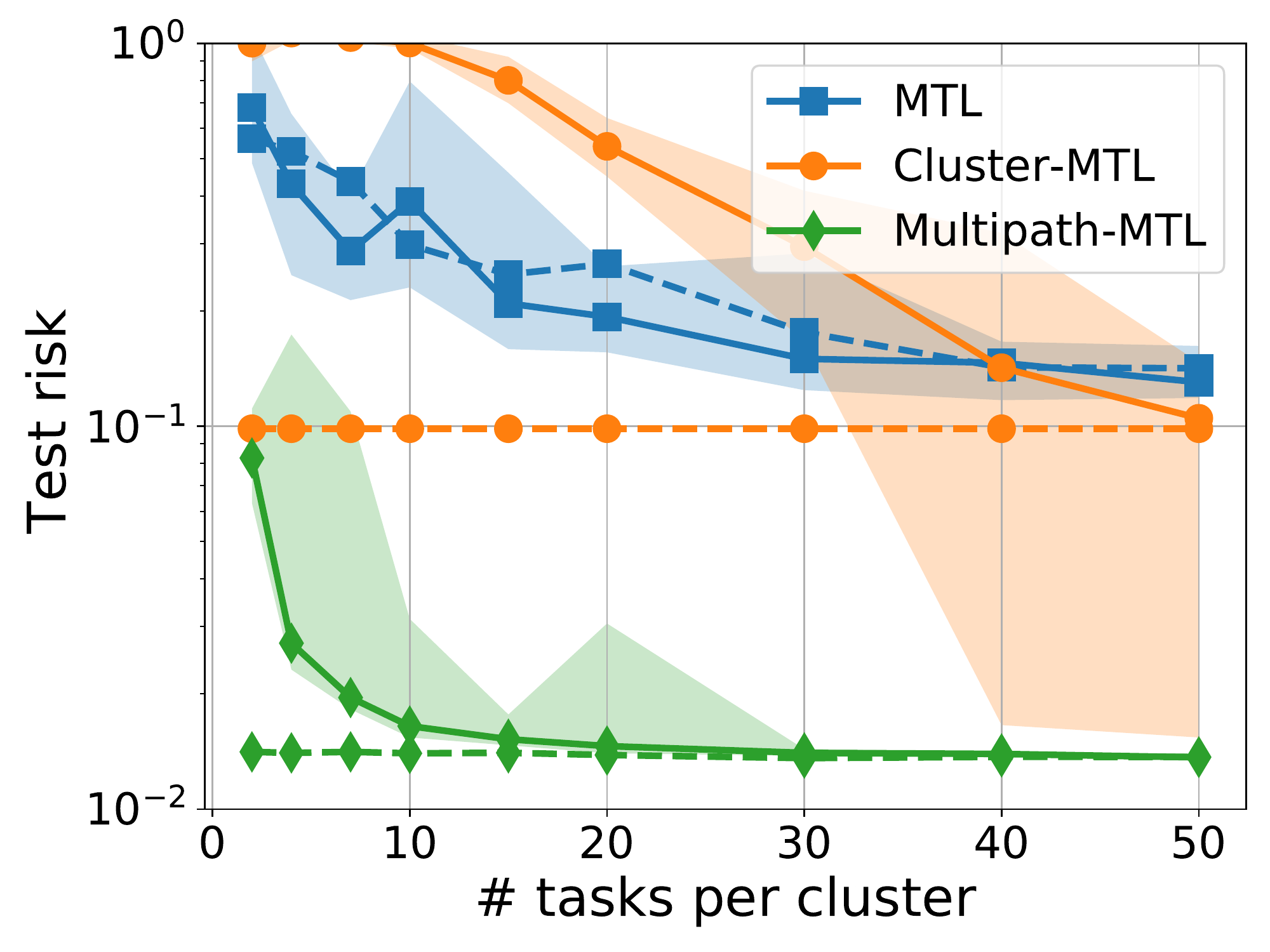}
    \caption{Varying $\Tbar_1$ with $N=10,K=20,\Tbar_2=50$}\label{fig:app2}
  \end{subfigure}
  \caption{We evaluate (Vanilla) MTL, Cluster-MTL and Multipath-MTL using imbalanced data in a linear regression setting (half tasks have more, half have less data resources). In both experiments, there are $K=20$ clusters in total. In Fig.~\ref{fig:app1}, all clusters have $\Tbar=10$ tasks. Then, we fix the sample size of the tasks in $10$ of the clusters ($N_2=33$), and change the sample size in the other $10$ clusters ($N_1$) from $1$ to $33$. Then we solve the three MTL (Vanilla, Cluster, \MP) problems. Solid/Dashed curves show the test risk of the tasks who have fewer/more samples. In Fig.~\ref{fig:app2}, instead we fix the sample size of each task ($N=10$). In $10$ of the clusters, we set the number of tasks $\Tbar_2=50$. While in the other clusters, the number of tasks is varied from $2$ to $50$. Again, we run experiments under all the three settings and plot the test risk of fewer/more tasks in solid/dashed curves. The curves in both figures show the median risks and the shaded regions highlight the first and third quantile risks. Each marker is average of 20 independent runs.}\label{fig:MTLcomparison2}
  \end{figure*}
\subsection{Additional Numerical Experiments}
In Figure~\ref{fig:MTLcomparison2}, we conduct more experiments to see how tasks with less data resources perform in MTL when trained together with other tasks which have more resources. Here, by resources we either mean a task having more samples $N$ or a task having other (related) tasks along its pathway/cluster. Thus, our experiments involve imbalanced training data. We consider two experimental settings to show how \MP MTL benefits accuracy compared to the other two MTL models: Vanilla MTL and Cluster MTL.

\noindent\textbf{Experimental settings:} Consider the same Vanilla MTL, Cluster MTL and \MP MTL problems in linear regression regime as discussed in Section~\ref{hierarchy} and follow the same algorithm in Section~\ref{sec:algo}. In the experiments, same as Section~\ref{hierarchy}, we set $p=32$, $R=8$, and $r=2$. We consider MTL problem with $K=20$ clusters. Here, data is noisy. In Fig.~\ref{fig:app1}, there are $10$ tasks in each cluster. In half of the clusters, each task has fixed sample size, $N_2=33$ (more resource); while in the remaining $10$ clusters, the sample size ($N_1$) varies from $1$ to $33$ (less resource). Solid curves display the test risk of the tasks with $N_1$ samples and dashed curves present the test risk of tasks with $N_2$ samples.  Rather than changing number of samples, in experiments shown in Figure~\ref{fig:app2}, we create another scenario where number of tasks per cluster is varied (as a measure of data resource). Here, instead all tasks contain $N=10$ samples. For $10$ of the total clusters, there are fixed $\Tbar_2=50$ tasks in each cluster. However, the other $10$ contain only $\Tbar_1$ tasks in each cluster, and we compare the performance with different $\Tbar_1$ selections. We change $\Tbar_1$ from $2$ to $50$ and results are displayed in Fig~\ref{fig:app2}. Similar, solid curves present the results of the clusters who contain fewer tasks (less resource), to the contrary, dashed curves present the test risk of clusters with fixed $\Tbar_2=50$ tasks (more resource).

In both figures, Multipath-MTL performs better than the other two models, which again shows that the sample complexity of hierarchical model is smaller than the vanilla and clustering models. When there are fewer samples or fewer tasks, all the three methods fail at learning a good representation. The three dashed curves in Fig.~\ref{fig:app1} behave in line with expectations: They follow from the fact that tasks with more samples can learn decent representations by themselves. The solid curve of Cluster MTL decreases slower, and it is because other than the other two methods where clusters are correlated and representations are shared, in Cluster MTL setting (as depicted in Fig.~\ref{fig:cluster}), clusters are separately trained. Therefore, there is no benefit across the clusters. In Fig.~\ref{fig:app2}, firstly, the evidence that orange and blue dashed curves are above the green one again shows the sample efficiency of Multipath MTL. Here, when there are only $2$ tasks for the $10$ resource-poor clusters, the Cluster MTL has the worst performance because there is no representation sharing across clusters. Test risk of Vanilla MTL does not change too much even the task number increases. It is because MTL representation of vanilla model is larger and tasks don't have enough samples to train their prediction heads. For instance, blue solid curve hits blue dashed curve at very beginning, which shows that the model is already trained well and adding more tasks cannot help too much (both more resource tasks and less resource tasks are doing similar).

\end{document}